\documentclass[twoside,11pt]{article}
\usepackage{myjmlr2e}

\usepackage{amsmath}
\usepackage{amsfonts}
\usepackage{amssymb}
\usepackage{mathrsfs}                  
\usepackage{enumitem}
\usepackage{amscd}
\usepackage{xcolor}
\usepackage{mathtools}
\usepackage{bbm}
\usepackage{bm}

\newtheorem{assumption}{Assumption}
\newcommand{\bfi}{\bfseries\itshape}

\DeclarePairedDelimiter\floor{\lfloor}{\rfloor}

\allowdisplaybreaks

\newcommand{\vertiii}[1]{{\left\vert\kern-0.25ex\left\vert\kern-0.25ex\left\vert #1 
    \right\vert\kern-0.25ex\right\vert\kern-0.25ex\right\vert}}
\newcommand{\vertii}[1]{{\left\vert\kern-0.25ex\left\vert #1 
    \right\vert\kern-0.25ex\right\vert}}

\jmlrheading{1}{2019}{1-48}{4/00}{10/00}{GGO19}{Lukas Gonon, Lyudmila Grigoryeva, and Juan-Pablo Ortega}

\ShortHeadings{Risk Bounds for Reservoir Computing}{Gonon, Grigoryeva, and Ortega}
\firstpageno{1}

\begin{document}

\title{Risk Bounds for Reservoir Computing}

\author{\name Lukas Gonon \email lukas.gonon@unisg.ch\\
\addr Faculty of Mathematics and Statistics \\
Universit\"at Sankt Gallen\\
Switzerland \AND
\name Lyudmila Grigoryeva \email Lyudmila.Grigoryeva@uni-konstanz.de \\
\addr Department of Mathematics and Statistics \\
Graduate School of Decision Sciences \\
Universit\"at Konstanz\\
Germany \AND
\name  Juan-Pablo Ortega \email Juan-Pablo.Ortega@unisg.ch\\
\addr Faculty of Mathematics and Statistics \\
Universit\"at Sankt Gallen\\
Switzerland \\
Centre National de la Recherche Scientifique (CNRS)\\
France}

\editor{}

\maketitle

\begin{abstract}
We analyze the practices of reservoir computing in the framework of statistical learning theory. In particular, we derive finite sample upper bounds for the generalization error committed by specific families of reservoir computing systems when processing discrete-time inputs under various hypotheses on their dependence structure. Non-asymptotic bounds are explicitly written down in terms of the multivariate Rademacher complexities of the reservoir systems and the weak dependence structure of the signals that are being handled. This allows, in particular, to determine the minimal number of observations needed in order to  guarantee a prescribed estimation accuracy with high probability for a given reservoir family. At the same time, the asymptotic behavior of the devised bounds guarantees the consistency of the empirical risk minimization procedure for various hypothesis classes of reservoir functionals. 
\end{abstract}


\section{Introduction}

Reservoir computing (RC) is a well established paradigm in the supervised learning of dynamic processes, which exploits the ability of specific families of semi-randomly generated state-space systems to solve computational and signal treatment tasks, both in deterministic and in stochastic setups. In recent years, both researchers and practitioners have been paying increasing attention to reservoir systems and their applications in learning. The main reasons behind this growing interest are threefold. Firstly, training strategies in reservoir computing are easy to implement as they simply consist in estimating the weights of memoryless static readouts, while the internal weights of the reservoir network are randomly created; this feature is closely linked to ideas originating in biology and the neurosciences in relation with the design of  brain-inspired computing algorithms. Second, there is an important interplay between reservoir systems and the theory of dynamical systems, recurrent neural networks, and nonlinear discrete-time state-space systems,  which makes the collection of tools available for their analysis very rich and explains in part why RC appears in the literature assimilated to other denominations, such as {\bfi  Liquid State Machines} (see \cite{Maass2000, maass1, Natschlager:117806, corticalMaass, MaassUniversality}) or  {\bfi  Echo State Networks} (see \cite{Jaeger04, jaeger2001}). Finally, several families of reservoir systems have shown excellent performance in various classification and forecasting  exercises including both standard machine learning benchmarks  (see \cite{lukosevicius} and references therein) and sophisticated applications that range from  learning  the attractors of chaotic nonlinear infinite dimensional dynamical systems (see \cite{Jaeger04, pathak:chaos, Pathak:PRL, Ott2018})  to the  detection of  Steady-State Visual Evoked Potentials (SSVEPs)  in electroencephalographic signals as in \cite{Ibanez-Soria2019}. It is also important to point out that RC implementations with dedicated hardware have been proved to exhibit information processing speeds that significantly outperform standard Turing-type computers (see, for instance,~\cite{Appeltant2011, Rodan2011, SOASforRC, Larger2012, Paquot2012, photonicReservoir2013, swirl:paper, Vinckier2015, Laporte2018}).

For a number of years, the reservoir computing community has  worked hard on characterizing the key properties that explain the performance of reservoirs in classification, forecasting, and memory reconstruction tasks and also on formulating necessary conditions for a given state-space system to  serve as a properly functioning reservoir system. Salient features of reservoir systems that have been shown to be important are the {\bfi  fading memory property (FMP)}, which appears in the context of systems theory \citep{volterra:book, wiener:book, Boyd1985}, computational neurosciences \citep{corticalMaass}, physics \citep{Coleman1968}, or mechanics (see \cite{Fabrizio2010} and references therein), the {\bfi  echo state property (ESP)} (\cite{jaeger2001, Yildiz2012, Manjunath:Jaeger}), and the pairwise {\bfi  separation property (SP)} (see, for instance, \cite{DynamicalSystemsMaass, lukosevicius, maass2} and references therein). Much effort has been made to provide rigorous definitions for these concepts and to characterize their relations under various hypotheses (see \cite{jaeger2001, RC9} and references therein). In particular, the crucial importance of these properties manifests itself in a series of universal approximation results which have been obtained for RC systems (see for example~\cite{MaassUniversality, RC6, RC7, RC8}). This feature is a {\it dynamic} analog to well-established universal approximation properties for {\it static} machine learning paradigms, like neural networks (\cite{cybenko, funahashi:universality, hornik}), for which  the  so-called {\bfi approximation error} (see \cite*{cucker:smale, Smale2003, cucker:zhou:book})  can be made arbitrary small.

From the point of view of learning theory, the most important feature for any paradigm is its ability to generalize. Here this means that the performance of a given RC architecture on a training sample should be comparable to  its behavior on previously unseen realizations of the same data generation process. 
In the RC literature, this problem has been traditionally tackled using the notion of {\bfi memory capacity}, that has been the subject of much research (\cite{Jaeger:2002, White2004, Ganguli2008, Hermans2010, dambre2012, GHLO2014_capacity, linearESN, RC3}). Unfortunately, it has been recently shown that optimizing memory capacity does not necessarily lead to higher prediction performance (see \cite{marzen:capacity}). Moreover, the recently proved universal approximation properties of RC that we just brought up do not guarantee that a given universal reservoir system will exhibit small {\bfi generalization errors}. In other words, they guarantee the availability of RC architectures that exhibit arbitrarily small training errors but give no control on their generalization power. 

Following the standard  learning theoretical approach to measure the generalization power would lead to consider the difference between the training error (empirical risk) and the testing error (statistical risk or generalization error) and aim at controlling it uniformly over a given class of reservoir systems by using a measure of the class complexity. A number of complexity measures for function classes have been proposed over the years. We can name the Vapnik-Chervonenkis (VC) dimension (\cite{Vapnik1998}),  Rademacher and Gaussian complexities (\cite{Bartlett2003}),   uniform stability (\cite{Mukherjee2002, Bousquet2002, Poggio2004}),  and their modifications. In particular, a vast literature is available also on complexities  and probably approximately correct (PAC) bounds for multilayer neural networks or recurrent neural networks (see for instance \cite*{Haussler1992}, \cite*{Koiran1998}, \cite*{sontag:VC}, \cite*{AB1999}, \cite{Bartlett2017}, and \cite{Zhang2018} and references therein). 

However, it is important to emphasize that using this traditional learning theoretical approach  to formulate generalization error bounds in the case of reservoir systems  is challenging and requires non-trivial extensions of this circle of ideas. Indeed, since a key motivation for our analysis are time series applications, the standard i.i.d.\ assumption on inputs and outputs cannot be invoked anymore, which makes a number of conventional  tools  unsuitable. Here the signals to be treated are stochastic processes with a particular dependence structure, which introduces mathematical difficulties that only a few works have analyzed in a learning theory context. In most of the available contributions on learning in a non-i.i.d.\ setting,  stationarity and specific mixing properties of the input are key assumptions (see, for instance, \cite{McDonald2012,Kuznetsov2017,Kuznetsov2018} and the references therein). The time series applications that we are interested motivate us, however, to part with the latter.  A common argument in this direction (\cite{Kuznetsov2018}) is that many standard time-series processes happen to be non-mixing; for example, one can easily construct AR(1)  and  ARFIMA processes which are not mixing (see \cite{Andrews1983} and \cite{Baillie1996}, respectively). On the other hand, it has been pointed out (see  \cite*{Adams2010}) that the convergence of empirical quantities to population-based ones can be arbitrarily slow for general stationary data  and one cannot hope to obtain distribution-free probability bounds as they exist for the i.i.d.\ case. Motivated by these observations, in this article we restrict to a particular type of dependent processes, namely, we focus on dependence structures created by causal Bernoulli shifts (see, for instance, \cite{Dedecker2007a, alquier:wintenberger}) and hence the error bounds that we obtain are valid for any input process with such a dependence structure. Apart from trivially incorporating the i.i.d.\ case, the Bernoulli shifts category includes the VARMA time-series class of models, financial econometric models such as various GARCH specifications \citep*{engle:arch, bollerslev:garch, engleCorrelationsBook}, and  the ARFIMA   \citep*{Beran1994} processes that allow the modeling of long memory behavior exhibited by many financial time series (for example realized variances). As we show later on, even though the bounds that we obtain depend on the weak dependence assumptions, they hold true without making precise the distributions of the input and the outputs, which are generally unknown. We hence place ourselves in a {\bfi semi-agnostic setup}.

Regarding  complexity measures, bounds for them are also customarily formulated in an i.i.d.\ setting. Recently, some authors addressed the question of constructing versions of Rademacher complexities for dependent inputs. For example, if one defines the risk in terms of conditional expectations, then the so-called sequential Rademacher complexities can be used to derive bounds (see, for instance, \cite[Proposition~15]{AlexanderRakhlinKarthikSridharan2015} and \cite{Rakhlin2014}). In this paper we pursue a more traditional approach in terms of the definition of the expected risk and hence the associated Rademacher complexity. 

The main contribution of this paper is the formulation of {\it the first explicit generalization bounds for reservoir systems such as recurrent neural networks with input data exhibiting a sequential dependence structure for the classical notion of risk defined as an expected loss}. The uniform high-probability bounds which we state in this paper depend exclusively on the weak dependence behavior of the input and target processes and a quantitative measure of the capacity (Rademacher complexity) of the set of functionals generated by the considered reservoir systems.  The finite sample guarantees provided by our generalization bounds explicitly answer practical questions concerning the bounds for the parameters within a particular reservoir family, the rates of uniform convergence, and hence the length of the training sample required to achieve a desired learning generalization quality within a given RC class. 
Finally,  when one wishes to apply  empirical risk minimization (ERM) in order to pick the reservoir functional within the hypothesis class, the asymptotic behavior of the devised  bounds guarantees the consistency of ERM for reservoir systems. 
 



\medskip

The paper is organized as follows:
\begin{itemize}
\item Section~\ref{Preliminaries} describes the notation used in the paper. We introduce reservoir systems, the associated filters and functionals, as well as a detailed description of various families of popular reservoir systems in the literature. 
\item Section~\ref{Learning problem for reservoir computing systems} sets up the statistical learning problem for reservoir computing systems. It starts by introducing a general framework for the learning procedure, necessary risk definitions, and criteria of risk-consistency for the particular case of empirical risk minimization (ERM). The second subsection constitutes the main part of Section~\ref{Learning problem for reservoir computing systems} and we present in it the setting in which reservoir systems are analyzed in the rest of the paper. First, three alternative core assumptions regarding the weak dependence structure of the input and target processes are analyzed and illustrated with examples. Second, the hypothesis classes of reservoir maps and functionals are constructed under a set of mild assumptions. Finally, the strategy for the derivation of risk bounds for a given choice of loss function is discussed.
\item Section~\ref{Main Results} contains the main results of the paper. Proofs and auxiliary results are postponed to the appendices. Section~\ref{Main Results} is structured as follows. In the first subsection the expected value of the worst-case difference between the generalization and training errors over the class   of reservoir functionals is shown to be bounded by its Rademacher complexity and terms related to the weak dependence structure of the input and target processes. The obtained rates differ depending on the various assumptions invoked. The second subsection provides explicit expressions for upper bounds of the Rademacher complexities associated to the families of reservoir systems presented in Section~\ref{Preliminaries}. The third subsection concludes with the formulation of high-probability finite-sample generalization bounds for reservoir systems. We emphasize that previously such bounds were not available in the literature. The asymptotic behavior of these bounds shows in passing the weak risk-consistency of the ERM procedure for reservoir systems. The last subsection contains a result that provides high-probability bounds for families of reservoir systems whose reservoir maps have been generated randomly. This result is a theoretical justification of the well-known good empirical properties of this standard {\it modus operandi} in reservoir computing.
\end{itemize}

\section{Preliminaries}
\label{Preliminaries}
We start by specifying our notation and introducing reservoir computing systems for which in the following sections we will set up a statistical learning strategy. In the last subsection we provide a list of particular families of reservoir systems which are popular in the RC literature and in applications.
\subsection{Notation}
\label{Notation}
We use the symbol $\mathbb{N}$ (respectively, $\mathbb{N}^+$) to denote  the set of natural numbers with the zero element included  (respectively, excluded). $\mathbb{Z}$ denotes the set of all integers, and $\mathbb{Z}_-$ (respectively, $\mathbb{Z}_+$) stands for the set of the negative (respectively, positive) integers with the zero element included. Let $d,n,m \in \mathbb{N}^+$. Given an element $\mathbf{x} \in \mathbb{R}^n$, we denote by $\mathbb{R}[ \mathbf{x}]$ the real-valued multivariate polynomials on $ \mathbf{x}$ with real coefficients. Given a vector $\mathbf{v}\in \mathbb{R}^n$, the symbol $\| \mathbf{v}\|_2$ stands for its Euclidean norm. We denote by $\mathbb{M}_{m,  n }$ the space of real $m\times n$ matrices. When $n=m$, we use the symbol $\mathbb{M}_n $ to refer to the space of square matrices of order 
$n$. For any $A\in \mathbb{M}_{m,n}$, $\vertiii{A}_2$ denotes its matrix norm induced by the Euclidean norms in $\mathbb{R}^m$ and $\mathbb{R}^n$, which satisfies that $\vertiii{A}_2 =\sigma_{max}(A)$ with $\sigma_{max}(A)$ the largest singular value of $A$. $\vertiii{A}_2$ is sometimes referred to as the spectral norm of $A$ (see \cite{horn:matrix:analysis}).

When working in a deterministic setup, the inputs and outputs will be modeled using semi-infinite sequences ${\bf z}  \in (\mathbb{R}^d)^{\mathbb{Z}_-}$ and ${\bf y} \in (\mathbb{R}^m)^{\mathbb{Z}_-} $, respectively. We shall restrict very frequently to input sequences that exhibit additional convergence properties that are imposed with the help of weighting sequences. A weighting sequence $w$ is a strictly decreasing sequence with zero limit  $w : \mathbb{N} \longrightarrow (0,1] $ such that $w _0=1 $. We define the {\bf weighted $1$-norm} or the $(1,w)${\bf  -norm} 
$\left\|\cdot \right\| _{1,w} $ in the space of semi-infinite sequences $(\mathbb{R}^d)^{\mathbb{Z}_{-}} $ as 
\begin{equation*}
\label{pw norm}
\left\|{\bf z} \right\|_{1,w}:=  \sum_{t \in \mathbb{Z}_{-}} \left\|{\bf z}_t\right\|_2  w _{-t},\quad \mbox{for any} \quad {\bf z}\in (\mathbb{R}^d)^{\mathbb{Z}_{-}}.
\end{equation*}
We then set 
\begin{equation}
\label{space l1definition}
\ell_{-}^{1,w}(\mathbb{R}^d) :=  \left\{{\bf z} \in (\mathbb{R}^d)^{\mathbb{Z}_{-}}\mid \left\|{\bf z}\right\|_{1, w}< \infty\right\}.
\end{equation}
This weighted sequence space can be characterized as a Bochner space \citep{AnalysisBanachSpaces:vol1} by noticing that
\begin{equation}
\label{bochner for lpw}
\left(\ell_{-} ^{1,w}(\mathbb{R}^d), \left\|\cdot \right\|_{1,w}\right)=\left(L ^1(\mathbb{Z}_{-}, \mathcal{P}(\mathbb{Z}_{-}), \mu_w; \mathbb{R}^d), \left\|\cdot \right\|_{L ^1(\mathbb{Z}_{-};\mathbb{R}^d)}\right),
\end{equation}
where $\mathcal{P}(\mathbb{Z}_{-}) $ stands for the power set of $\mathbb{Z}_{-} $ and $\mu_w$ is the measure defined on $(\mathbb{Z}_{-},\mathcal{P}(\mathbb{Z}_{-})) $ generated by the assignments $\mu(\{t\}):=w_{-t} $, for any $t \in \mathbb{Z}_{-}$. This equality guarantees that the pair $\left(\ell_{-} ^{1,w}(\mathbb{R}^d), \left\|\cdot \right\|_{1,w}\right) $ forms a separable Banach space.  

Let now $\tau \in \mathbb{Z}_-$ and define the {\bf time delay operator} $T_{-\tau}: (\mathbb{R}^d)^{\mathbb{Z}_-} \longrightarrow (\mathbb{R}^d)^{\mathbb{Z}_-}$ by $T_{-\tau}({\bf z} ) _t:={\bf z} _{t+ \tau}$, for any $t \in \mathbb{Z}_-$. We call $T_{-\tau}({\bf z}) \in (\mathbb{R}^d)^{\mathbb{Z}_-}$ the $\tau${\bf -shifted 
version} of the  semi-infinite sequence ${\bf z}  \in (\mathbb{R}^d)^{\mathbb{Z}_-}$. It can be proved \citep*{RC9} that $T_{-\tau} $ restricts to a continuous linear operator in $\left(\ell_{-} ^{1,w}(\mathbb{R}^d), \left\|\cdot \right\|_{1,w}\right) $ and that the operator  norm of the resulting maps $T_{-\tau}: \ell_{-} ^{1,w}(\mathbb{R}^d) \longrightarrow \ell_{-} ^{1,w}(\mathbb{R}^d)$ satisfies
\begin{equation}
\label{operator norms tt}
\vertiii{ T _{1}} _{1,w}= L _w \mbox{ and } \quad \vertiii{ T _{-\tau}} _{1,w}\leq L _w^{- \tau},  \mbox{ for all $\tau \in \mathbb{Z}_{-} $},
\end{equation}
provided that the condition $L _w < \infty $ holds, where $1 \leq L _w \leq \infty$ is the {\bf inverse decay ratio} of $w$ defined as
\begin{equation*}
\label{inverse decay ratio}
L _w:=\sup_{t \in \mathbb{N}}\left\{\frac{w _t}{w _{t+1}}\right\}.
\end{equation*} 
We define for future reference the {\bfi  decay ratio} $D_w $ of $w$ as
\begin{equation}
\label{definition decay ratio}
 D_{w}:=\sup_{t \in \mathbb{N}} \left\{\frac{w_{t+1}}{w_t}\right\}\leq1.
\end{equation}
The other weighted norm of much use in the context of reservoir computing is the $(\infty,w)${\bf  -norm}, defined by 
\begin{equation*}
\| {\bf z} \| _{\infty,w}:= \sup_{t \in \mathbb{Z}_-}\{\| {\bf z}_t\|_2 w_{-t}\},\quad \mbox{for any} \quad {\bf z}\in ({\mathbb{R}}^d)^{\mathbb{Z}_{-}}.
\end{equation*}
We then set 
$
\ell_{-}^{\infty,w}(\mathbb{R}^d):=  \left\{{\bf z} \in (\mathbb{R}^d)^{\mathbb{Z}_-} \mid \left\|{\bf z}\right\|_{\infty, w}< \infty\right\}$. It can also be showed that the pair $(\ell_{-}^{\infty,w}(\mathbb{R}^d), \left\|\cdot \right\| _{\infty,w})$ is  a Banach space \citep*{RC7}. Additionally, the time delay operators also restrict to $\ell_{-}^{\infty,w}(\mathbb{R}^d) $ and the corresponding operator norms  satisfy \eqref{operator norms tt}.


\subsection{Filters and reservoir computing systems}
\label{Filters and reservoir computing systems}
The objects at the core of this paper are input/output maps of the form $U: \left({\mathbb{R}}^d\right)^{\mathbb{Z}} \longrightarrow\left({\mathbb{R}}^m\right)^{\mathbb{Z}} $. We will restrict to the case in which the maps $U$ are {\bf causal} and {\bf time-invariant} (see \cite*{RC9} for definitions and the proofs of the facts that we now state) and hence it suffices to work with the restrictions $U: \left({\mathbb{R}}^d\right)^{\mathbb{Z}_-} \longrightarrow\left({\mathbb{R}}^m\right)^{\mathbb{Z}_-} $. Moreover, causal and time-invariant filters $U$ uniquely determine functionals of the type $H _U: \left({\mathbb{R}}^d\right)^{\mathbb{Z}_-} \longrightarrow {\mathbb{R}}^m $ by 
\begin{equation*}
\label{definition of functional}
H _U ({\bf z}):= U({\bf z}) _0, \quad \mbox{for any} \quad {\bf z} \in  \left({\mathbb{R}}^d\right)^{\mathbb{Z}_-}.
\end{equation*}
In this setup, we shall say that $U: \left({\mathbb{R}}^d\right)^{\mathbb{Z}_-} \longrightarrow\left({\mathbb{R}}^m\right)^{\mathbb{Z}_-} $ is a {\bf filter} and that $H _U: \left({\mathbb{R}}^d\right)^{\mathbb{Z}_-} \longrightarrow {\mathbb{R}}^m $ is its corresponding {\bf functional}. Conversely, given a  functional $H: \left({\mathbb{R}}^d\right)^{\mathbb{Z}_-} \longrightarrow {\mathbb{R}}^m $, there is a unique causal and time-invariant filter $U _H :\left({\mathbb{R}}^d\right)^{\mathbb{Z}_-} \longrightarrow\left({\mathbb{R}}^m\right)^{\mathbb{Z}_-} $ determined by it as
\begin{equation}
\label{definition of filter associated}
U _H ({\bf z})_t:= H(T_{-t}({\bf z})), \quad \mbox{for any} \quad {\bf z} \in  \left({\mathbb{R}}^d\right)^{\mathbb{Z}_-},\, t \in \mathbb{Z}_{-}.
\end{equation}

Suppose that given a weighting sequence $w$, the filter $U$ restricts to a map between weighted $(\infty, w)$-spaces, that is, $U: \ell_{-}^{\infty,w}(\mathbb{R}^d) \longrightarrow\ell_{-}^{\infty,w}(\mathbb{R}^m) $ and that, additionally, $U$ is continuous with respect to the norm topology in those spaces. In that case we say that $U$ has the {\bf fading memory property (FMP)} with respect to $w$. 

We shall provide an answer to the supervised learning of filters by estimating approximants built as {\bf reservoir filters}. Reservoir filters are obtained out of a {\bf reservoir system}, that is, a state-space system made out of two recurrent equations of the form:
\begin{equation} 
\label{eq:RCSystemDet}
\left\{
\begin{array}{rcl}
\mathbf{x}_t &=&F(\mathbf{x}_{t-1}, {\bf z}_t), \\
{\bf y}_t&= & h (\mathbf{x} _t),
\end{array}
\right.
\end{equation}
for all $t \in \mathbb{Z}_- $ and  where $F \colon D_N\times D_d\longrightarrow  D_N$ and $h: D_N \longrightarrow {\mathbb{R}}^m $ are maps, $D_d \subset \mathbb{R}^d$, $D_N \subset \mathbb{R}^N$. The sequences  ${\bf z} \in (D_d)^{\mathbb{Z}_-}$ and ${\bf y} \in (\mathbb{R}^m)^{\mathbb{Z}_-}$ stand for the {\bfi  input} and the {\bfi  output (target)} of the system, respectively, and $\mathbf{x} \in  (D_N)^{\mathbb{Z}_-}$ are the associated {\bfi  reservoir states}.

A reservoir system determines a filter when the first equation in \eqref{eq:RCSystemDet} satisfies the so-called {\bf echo state property (ESP}), that is, when for any  ${\bf z} \in (D_d)^{\mathbb{Z}_-}$ there exists a unique $\mathbf{x} \in (D_N)^{\mathbb{Z}_-}$ such that \eqref{eq:RCSystemDet} holds. In that case, we  talk about the reservoir filter $U ^F_h: (D_d)^{\mathbb{Z}_-}\longrightarrow  ({\mathbb{R}}^m)^{\mathbb{Z}_-} $ associated to the reservoir system \eqref{eq:RCSystemDet} that is defined by: 
\begin{equation*}
U ^F_h:= h \circ U ^F, \quad \mbox{where} \quad U  ^F({\bf z}):= \mathbf{x},
\end{equation*}
with ${\bf z} \in (D_d)^{\mathbb{Z}_-} $  and $\mathbf{x} \in (D_N)^{\mathbb{Z}_-} $ linked by the first equation in \eqref{eq:RCSystemDet} via the ESP.
It is easy to show that reservoir filters are automatically causal and time-invariant (see \cite[Proposition 2.1]{RC7}) and hence determine a reservoir functional $H ^F_h: (D_d)^{\mathbb{Z}_-}\longrightarrow  {\mathbb{R}}^m$.

As the following Proposition shows, a sufficient condition guaranteeing that the echo state property holds is that $D_N$ is a closed ball and that the map $F$ is continuous and a contraction in the first argument. 

\begin{proposition} 
\label{prop:contractionEchoState} 
Let $S>0$, $\overline{B_S}=\{{\bf x} \in \mathbb{R}^N \colon \|{\bf x}\|_2\leq S \}$ and suppose that $F \colon \overline{B_S} \times D_d \to \overline{B_S}$ is continuous. Assume that $F$ is a contraction in the first argument, that is, there exists $0<r<1$  such that for all ${\bf x}_1,{\bf x}_2 \in \overline{B_S}$, ${\bf z} \in D_d$ it holds that
\begin{equation}\label{eq:Fcontractive} \|F({\bf x}_1,{\bf z})- F({\bf x}_2,{\bf z}) \|_2 \leq r \|{\bf x}_1-{\bf x}_2\|_2. \end{equation}
Then the system \eqref{eq:RCSystemDet} has the echo state property and hence its first equation determines a unique causal and time-invariant filter $U  ^F: (D_d)^{\mathbb{Z}_-}\longrightarrow  (\overline{B_S})^{\mathbb{Z}_-} $ as well as a functional $H ^F\colon (D_d)^{\mathbb{Z}_-} \to \overline{B_S}$ that are continuous (where both $(D_d)^{\mathbb{Z}_-}$ and $(\overline{B_S})^{\mathbb{Z}_-}$ are equipped with the product topologies). 

\end{proposition}

\begin{remark}
\label{rmk:boundedImage} 
\normalfont If we have a continuous function $F \colon \mathbb{R}^N \times D_d \to \mathbb{R}^N$ which satisfies \eqref{eq:Fcontractive} for all ${\bf x}_1,{\bf x}_2 \in \mathbb{R}^N$ and $\|F({\bf 0},{\bf z})\|_2 \leq c$ for all  ${\bf z} \in D_d$ and a certain $c >0 $, then choosing any $S\geq c/(1-r)$ it is easy to see that for all ${\bf u} \in \overline{B_S}$ 
\[\begin{aligned} \|F({\bf u},{\bf z})\|_2 & \leq \|F({\bf u},{\bf z})-F({\bf 0},{\bf z})\|_2+\|F({\bf 0},{\bf z})\|_2  \leq r \|{\bf u}\|_2 + c \leq r S +c \leq S.
\end{aligned} \]  
Thus, $F(\overline{B_S} \times D_d) \subset \overline{B_S}$ and so the restriction of $F$ to $\overline{B_S} \times D_d$ satisfies the assumptions of Proposition~\ref{prop:contractionEchoState}.
\end{remark}

\subsection{Families of reservoir systems}
\label{Families of reservoir systems}

The following paragraphs introduce various families of reservoir systems that appear in applications. We shall later on explicitly construct for these specific families the risk bounds contained in the main results of the paper.

\subsubsection*{Reservoir systems with linear reservoir maps (LRC)} 
In this case one associates to each input signal ${\bf z} \in (D_d)^{\mathbb{Z}_-}$ an output ${\bf y} \in ({\mathbb{R}}^m)^{\mathbb{Z}_-}$ via the two recurrent equations 
\begin{align} 
\label{eq:LRCstate}
 \mathbf{x}_t &= A \mathbf{x}_{t-1} + C {\bf z}_t + \bm{\zeta},\\
 \label{eq:LRCreadout}
 \mathbf{y}_t &= h(\mathbf{x}_{t}),
 \end{align}
with $t \in \mathbb{Z}_- $ and  $A \in \mathbb{M}_{N}, C \in \mathbb{M}_{N,d}, \bm{\zeta} \in \mathbb{R}^N$. Systems with linear reservoir maps of the type \eqref{eq:LRCstate} have been vastly studied in the literature in numerous contexts and under different denominations. In the RC setting, systems of the type \eqref{eq:LRCstate}-\eqref{eq:LRCreadout} with polynomial readout maps $h: D_N \longrightarrow {\mathbb{R}}^m $ have been proved in \cite{RC6} to be universal approximators in the category of fading memory filters either when presented with uniformly bounded inputs in the deterministic setup (see Corollary 3.4) or with almost surely uniformly bounded stochastic inputs (see Corollary 4.8). These boundedness hypotheses have been dropped in \cite{RC8} by considering density with respect to $L ^p $  norms, $1\leq p< \infty $, defined using the law of the input data generating process. Sufficient conditions which ensure the echo state property and the fading memory property for these systems have been  established (see   Section 3.1 in \cite{RC9}). More specifically,  consider the reservoir map $F^{A,C, \bm{\zeta}}: D_N \times D_d \longrightarrow D_N$ of the system \eqref{eq:LRCstate}-\eqref{eq:LRCreadout} given by 
\begin{equation}
\label{eq:FLRC} 
F^{A,C,\bm{\zeta}}(\mathbf{x},{\bf z})=  A\mathbf{x} + C{\bf z} + \bm{\zeta}.
\end{equation}
It is easy to see that $F^{A,C, \bm{\zeta}}$ is a contraction in the first entry whenever the matrix $A$ satisfies $\vertiii{A}_2<1$. For these systems we  consider only  the case of uniformly bounded input signals:  

\medskip
\noindent{\bf Case with uniformly bounded inputs}. Suppose now $\vertiii{A}_2<1$. If the inputs are uniformly bounded, that is, if $D_d = \overline{B_M}$ for some $M>0$ and so ${\bf z} \in K_M$ with $K_M := \left\{ {\bf z} \in (\mathbb{R}^d)^{ \mathbb{Z}_-} | \| {\bf z}_t \|_2 \leq M \ {\rm for \ all} \  t\in \mathbb{Z}_- \right\}$, then  the reservoir system \eqref{eq:LRCstate}-\eqref{eq:LRCreadout} has the echo state property and defines a unique causal and time-invariant reservoir filter $U  ^{A,C,\bm{\zeta}}: K_M\longrightarrow (D_N)^{ \mathbb{Z}_-}  $ given by $U  ^{A,C,\bm{\zeta}}({\bf z})_t:=\sum_{j=0} ^{\infty}A ^j (C {\bf z}_{t-j}+ \boldsymbol{\zeta}) $, $t  \in \mathbb{Z}_{-}  $. Here $D_N = \overline{B_{M_F}}$ with  $M_F = ({\vertiii{C}_2M + \|\bm{ \zeta} \|_2})/({1-\vertiii{A}_2})$ (see Remark 2 and part {\bf (ii)} in the first example in Section 4.1 of \cite{RC9}). In particular, the corresponding functional $H ^{A,C,\bm{\zeta}}\colon K_M \to D_N$ satisfies that $\|H ^{A,C,\bm{\zeta}}( {\bf z})\|_2\leq M_F$ for all ${\bf z} \in  K _M  $. Additionally, it can be shown that the reservoir system \eqref{eq:LRCstate}-\eqref{eq:LRCreadout} has the fading memory property with respect to any weighting sequence. 

In what follows we consider a particular subfamily of systems \eqref{eq:LRCstate}-\eqref{eq:LRCreadout}, namely {\it reservoir systems with linear reservoir and linear readout maps}, in which case  $h: D_N \longrightarrow {\mathbb{R}}^m $ is given by applying $W\in \mathbb{M}_{m,N}$.  

\subsubsection*{Echo State Networks (ESN)}
{Echo State Networks} (\cite{Matthews:thesis, Jaeger04}) are a family of reservoir systems that exhibit excellent performance in many practical applications and have been recently proved to have universal approximation properties. More specifically, \cite{RC7} proved  ESNs to be universal in the category of fading memory  filters with semi-infinite uniformly bounded inputs in a deterministic setup, and \cite{RC8} obtained universality results for ESNs in the stochastic situation with respect to $L^p$-type criteria for stochastic discrete-time semi-infinite inputs. 

An echo state network of dimension  $N \in \mathbb{N}^+$ with  reservoir matrix $A \in \mathbb{M}_{N}$, input mask $C \in \mathbb{M}_{N,d}$, input shift $\bm{\zeta} \in \mathbb{R}^N$, and readout matrix $W\in \mathbb{M}_{m,N}$ is the system 
\begin{align} 
\label{eq:ESNstate}
 \mathbf{x}_t &= \bm{\sigma}( A \mathbf{x}_{t-1} + C {\bf z}_t + \bm{\zeta}),\\
 \label{eq:ESNreadout}
  \mathbf{y}_t &= W  \mathbf{x}_t,
 \end{align}
  which 
 for each $t \in \mathbb{Z}_-$ transforms the input ${\bf z}_t \in D_d \subset \mathbb{R}^d$ into the reservoir state ${\bf x}_t \in D_N\subset \mathbb{R}^N$ and, consequently, into the corresponding output ${\bf y}_t \in \mathbb{R}^m$. 
The reservoir map $F^{\sigma,A,C, \bm{\zeta}}: D_N \times D_d \longrightarrow D_N$ of the system \eqref{eq:ESNstate}-\eqref{eq:ESNreadout} is given by 
\begin{equation}
\label{eq:FESN} 
F^{\sigma,A,C, \bm{\zeta}}(\mathbf{x},{\bf z})=  \bm{\sigma}(A\mathbf{x} + C{\bf z} + \bm{\zeta}), 
\end{equation}
where  $\bm{\sigma} \colon \mathbb{R}^N \to \mathbb{R}^N$ is defined by the componentwise application of a given {\bfi  activation function} $\sigma \colon \mathbb{R} \to \mathbb{R}$. Throughout, we assume that $\sigma$ is Lipschitz-continuous with Lipschitz-constant $L_\sigma$. It is straightforward to verify that $F^{\sigma,A,C, \bm{\zeta}}$ is a contraction in the first entry whenever $L_\sigma\vertiii{A}_2<1$. The sufficient conditions which ensure the echo state and the fading memory properties of \eqref{eq:ESNstate}-\eqref{eq:ESNreadout} have been also carefully studied in the literature (see  \cite{Buehner:ESN}, \cite{Yildiz2012}, and \cite{RC7} for details) and depend both on the type of the activation function $\sigma \colon \mathbb{R} \to \mathbb{R}$ and on the type of the input presented to the network. We consider the following two cases:
 
\medskip
\noindent{\bf Case with arbitrary input signals and bounded activation function}.
 In this situation, $D_d$ is arbitrary and so generic input signals ${\bf z} \in (D_d)^{ \mathbb{Z}_-}$ are considered, but we assume that the range of the activation function $\sigma$ is bounded and contained in $[\sigma_{min},\sigma_{max}]$ with $\sigma_{min}<\sigma_{max}\in \mathbb{R}$. Then, by Proposition~\ref{prop:contractionEchoState}, the condition $L_\sigma\vertiii{A}_2<1$ suffices to ensure that the system \eqref{eq:ESNstate}-\eqref{eq:ESNreadout} has the echo state property and hence defines a unique causal and time-invariant filter $U^{\sigma,A,C, \bm{\zeta}}: (D_d)^{\mathbb{Z}_-}\longrightarrow  (D_N)^{\mathbb{Z}_-} $ as well as a functional $H^{\sigma,A,C, \bm{\zeta}}\colon (D_d)^{\mathbb{Z}_-} \to D_N$ that are additionally continuous with respect to the product topologies on the spaces $(D_d)^{\mathbb{Z}_-}$ and $(D_N)^{\mathbb{Z}_-}$. Here $D_N = \overline{B_{M_F}}$ with  $M_F = \sqrt{N} \max(| \sigma_{min}|, | \sigma_{max}|)$ and in particular, we obviously have that $\|H ^{\sigma, A,C,\bm{\zeta}}( {\bf z})\|_2\leq M_F$.
 
\medskip
 \noindent{\bf Case with uniformly bounded inputs}. Suppose that $L_\sigma\vertiii{A}_2<1$ and $D_d = \overline{B_{M}}$ for some $M>0$ and so the inputs are uniformly bounded, that is, ${\bf z} \in K_M$ with $K_M := \left\{ {\bf z} \in (\mathbb{R}^d)^{ \mathbb{Z}_-} | \| {\bf z}_t \|_2 \leq M \ {\rm for \ all} \  t\in \mathbb{Z}_- \right\}$. In this case the reservoir system \eqref{eq:ESNstate}-\eqref{eq:ESNreadout} has the echo state property and defines a unique causal and time-invariant reservoir filter $U  ^{\sigma, A,C,\bm{\zeta}}: K_M\longrightarrow  (D_N)^{ \mathbb{Z}_-} $ as well as a functional $H ^{\sigma, A,C,\bm{\zeta}}\colon K_M \to D_N$, where $D_N = \overline{B_{M_F^1}}$ with $M_F^1 := [L_\sigma({\vertiii{C}_2 M + \|\bm{ \zeta} \|_2})+\sqrt{N}\sigma(0)]/({1-L_\sigma\vertiii{A}_2})$ (see Proposition~\ref{prop:contractionEchoState} and Remark~\ref{rmk:boundedImage} or part {\bf (ii)} in the first example in Section 4.1 of \cite{RC9}). Additionally, the fading memory property holds with respect to any weighting sequence. Note that these results hold true even though the range of the activation function is not assumed to be bounded and $\|H ^{\sigma,A,C,\bm{\zeta}}( {\bf z})\|_2\leq M_F $ with $M_F=M_F^1$ when the activation function $\sigma$ has an unbounded range and with $M_F = \min(M_F^1,\sqrt{N} \max \left(| \sigma_{min}|, | \sigma_{max}| \right))$, otherwise.

 \subsubsection*{State-Affine Systems (SAS)}
The so-called homogeneous state-affine systems have been first introduced in the systems theory literature and were shown to exhibit universality properties in the discrete-time setting for compact times (see \cite{FliessNormand1980}, \cite{Sontag1979}, \cite{sontag:polynomial:1979}). A non-homogeneous version of these systems was introduced in \cite{RC6}, where they were proved to be universal approximants in the category of fading memory filters for the non-compact discrete-time deterministic setup. Trigonometric state-affine systems were later on studied in a stochastic setup in \cite{RC8}, where their universality for stochastic discrete-time semi-infinite inputs  with respect to $L^p$-criteria was established. State-affine systems serve as an excellent example of reservoir systems with easy-to-train linear readouts and even though little is known about their empirical performance in learning tasks, we find it important to provide explicit risk bounds for this family. In the rest of the paper  we reserve the name {State-Affine Systems (SAS)} for the non-homogeneous version if not stated otherwise and leave the trigonometric family for future work.

The following notation for multivariate polynomials will be used: for any multi-index $\bm{\alpha} \in \mathbb{N}^d$ and any ${\bf z}\in \mathbb{R}^d$, we write ${\bf z}^{\bm{\alpha}}:=z_1^{\alpha_1} \cdots z_d^{\alpha_d}$. Furthermore, the space $ \mathbb{M}_{N,M}[{\bf z}]$, $N,M \in \mathbb{N}^+ $, of polynomials in the variable ${\bf z} \in \mathbb{R} ^d$ with matrix coefficients in $\mathbb{M}_{N,M}$ is the set of elements $p$ of the form
\[ p({\bf z}) = \sum_{\bm{\alpha} \in V _p} {\bf z}^{\bm{\alpha}} A_{\bm{\alpha}}, \quad {\bf z} \in \mathbb{R} ^d, \]
where $V _p \subset \mathbb{N} ^d $ is a finite subset and the elements
$A_{\bm{\alpha}} \in \mathbb{M}_{N,M}$ are matrix coefficients. The degree ${\rm deg}(p) $ of the polynomial $p$ is defined as 
\begin{equation*}
{\rm deg}(p)=\max_{\boldsymbol{\alpha} \in V _p} \left\{\left\|\boldsymbol{\alpha}\right\|_1\right\}, \, \text{where} \, \left\|\boldsymbol{\alpha}\right\|_1:=\alpha _1+ \cdots+ \alpha _d.
\end{equation*}
We also define the following norm on $\mathbb{M}_{N,M}[{\bf z}]$:
\begin{equation}
\label{norm for later sas}
\vertiii{p}  = \max_{\bm{\alpha} \in V _p} \vertiii{A_{\bm{\alpha}}}_2. 
\end{equation}
The non-homogeneous state-affine system (SAS) of dimension $N\in \mathbb{N}^+$ associated to two given polynomials $p \in \mathbb{M}_{N,N}[{\bf z}]$ and $q \in \mathbb{M}_{N,1}[{\bf z}]$ with matrix and vector coefficients, respectively, is the reservoir system determined by the following state-space
transformation of each input signal ${\bf z} \in (D_d)^{\mathbb{Z}_-}$ into the output signal ${\bf y} \in ({\mathbb{R}}^m)^{\mathbb{Z}_-}$,
\begin{align} 
\label{eq:SASstate}
 \mathbf{x}_t &= p({\bf z}_t)\mathbf{x}_{t-1} + q({\bf z}_t),\\
 \label{eq:SASreadout}
  \mathbf{y}_t &= W  \mathbf{x}_t,
 \end{align}
 for $t \in \mathbb{Z}_-$, with $W\in \mathbb{M}_{m,N}$ the readout map. 
The reservoir map $F^{p,q}: D_N \times D_d \longrightarrow D_N$ of the system \eqref{eq:SASstate}-\eqref{eq:SASreadout} is given by 
\begin{equation}
\label{eq:FSAS} 
F^{p,q}(\mathbf{x},{\bf z})=  p({\bf z})\mathbf{x} + q({\bf z}). 
\end{equation}
Additionally, we define 
\begin{align*}
M_p&:=\sup_{{\bf z}\in D_d} \vertiii{p( {\bf z})}_2,\\
M_q&:=\sup_{{\bf z}\in D_d} \vertiii{q( {\bf z})}_2.
\end{align*}
First, we notice that for regular SAS defined by nontrivial polynomials, the set $D_d$ needs to be bounded in order for $M_p$ and $M_q$ to be finite. It is easy to see that $F$ in \eqref{eq:FSAS} is a contraction in the first entry with constant $M_p$ whenever $M_p < 1$, which is a condition that we will assume holds true together with $M_q<\infty$ in the next paragraph.

\medskip
 \noindent{\bf Case with uniformly bounded input signals}. Let $D_d = \overline{B_{M}}$ for some $M>0$ so that we consider inputs ${\bf z} \in K_M$ with $K_M := \left\{ {\bf z} \in (\mathbb{R} ^d)^{ \mathbb{Z}_-} | \| {\bf z}_t \|_2 \leq M \ {\rm for \ all} \  t\in \mathbb{Z}_- \right\}$. In that case the system \eqref{eq:SASstate}-\eqref{eq:SASreadout} has the echo state property and determines (see Proposition~\ref{prop:contractionEchoState} and Remark~\ref{rmk:boundedImage} or part {\bf (ii)} in the third example in Section 4.1 of \cite{RC9}) a unique reservoir filter $U  ^{p,q}: K _M \longrightarrow  (D_N)^{ \mathbb{Z}_-} $ as well as a functional $H ^{p,q}\colon K _M \to D_N$, where $D_N = \overline{B_{M_F}}$ with $M_F = M_q/(1-M_p)$. In addition, the fading memory property holds with respect to any weighting sequence. Moreover, in this case the filter can be explicitly written as $(U  ^{p,q}({\bf z}))_t = \sum^{\infty}_{j = 0} (\prod_{k=0}^{j-1} p( {\bf z}_{t-k}))q({\bf z}_{t-j})$, $t \in \mathbb{Z}_-$,  and $\|H ^{p,q}( {\bf z})\|_2\leq M_F$, for all ${\bf z} \in K _M$.

\section{The learning problem for reservoir computing systems}
\label{Learning problem for reservoir computing systems}
In this paper we work in the setting of supervised learning in a probabilistic framework and our goal is to provide performance estimates for reservoir systems from the statistical learning theory perspective. With that in mind, we start  this section by stating the general learning problem for  systems with stochastic input and target signals. We then introduce three alternative assumptions on the weak dependence of input and output processes which will be assumed later on in the paper and provide examples of important time series models that satisfy the conditions under consideration.  We define the statistical risk and its empirical analogs for reservoir functionals and motivate the need for generalization error bounds. More specifically, on the one hand the in-class generalization error  (risk) can be used to bound the estimation error of a class. On the other hand, whenever the learner follows the empirical risk minimization (ERM) strategy to select the reservoir computing system within the RC hypothesis class based on minimization of the empirical (training) error,  generalization error bounds can be used to prove the weak universal risk-consistency of ERM for reservoir systems. If the inputs are  i.i.d. (which is a particular case of our setup), this definition is essentially equivalent to saying that  the hypothesis class of reservoir functionals is a (weak) uniform Glivenko-Cantelli class (see for example \cite{Mukherjee2002}).

\subsection{General setup of the learning procedure} 
\label{General setup of empirical risk minimization procedure}

\paragraph{Input and target stochastic processes.} We fix a probability space $(\Omega,\mathcal{A},\mathbb{P})$ 
on which all random variables are defined. The triple consists of the sample space $\Omega$, which is the set of possible outcomes, the $\sigma$-algebra $\mathcal{A}$ (a set of subsets of $\Omega$ (events)),  and a probability measure $\mathbb{P}:\mathcal{A}\longrightarrow [0,1]$. The input and target signals are modeled by discrete-time stochastic processes ${\bf Z} = ({\bf Z}_t)_{t \in \mathbb{Z}_-}$ and ${\bf Y} = ({\bf Y}_t)_{t \in \mathbb{Z}_-}$ taking values in $D_d \subset \mathbb{R}^d$ and $\mathbb{R}^m$, respectively. Moreover, we write ${\bf Z}(\omega)=({\bf Z}_t( \omega))_{t \in \mathbb{Z}_-}$ and ${\bf Y}(\omega)=({\bf Y}_t( \omega))_{t \in \mathbb{Z}_-}$ for each outcome $\omega \in \Omega$ to denote the realizations or sample paths of ${\bf Z}$ and ${\bf Y}$, respectively. Since  ${\bf Z}$ can be seen as a random sequence in $D_d \subset \mathbb{R}^d$, we write interchangeably ${\bf Z}:{\mathbb{Z}}_{-} \times \Omega \longrightarrow D_d$ and ${\bf Z}: \Omega\longrightarrow (D_d)^{{\mathbb{Z}}_{-}}$. The latter is necessarily measurable
with respect to the Borel $\sigma $-algebra induced by the product topology in $(D_d)^{{\mathbb{Z}}_{-}}$.
The same applies to the analogous assignments involving ${\bf Y}$.

\paragraph {Hypothesis class $\mathcal{H}$, loss functions,  statistical, and empirical risk.} Let $\mathcal{F}$ be the class of all measurable functionals $H \colon (D_d)^{\mathbb{Z}_-} \longrightarrow {\mathbb{R}}^m$, $D_d\subset \mathbb{R}^d$, that is  $\mathcal{F}:= \{ H \colon (D_d)^{\mathbb{Z}_-} \longrightarrow {\mathbb{R}}^m \mid H \text{ is measurable} \}$. Consider a smaller {\bfi  hypothesis class} $\mathcal{H}$ of admissible functionals $\mathcal{H} \subset \mathcal{F}$. For a fixed {\bfi  loss}, that is, a measurable function\footnote{It is customary in the literature to consider nonnegative loss functions. This automatically guarantees that the expectation in \eqref{eq:riskDef} is well-defined,  although it is not necessarily finite. In this paper, for the sake of mathematical convenience, we allow for general real-valued loss functions but carefully address technical questions where relevant.} $L \colon \mathbb{R}^m \times \mathbb{R}^m \to \mathbb{R}$ and for any functional $H\in \mathcal{F}$  we define the {\bfi statistical risk} (sometimes just referred to as {\bfi  risk}) or {{\bfi  generalization error}} associated with $H$ as
\begin{equation}
\label{eq:riskDef} 
R(H):= \mathbb{E}[L(H({\bf Z}),{\bf Y}_0)],
\end{equation}
where by definition the expectation is taken with respect to the joint law of $({\bf Z},{\bf Y})$. The ultimate goal of the learning procedure consists in determining the {\bfi  Bayes
functional} $H^\ast_{\mathcal{F}} \in \mathcal{F}$ that exhibits the minimal statistical risk ({\bfi  Bayes risk}) in the class of all measurable functionals, which we  denote as 
\begin{equation}
\label{eq:riskMin} 
R_{\mathcal{F}}^\ast :=R(H^\ast_\mathcal{F})= \inf_{H \in \mathcal{F} } R(H).
\end{equation}
Even though this task is generally infeasible, one may hope to solve it for the {\bfi  best-in-class} functional $H^\ast_{\mathcal{H}} \in \mathcal{H}$ with the minimal associated in-class statistical risk ({\bfi  Bayes in-class risk}), which is assumed achievable, and which we denote as 
\begin{equation*}
R_{\mathcal{H}}^\ast :=R(H^\ast_\mathcal{H})= \inf_{H \in \mathcal{H} } R(H). 
\end{equation*}
The standard learning program is then based on the following error decomposition. For any $H\in \mathcal{H}$ we can write that
\begin{equation*}
R(H) - R_{\mathcal{F}} ^\ast = (R(H) -  R_{\mathcal{H}} ^\ast ) + (R_{\mathcal{H}} ^\ast - R_{\mathcal{F}} ^\ast),
\end{equation*}
where the first term is called the {\bfi  estimation error} and the second one is the {\bfi  approximation error}. In this paper we focus on upper bounds of the estimation component, while the same problem for the approximation error will be treated in the forthcoming work \cite{RC12}. 
We emphasize  that since the universal approximation properties of reservoir systems have been established in numerous situations (see the introduction and Section~\ref{Families of reservoir systems}) the  approximation error can be made arbitrarily small  by choosing an appropriate hypothesis class $\mathcal{H}$.
 
The distribution of $({\bf Z},{\bf Y})$ is generally unknown, and hence computing the risks \eqref{eq:riskDef} or \eqref{eq:riskMin}  is in practice infeasible. This implies, in particular, that the estimation error cannot be explicitly evaluated. Therefore, the usual procedure is in this case to use an empirical counterpart for \eqref{eq:riskDef} which can be computed using a training dataset. 

Suppose that a training sample for both the input and the target discrete-time stochastic processes is available up to some $n \in \mathbb{N}^+$ steps into the past, namely $( {\bf Z}_{-i}, {\bf Y}_{-i})_{i\in \left\{ 0, \dots, n-1\right\}}$. For each time step $i\in \{ 0, \dots, n-1\}$ we define the {\bfi  truncated training sample} for the input stochastic process $ {\bf Z}$ as
\begin{equation} \label{eq:truncatedZ}
{\bf Z}_{-i}^{-n+1} := (\ldots,{\bf 0},{\bf 0},{\bf Z}_{-n+1},\ldots,{\bf Z}_{-i-1},{\bf Z}_{-i}).
\end{equation}
In this time series context the {\bfi  training error} or the {\bfi  empirical risk} analog $\widehat{R}_n(H) $ of \eqref{eq:riskDef} is  given by
\begin{equation}
\label{eq:empiricalRiskDef}
\widehat{R}_n(H) = \frac{1}{n} \sum_{i=0}^{n-1} L(H({\bf Z}_{-i}^{-n+1}),{\bf Y}_{-i})
= \frac{1}{n} \sum_{i=0}^{n-1} L(U_H({\bf Z}_{0}^{-n+1})_{-i},{\bf Y}_{-i}),
\end{equation}
where $U_H$ denotes the filter associated to the functional $H$ as introduced in \eqref{definition of filter associated}.
In what follows we will also make use of what we call its {\bfi   idealized empirical risk}  version defined as
\begin{equation}
\label{eq:empiricalRiskFullHistoryDef}
\widehat{R}_n^\infty(H) = \frac{1}{n} \sum_{i=0}^{n-1} L(H({\bf Z}_{-i}^{-\infty}),{\bf Y}_{-i})
= \frac{1}{n} \sum_{i=0}^{n-1} L(U_H({\bf Z}_{0}^{-\infty})_{-i},{\bf Y}_{-i}),
\end{equation}
which makes use of a larger training sample containing all the past values of the input process  ${\bf Z}$.

\begin{remark}
\normalfont
The results of this paper are also valid if one replaces the zero elements  in the truncated training sample \eqref{eq:truncatedZ} by an arbitrary sequence (deterministic, random, or dependent on the training sample). More specifically, consider an arbitrary function $\mathcal{I} \colon (D_d)^{\mathbb{Z}_-} \to (D_d)^{\mathbb{Z}_-}$ that we use to extend the input training sample, for each $i\in \left\{ 0, \dots, n-1\right\}$, as
\begin{equation}\label{eq:extension}
\widetilde{{\bf Z}}_{-i}^{-n+1} = (\ldots,(\mathcal{I}({\bf Z}_{0}^{-n+1}))_{-1},(\mathcal{I}({\bf Z}_{0}^{-n+1}))_{0},{\bf Z}_{-n+1},\ldots,{\bf Z}_{-i-1},{\bf Z}_{-i}),
\end{equation}
and use this sample to define the empirical risk as in \eqref{eq:empiricalRiskDef}.
Later on in Proposition~\ref{prop:finiteHistoryError}, we show that the difference between the empirical risk \eqref{eq:empiricalRiskDef} and its idealized counterpart  \eqref{eq:empiricalRiskFullHistoryDef} can be made arbitrarily small under various assumptions that we shall consistently invoke. The proof of that result remains valid  for the more general definition of empirical risk using \eqref{eq:extension}. Moreover, that result is used to justify why, in the rest of the paper, it will be sufficient to work almost exclusively with the idealized empirical risk \eqref{eq:empiricalRiskFullHistoryDef}.
\end{remark}

\paragraph {Risk bounds and risk-consistency.} As we have already discussed, the learner is interested in obtaining upper bounds of the estimation error $(R(H) -  R_{\mathcal{H}} ^\ast )$. In many cases, these bounds can be constructed by bounding  the statistical risk (generalization error) or 
\begin{equation*}
\Delta_n:=\sup_{H\in \mathcal{H}}\{ {R(H) - \widehat{R}_n(H)} \}.
\end{equation*}
An upper bound of $\Delta_n$ allows to quantify the worst in-class error between the statistical risk and its empirical analog. We emphasize that bounding this worst in-class error  gives guarantees of performance for any learning algorithm which builds upon the idea of using the empirical risk to pick a concrete $\widehat{H}_n$ out of the hypothesis class $\mathcal{H}$ based on available training data. 

A standard example of such learning rules is the so-called {\bfi  empirical risk minimization} (ERM) principle for which generalization error bounds or  bounds for $\Delta_n$ can be used in a straightforward manner to bound the estimation error and, moreover, to establish some important consistency properties.  

More specifically, in the ERM procedure  the learner chooses the desired  functional $\widehat{H}_n $ out of the hypothesis class $\mathcal{H}$ of the admissible ones  using \eqref{eq:empiricalRiskDef} (the empirical version of  \eqref{eq:riskMin}), that is, 
\begin{equation}
\label{eq:empiricalRiskMin} 
\widehat{H}_n = \arg \min_{H \in \mathcal{H}} \widehat{R}_n(H), 
\end{equation}
which is well defined provided that such a minimizer exists and is unique; otherwise one may define $\widehat{H}_n$ to be an $\epsilon$-minimizer of the empirical risk (see \cite{Alon1997} for details).
We say that the  ERM  is {\bfi  strongly consistent} within the hypothesis class $\mathcal{H}$ if the { generalization error} $R(\widehat{H}_n)$ (or statistical risk) and the {training error} $\widehat{R}_n(\widehat{H}_n)$ (or empirical risk) as defined in \eqref{eq:riskDef} and in \eqref{eq:empiricalRiskDef}, respectively, for a sequence of functionals  $(\widehat{H}_n)_{n \in \mathbb{N}}$ picked by ERM from $\mathcal{H}$ using random samples of increasing length, both converge almost surely to the  Bayes in-class risk $R_{\mathcal{H}} ^\ast $ in \eqref{eq:riskMin}, that is  
\begin{equation}
\label{eq:genError}
\lim_{n \to \infty}  R(\widehat{H}_n) =  R _{\mathcal{H}}^\ast  \enspace {\rm a.s.}
\end{equation}
and
\begin{equation}
\label{eq:trainError}
\lim_{n \to \infty}  \widehat{R}_n(\widehat{H}_n) =  R_{\mathcal{H}} ^\ast   \enspace {\rm a.s.}
\end{equation}
When no assumptions on the distribution of $({\bf Z},{\bf Y})$ are used to prove \eqref{eq:genError} and \eqref{eq:trainError}, then this means that \eqref{eq:genError} and \eqref{eq:trainError} hold for all distributions and one talks about {\bfi  universal strong risk-consistency} of the ERM principle over the class $\mathcal{H}$. This is essentially the case in our setting, since we are working in a semi-agnostic setup and only invoke assumptions on the temporal dependence (but not on the marginal distributions of the input  and target stochastic processes $({\bf Z},{\bf Y})$). 

A standard approach to proving the strong  risk-consistency of the ERM procedure for the hypothesis class of functionals $\mathcal{H}$ consists in finding a sequence $(\eta_n)_{n \in \mathbb{N}}$  converging to zero for which the  inequality  
\begin{equation*}
\overline{\Delta}_n := \sup_{H \in \mathcal{H}} | \widehat{R}_n(H) - R(H) | \le \eta_n, 
\end{equation*}
holds $\mathbb{P}$-a.s.
To see that this implies \eqref{eq:genError} and \eqref{eq:trainError} one notes the following inequalities:
\begin{align}
R(\widehat{H}_n) - R _{\mathcal{H}}^\ast 
&=  \left(R(\widehat{H}_n) - \widehat{R} _n(\widehat{H}_n) \right)  + \left(\widehat{R} _n(\widehat{H}_n)  - \widehat{R} _n({H}^\ast ) \right) + \left(\widehat{R} _n({H}^\ast) - R _{\mathcal{H}}^\ast\right)\notag\\
&\le2 \eta_n + \left(\widehat{R} _n(\widehat{H}_n)  - \widehat{R} _n({H}^\ast ) \right) \le 2 \eta_n,
\label{eq:UnifCovergenceArgChain}
\end{align}
where the last inequality follows from the fact that, by definition  \eqref{eq:empiricalRiskMin}, $\widehat{H}_n$ is a minimizer of the empirical risk $\widehat{R} _n $. 

In the context of reservoir systems, we shall be working with a {\bfi  weak version of consistency} which imposes all the convergence conditions to hold only in probability.  In what follows we devise bounds for $\overline{\Delta}_n$ that allow us to establish the risk-consistency of the ERM procedure for reservoir systems. Additionally, we formulate high-probability bounds for $\overline{\Delta}_n$ which   provide us with convergence rates for the ERM-based estimation of RC systems that, to our knowledge, are not yet available  in the literature. It is well known that in some  cases  (for small classes with zero Bayes risk, see for example \cite{Bartlett2006}) the argument that we just discussed results in unreasonably slow rates. We defer the discussion of  possible refinements of the rates obtained in this paper to future projects.

\subsection{Learning procedure for reservoir systems} 
\label{Learning procedure for reservoir systems} 
The following paragraphs describe the implementation of the empirical risk minimization procedure in the setting of reservoir computing. We spell out the assumptions needed to derive the results in the next sections, construct the  hypothesis classes, and set up the ERM learning strategy for the different families  of reservoir systems discussed in Section~\ref{Families of reservoir systems}. 
\paragraph{Input and target stochastic processes.}For both the input  ${\bf Z}$ and the target   ${\bf Y}$ processes we assume a {\bfi  causal Bernoulli shift} structure (see for instance \cite{Dedecker2007a}, \cite{alquier:wintenberger}). More precisely, for $I=y,z$ and $q_I \in \mathbb{N}^+$  suppose that the so-called causal functional $G^I \colon (\mathbb{R}^{q_I})^{\mathbb{Z}_-} \to D_{o_I}$ (with $o_z = d$ and $D_{o_y} = \mathbb{R}^m$) is measurable and that $\bm{\xi}=((\bm{\xi}_{t}^y,\bm{\xi}_{t}^z))_{t \in \mathbb{Z}_-}$ are independent and identically distributed \ $\mathbb{R}^{q_y} \times \mathbb{R}^{q_z}$-valued random variables. We assume then that the input  ${\bf Z}$ and target processes  ${\bf Y}$ are Bernoulli shifts, that is, they are the (strictly) stationary processes determined by 
\begin{equation} 
\label{eq:ZFMPDef} 
\begin{aligned}
{\bf Z}_t & = G^z(\ldots,\bm{\xi}_{t-1}^z,\bm{\xi}_{t}^z), \quad t \in \mathbb{Z}_-, \\
{\bf Y}_t & = G^y(\ldots,\bm{\xi}_{t-1}^y,\bm{\xi}_{t}^y), \quad t \in \mathbb{Z}_-,
\end{aligned}
\end{equation} 
with  $\mathbb{E}[\|{\bf Z}_0\|_2] < \infty$, $\mathbb{E}[\|{\bf Y}_0\|_2] < \infty$.

Many processes derived from  stationary innovation sequences have causal Bernoulli shift structure including some that are of non-mixing type (see for instance the introduction in \citep{Dedecker2007a}). In order to obtain risk bounds for reservoir functionals as learning models, we  need to additionally impose assumptions on the weak dependency of the processes \eqref{eq:ZFMPDef}. More specifically, each of the three main results provided in the next section is formulated under a different weak dependence assumption which we now spell out in detail.  

We start with the strongest assumption of the three but which will allow us to obtain the strongest conclusions in terms of risk bounds for reservoir systems.
\begin{assumption} 
\label{ass:GLipschitz}
 For $I=y,z$ the functional $G^I$ is $L_I$-Lipschitz continuous when restricted to $(\ell^{1,{w^I}}_-(\mathbb{R}^{q_I}),\left\|\cdot \right\|_{1,w^I})$ for some strictly decreasing weighting sequence $w^I: \mathbb{N} \longrightarrow (0, 1]$ with finite mean, that is, $\sum_{j \in \mathbb{N}} j w_{j}^I < \infty$. More specifically, there exists $L_I > 0$ such that for all ${\bm{u}^I}= (\bm{u}^I_t)_{t \in \mathbb{Z}_-} \in \ell^{1,{w^I}}_-(\mathbb{R}^{q_I})$ and ${\bm{v}}^{I}= ({{\bm{v}}^{I}_t})_{t \in \mathbb{Z}_-} \in \ell^{1,{w^I}}_-(\mathbb{R}^{q_I})$ it holds that
\begin{equation} 
\label{eq:ZFMPProperty}
\|G^I(\bm{u}^I)-G^I({\bm{v}}^{I}) \|_2 \leq  L_I \left\|{\bm{u}}^{I}-{\bm{v}}^{I} \right\|_{1,w^I}.
\end{equation} 
Additionally, let the innovations in \eqref{eq:ZFMPDef} satisfy $\mathbb{E}[\|{\bm{\xi}}^{I}_0\|_2] < \infty$ for $I=y,z$.
\end{assumption}

The following example shows that one can easily construct causal Bernoulli shifts using reservoir functionals.
\begin{example}[Causal Bernoulli shifts out of reservoir functionals]
\label{ex:linearReservoir} 
\normalfont Consider a reservoir system of the type \eqref{eq:RCSystemDet} (see also the examples in Section~\ref{Families of reservoir systems}) determined by the Lipschitz-continuous reservoir map $F \colon D_N\times D_d\longrightarrow  D_N$ with $D_d \subset \mathbb{R}^d$, $D_N \subset \mathbb{R}^N$. Assume, additionally, that $F$ is a $r$-contraction on the first entry and denote by $L>0$ the Lipschitz constant of $F$ with respect to the second entry. Let $w: \mathbb{N} \longrightarrow (0, 1]$ be a strictly decreasing weighting sequence  with finite mean, that is $\sum_{j \in \mathbb{N}} j w_{j} < \infty$, and a finite associated inverse decay ratio $L_{w}$ (see Section~\ref{Notation}). Let now $V_d \subset (D_d)^{ \mathbb{Z}_-} \cap \ell^{w,1}_-( \mathbb{R}^d)$ be a time-invariant set and consider inputs  ${\bf z} \in V_d$. 
Suppose that the reservoir system \eqref{eq:RCSystemDet}    has a solution $(\mathbf{x}^0, {\bf z}^0) \in  (D_N)^{\mathbb{Z}_{-}} \times V _d$, that is, $\mathbf{x}_t^0=F(\mathbf{x}_{t-1}^0, {\bf z}_t ^0)$, for all $t \in \mathbb{Z}_{-} $. Then, by Theorem~4.1 and Remark~4.4 in \cite{RC9}, if 
\begin{equation}
\label{fmp condition on w}
r L _{w}<1,
\end{equation}
then the reservoir system associated to $F$ with inputs in $V _d $ has the echo state property and hence determines a unique continuous, causal, and time-invariant reservoir filter 
$
U^F:(V _d, \left\|\cdot \right\|_{1,w}) \longrightarrow ((D_N)^{\mathbb{Z}_{-}}, \left\|\cdot \right\|_{1,w})
$
which is Lipschitz-continuous with constant
\begin{equation}
\label{lips in case continuous Lz}
L_{U ^F}:= \frac{L}{1-r L _{w}}.
\end{equation}
It hence also has the fading memory property with respect to $w$.  
The Lipschitz continuity of the filter $U^F$ implies that the associated functional $H_{U^F}$ is also Lipschitz-continuous with the same Lipschitz constant (see Proposition 3.7 in  \cite{RC9}). Taking $G^I:=H_{U^F}$, it is easy to see that \eqref{eq:ZFMPProperty} indeed holds with $L_I = L_{U ^F}$.
\end{example}

The next assumption is weaker and it is satisfied by many discrete-time stochastic processes. The results that we obtain in the following sections invoking this type of weak dependence will  be also less strong than under Assumption~\ref{ass:GLipschitz}.
\begin{assumption} 
\label{ass:ThetaDecay}
For $I=y,z$ denote by  $(\widetilde{\bm{\xi}}_{t}^I)_{t \in \mathbb{Z}_-}$ an independent copy  of $(\bm{\xi}_{t}^I)_{t \in \mathbb{Z}_-}$ and define
\begin{equation} 
\label{eq:ThetaDef} 
\theta^I(\tau):= \mathbb{E}[\|G^I(\ldots,\bm{\xi}_{-1}^I,\bm{\xi}_{0}^I)-G^I(\ldots,\widetilde{\bm{\xi}}_{-\tau-1}^I,\widetilde{\bm{\xi}}_{-\tau}^I,\bm{\xi}_{-\tau+1}^I,\ldots,\bm{\xi}_{0}^I) \|_2], \quad \tau\in \mathbb{N}^+.
\end{equation}
Assume that for $I=y,z$ there exist $\lambda_I \in (0,1)$ and $C_I >0$ such that   it holds that
\begin{equation}
 \label{eq:tauExpDecay} 
 \theta^I(\tau) \leq C_I \lambda_I^\tau, \quad {\text for \ all}  \enspace \tau \in \mathbb{N}^+.
 \end{equation}
\end{assumption}
\begin{remark} 
\label{one implies 2}
\normalfont Note that whenever the weighting sequence $w^I$ in Assumption~\ref{ass:GLipschitz} can be chosen to be a geometric one, that is, $w^{I}_j = \lambda_{I}^j$, $j \in \mathbb{N}$ with $\lambda_I\in (0,1)$, then Assumption~\ref{ass:ThetaDecay} is also automatically satisfied. The argument proving this appears for instance in the proof of part {\bf (i)} of Corollary~\ref{cor:ThetaWeighting}.
\end{remark}
The following example illustrates that for many widely used time series models this assumption does hold. In particular, we show that vector autoregressive VARMA processes with time-varying coefficients under mild conditions and, in particular, GARCH processes satisfy Assumption~\ref{ass:ThetaDecay}.
\begin{example}[VARMA process with time-varying coefficients]
\label{ex:GARCH} 
\normalfont Suppose ${\bf Z}= ({\bf Z}_t)_{t \in \mathbb{Z}_-}$ is a vector autoregressive process of first order with time-varying coefficients, which we write as
\begin{equation} \label{eq:auxEq27} {\bf Z}_t = { A}_t {\bf Z}_{t-1}+ \boldsymbol{\eta}_t, \quad t \in \mathbb{Z}_-,  \end{equation} 
where $(\boldsymbol{\eta}_t)_{t \in \mathbb{Z}_-} \sim {\rm IID}$ with  $\boldsymbol{\eta}_t \in \mathbb{R}^d$ and $\mathbb{E}[\|\boldsymbol{\eta}_0\|_2]<\infty$, and where $({ A}_t)_{t \in \mathbb{Z}_-} \sim {\rm IID}$ with  ${ A}_t \in \mathbb{M}_{d}$ and $\mathbb{E}[\vertiii{ A_0}_2]<1$. Under these hypotheses (see for instance \citet*{Brandt1986} and \citet*[Theorem~1.1]{Bougerol1992})  there exists a unique stationary process satisfying \eqref{eq:ZFMPDef} and \eqref{eq:auxEq27}  and $\mathbb{E}[\|{\bf Z}_0\|_2]<\infty$. Iterating \eqref{eq:auxEq27} yields 
\[ {\bf Z}_0 = \boldsymbol{\eta}_0 + { A}_0 \boldsymbol{\eta}_{-1} + \cdots + { A}_0 \cdots  { A}_{-\tau+1} {\bf Z}_{-\tau}, \] 
and so by definition, using the independence of $({ A}_t)_{t \in \mathbb{Z}_-} $ and stationarity, one gets
 \[\begin{aligned} 
 \theta^z(\tau) & \leq 2 \mathbb{E}[\vertiii{{ A}_0 \cdots  { A}_{-\tau+1}}_2] \mathbb{E}[\|{\bf Z}_{-\tau}\|_2] \le 2 \mathbb{E}[\vertiii{{ A}_0}_2]^\tau \mathbb{E}[\|{\bf Z}_{0}\|_2]. \end{aligned} \]
 We now define  $C_z := 2 \mathbb{E}[\|{\bf Z}_{0}\|_2]$, $\lambda_z := \mathbb{E}[\vertiii{ A_0}_2] $ and immediately obtain that \eqref{eq:tauExpDecay} indeed holds, that is, for all $\tau \in \mathbb{N}^+$
  \[\begin{aligned} 
 \theta^z(\tau)  \leq C_z \lambda_z^\tau, \end{aligned} \]
 as required.
\end{example}
 We now consider a concrete example of an autoregressive process of the type \eqref{eq:auxEq27} which is extensively used to describe and eventually to forecast the volatility of financial time series, namely the generalized autoregressive conditional heterostedastic (GARCH) family (\cite{engle:arch, bollerslev:garch, Francq2010}).

\begin{example}[GARCH process] 
\normalfont
Consider a GARCH(1,1) model given by the following equations:
\begin{align} 
r_t & = \sigma_t \varepsilon_t, \quad \varepsilon_t \sim {\rm IID}(0, 1), \quad t \in \mathbb{Z}_-
\label{retgarch}\\ 
\sigma_t^2 & = \omega + \alpha r_{t-1}^2 + \beta \sigma_{t-1}^2, \quad t \in \mathbb{Z}_- \label{volgarch} \end{align}
with parameters that satisfy $\alpha, \beta, \omega \geq 0$, $\alpha+\beta <1$, which guarantees the second order stationarity of the process $(r_t)_{t\in \mathbb{Z}_-}$ and the positivity of the conditional variances $(\sigma _t ^2)_{t \in \mathbb{Z}_{-}}$. 
We now check if the GARCH(1,1) process in \eqref{retgarch}-\eqref{volgarch} falls in the framework  \eqref{eq:auxEq27} introduced in the previous example. Let $d=2$ and define \[  {\bf Z}_t := \begin{pmatrix}
  r_t^2 \\
   \sigma_t^2    
 \end{pmatrix}, \quad  \boldsymbol{\eta}_t := \begin{pmatrix}
   \omega \varepsilon_t^2 \\
    \omega    
  \end{pmatrix}, \quad { A}_t := \begin{pmatrix}
    \alpha \varepsilon_t^2 & \beta \varepsilon_t^2 \\
     \alpha & \beta   
   \end{pmatrix},  \quad t \in \mathbb{Z}_-.
  \]
It is easy to verify that with this choice of matrix $A_t$, $t\in \mathbb{Z}_-$, one has $\mathbb{E}[\vertiii{ A_0}_2] = \mathbb{E}[\alpha \varepsilon_0^2 + \beta] = \alpha + \beta < 1$, by the stationarity condition.
Additionally,  
$\mathbb{E}[\|\boldsymbol{\eta}_0\|_2] =\omega \mathbb{E}[ \sqrt{ \varepsilon_0^4 + 1}] \le 
\omega \mathbb{E}[ \varepsilon_0^2 + 1]=2 \omega < \infty$. 
Hence, the GARCH(1,1) model in \eqref{retgarch}-\eqref{volgarch} can be represented as \eqref{eq:auxEq27} and   automatically satisfies Assumption~\ref{ass:ThetaDecay}.
\end{example}

\begin{assumption} \label{ass:ThetaAlgDecay}
Assume that  for $I=y,z$ there exist $\alpha_I \in (0,\infty)$ and $C_I >0$ such that, 
\begin{equation} 
\label{eq:tauAlgDecay} 
\theta^I(\tau) \leq C_I \tau^{-\alpha_I} , \quad \text{ for all }  \tau \in \mathbb{N}^+, 
\end{equation}
with $\theta^I$ as in \eqref{eq:ThetaDef}.
\end{assumption}

\begin{example}[ARFIMA process]
\label{ARFIMA process}
\normalfont
Let $\overline{d} \in (-\frac{1}{2},\frac{1}{2})$ and suppose that ${\bf Z}= ({Z}_t)_{t \in \mathbb{Z}_-}$ is an  autoregressive fractionally integrated moving average  ARFIMA $(0,\overline{d},0)$ process (see, for instance, \cite{Hosking1981} and \cite{Beran1994} for details). The process ${\bf Z}$ admits an infinite  moving average (MA($\infty$)) representation 
\begin{equation*}
{ Z}_t = \sum_{k=0}^\infty \phi_k \varepsilon_{t-k}, \quad t \in \mathbb{Z}_-,
\end{equation*}
with  innovations $(\varepsilon_t )_{t \in \mathbb{Z}_-} \sim {\rm IID}(0,1)$  and where the coefficients are given by
$\phi_k = \frac{\Gamma(k+\overline{d})}{\Gamma(k+1)\Gamma(\overline{d})}$ so that $\Gamma(\overline{d}) k^{1-\overline{d}} \phi_k \to 1$, as $k \to \infty$. Using this asymptotic behaviour and the independence of the innovations one obtains 
\begin{equation*}
\theta^z(\tau) \leq 2 \mathbb{E}\left[\left| \sum_{k=\tau}^\infty \phi_k \varepsilon_{-k} \right|\right] \leq 2 \left( \sum_{k=\tau}^\infty \phi_k^2 \right)^{1/2} \leq 2\sup_{l \in \mathbb{N}^+}\{l^{1-\overline{d}} \phi_l\} \left( \sum_{k=\tau}^\infty k^{2\overline{d}-2} \right)^{1/2}.
\end{equation*}
Comparing the sum to the integral $\int_{\tau}^\infty x^{2\overline{d}-2} \mathrm{d} x = \frac{1}{1-2 \overline{d}} \tau^{2\overline{d}-1}$, it is easy to see that  \eqref{eq:tauAlgDecay} is satisfied with $\alpha_z = \frac{1}{2}-\overline{d}>0$. 
\end{example}

\paragraph {Hypothesis classes of reservoir maps $\mathcal{F}^{RC}$ and reservoir functionals $\mathcal{H}^{RC}$.} The next step  in order to set up the learning  program in the context of reservoir systems is  to construct the associated hypothesis classes. These classes need to be chosen beforehand and consist of candidate functionals associated to causal time-invariant reservoir filters of the type discussed in Sections \ref{Filters and reservoir computing systems} and~\ref{Families of reservoir systems}. 

For fixed  $N, d \in \mathbb{N}^+$, consider  a class $\mathcal{F}^{RC}$ of reservoir maps $F \colon D_N\times D_d\longrightarrow  D_N$,  ${\bf 0} \in D_N \subset \mathbb{R}^N$, $D_d \subset \mathbb{R}^d$ that we assume is (a subset of and) separable in the space of bounded continuous functions when equipped with the supremum norm and, additionally, satisfies the following assumptions: 
\begin{assumption} 
\label{ass:FLipschitz}
There exist $r \in (0,1)$ and $L_R >0$ such that for each $F \in \mathcal{F}^{RC}$:
\begin{description}
	\item[(i)] for any ${\bf z} \in D_d$, $F(\cdot,{\bf z})$ is an $r$-contraction,
	\item[(ii)] for any ${\bf x} \in D_N$, $F({\bf x},\cdot)$ is $L_R$-Lipschitz.
\end{description}
\end{assumption}
\begin{assumption}
\label{ass:FESP}
	For any $F \in \mathcal{F}^{RC}$
	the (first equation in the) system \eqref{eq:RCSystemDet} has the echo state property. If $H^F$ is the functional associated to it, we assume that $H^F$ is measurable with respect to the Borel $\sigma $-algebra associated to the product topology on its domain. 
\end{assumption}
Notice that if the state space $D_N$ is a closed ball,  then   Assumption~\ref{ass:FLipschitz} implies Assumption~\ref{ass:FESP} by Proposition~\ref{prop:contractionEchoState}. This implication holds for any reservoir system with bounded reservoir maps, an example of which are the elements of the echo state networks family with bounded activation functions $\sigma$ in Section~\ref{Families of reservoir systems}.

\begin{assumption} 
\label{ass:XBounded}
There exists $M_\mathcal{F}> 0$ such that 
\begin{equation*}
\|H^F({\bf z})\|_2 \leq M_\mathcal{F}, \quad \mbox{for all ${\bf z} \in (D_d)^{\mathbb{Z}_-}$ and for each $F \in \mathcal{F}^{RC}$.}
\end{equation*}
\end{assumption}

This assumption automatically holds for many families of reservoir systems. We carefully addressed this question in Section~\ref{Families of reservoir systems}, where  we discussed various families and input types  for which the reservoir functionals are indeed bounded.  For example, Assumption~\ref{ass:XBounded} is satisfied by construction in the case of bounded inputs for all the families in Section~\ref{Families of reservoir systems}. In the presence of generic unbounded inputs, Assumption~\ref{ass:XBounded} obviously holds for echo state networks (ESN) with bounded activation function. 
In addition, the condition in Assumption~\ref{ass:XBounded} appears in many applications. For instance, in the recent paper by \cite{Verzelli2019} it is shown that using a so-called self-normalizing activation function allows one to achieve high performances in standard benchmark tasks. It is not difficult to see that  self-normalizing functions yield $\|H^F({\bf z})\|_2 \leq 1$.

Our assumptions also guarantee that various suprema over the classes $\mathcal{H}^{RC}$ and $\mathcal{F}^{RC}$ that will appear in the sequel are measurable random variables.  There are very general conditions that guarantee such a fact holds (see \cite[Corollary~5.25]{Dudley2014}) but here we simply assume that $H_F$ for all $F \in \mathcal{F}^{RC}$ is bounded (see Assumption~\ref{ass:XBounded}) and that $\mathcal{F}^{RC}$ is separable in the space of bounded continuous functions when equipped with the supremum norm. This condition together with the continuity assumptions imposed below on the loss function allows us to conclude the measurability of the suprema over $\mathcal{H}^{RC}$ and $\mathcal{F}^{RC}$ (see Lemma~\ref{lem:supmeasurable}  in Appendix~\ref{Preliminary results} for the details).

Once we have spelled out Assumptions~\ref{ass:FLipschitz}-\ref{ass:XBounded} that define the class $\mathcal{F}^{RC}$, we proceed to construct the corresponding hypothesis class of reservoir functionals $\mathcal{H}^{RC}$. Since in most of the cases considered in the literature  the readouts  $h$ in \eqref{eq:RCSystemDet} are either polynomial (as in the case of reservoir systems with linear reservoir maps and polynomial readouts) or linear  (as in the case of reservoir systems with linear reservoir maps and linear readouts, ESNs, and SAS in Section~\ref{Families of reservoir systems}), we shall treat the case of generic Lipschitz readouts and the linear case separately:


\begin{description}
\item [(i)] {\bf Reservoir functionals hypothesis class} $\mathcal{H}^{RC} $ {\bf with Lipschitz readouts.} We consider a set $\mathcal{F}^{O}$ of readout maps $h \colon D_N \to \mathbb{R}^m$ that are Lipschitz-continuous  with Lipschitz constant  $L_h>0$. We assume that for all the members of the class it holds that $L_h\leq \overline{L_h}$ and $\|h( {\bf 0})\|_2\leq L_{h,0}$, for some fixed $\overline{L_h}, L_{h,0}>0$, and that the class contains the zero function and is separable in the space of bounded continuous functions when equipped with the supremum norm.
In this situation, we define the hypothesis class $\mathcal{H}^{RC}$ of reservoir functionals  as
\begin{align}
\label{eq:Hfixedgeneral} 
\mathcal{H}^{RC} := \{ H \colon (D_d)^{\mathbb{Z}_-} \to {\mathbb{R}}^m \, \mid \, &H({\bf z})=h(H^F({\bf z})), h \in \mathcal{F}^{O},  F \in \mathcal{F}^{RC} \}.
\end{align}

\item [(ii)] {\bf Reservoir functionals hypothesis class} $\mathcal{H}^{RC} $ {\bf with linear readouts.} Most of the examples of reservoir systems which we discussed in Section~\ref{Families of reservoir systems} are constructed  using linear readout maps, which are known to be easier to train  and  popular in many practical applications. We hence treat this case separately.
Let now the readouts $h$ be given by maps of the type  $h( \mathbf{x}) = W \mathbf{x} + \boldsymbol{a}$, $\mathbf{x}\in D_N$, with  $W \in \mathbb{M}_{m,N}$ and $\boldsymbol{a}\in \mathbb{R}^m$. We assume that for all the members of the class it holds that $\vertiii{W}_2\leq \overline{L_h}$ and $\|h( {\bf 0})\|_2= \|\boldsymbol{a}\|_2 \leq L_{h,0}$,
for some fixed $\overline{L_h}, L_{h,0}>0$. In this case, such a class of readouts is automatically separable in the space of bounded continuous functions when equipped with the supremum norm. In this situation we hence define the hypothesis class $\mathcal{H}^{RC}$ of reservoir functionals  as
\begin{align}
\mathcal{H}^{RC} := \{ H \colon (D_d)^{\mathbb{Z}_-} \to {\mathbb{R}}^m \, \mid \, &H({\bf z})=W H^F({\bf z}) + \boldsymbol{a}, W \in \mathbb{M}_{m,N}, \boldsymbol{a}\in \mathbb{R}^m,\nonumber\\
\label{eq:Hfixed} 
&\vertiii{W}_2\leq \overline{L_h}, 
\|\boldsymbol{a}\|_2 \leq L_{h,0}, F \in \mathcal{F}^{RC} \}.
\end{align}
\end{description}

\paragraph {Loss function.} The choice of   loss function is often  key to the success in quantifying risk bounds for learning models. In this paper we work with distance-based loss functions of the form
\begin{equation}\label{defLoss}
L(\mathbf{x},{\bf y}) = \sum_{i=1}^m f_i(x_i-y_i),
\end{equation}
for $\mathbf{x},{\bf y} \in \mathbb{R}^m$, where for each $i\in \left\{1,\ldots,m \right\}$, the so-called representing functions $f_i \colon \mathbb{R} \longrightarrow \mathbb{R}$ are all Lipschitz-continuous with the same Lipschitz constant $L_L / \sqrt{m}$ and satisfy $f_i(0)=0$.  The assumption of Lipschitz-continuity on the loss $L$ in the case in which its codomain is restricted to $\mathbb{R}^+$ guarantees that it is also a Nemitski loss of order $p=1$ (see \cite{Christmann2008} for detailed discussion of Nemitski losses and their associated risks). 
Notice that our assumptions imply in particular that
\begin{equation} \label{eq:lossMoment}
\mathbb{E}[|L({\bf 0},{\bf Y}_{0})|] < \infty
\end{equation}
and
\begin{equation}\label{eq:lossLipschitz} |L(\mathbf{x},{\bf y}) - L(\overline{\mathbf{x}},\overline{{\bf y}}) | \leq L_L (\|\mathbf{x}-\overline{\mathbf{x}} \|_2 +  \|\mathbf{y}-\overline{\mathbf{y}} \|_2), \enspace \mathbf{x}, \overline{\mathbf{x}}, {\bf y}, \overline{\mathbf{y}}\in \mathbb{R}^m. \end{equation}


Additionally, we notice that since we restrict to reservoir systems satisfying the echo state property and the hypothesis class $\mathcal{H}^{RC}$ contains their associated reservoir functionals, for $H= h \circ H^F$ the idealized empirical risk \eqref{eq:empiricalRiskFullHistoryDef} can be written as
\begin{equation*}
\widehat{R}_n^\infty(H) = \frac{1}{n} \sum_{i=0}^{n-1} L(h(H^F({\bf Z}_{-i}^{-\infty})),{\bf Y}_{-i})=
\frac{1}{n} \sum_{i=0}^{n-1} L(U_h^F({\bf Z}_{0}^{-\infty})_{-i},{\bf Y}_{-i})
=\frac{1}{n} \sum_{i=0}^{n-1} L(h({\bf X}_{-i}),{\bf Y}_{-i}),
\end{equation*}
where $\mathbf{X}$ is the solution of the reservoir  system 
\begin{equation*}
\mathbf{X}_t =F(\mathbf{X}_{t-1}, {\bf Z}_t), \quad t \in \mathbb{Z}_-.
\end{equation*}

\paragraph{Risk consistency and risk bounds of reservoir systems.} As  discussed in Section~\ref{General setup of empirical risk minimization procedure} we are interested in generalization error bounds or, in particular, in deriving uniform bounds for $\Delta_n =  \sup_{H \in \mathcal{H}^{RC}} \lbrace R(H)  - \widehat{R}_n(H)  \rbrace$. In order to proceed, we first decompose $\Delta_n$ and write
\begin{align}
\label{unifc}
\Delta_n =  \sup_{H \in \mathcal{H}^{RC}} \lbrace R(H)  - \widehat{R}_n(H) \rbrace & \leq 
\sup_{H \in \mathcal{H}^{RC}} \lbrace R(H)  - \widehat{R}_n(H) - \widehat{R}_n^\infty(H) + \widehat{R}_n^\infty(H)\rbrace \nonumber\\
&\leq\sup_{H \in \mathcal{H}^{RC}} \left\{ \widehat{R}_n^\infty(H)-\widehat{R}_n(H) \right\} + \sup_{H \in \mathcal{H}^{RC}} \left\{ {R}(H)  - \widehat{R}_n^\infty(H) \right\}\nonumber\\
&\leq\sup_{H \in \mathcal{H}^{RC}} \left| \widehat{R}_n(H)  - \widehat{R}_n^\infty(H) \right|  + \sup_{H \in \mathcal{H}^{RC}} \left| {R}(H)  - \widehat{R}_n^\infty(H) \right|.
\end{align}
This means that one can find  upper bounds for both $\Delta_n$ and $\overline{\Delta}_n= \sup_{H \in \mathcal{H}^{RC}} | \widehat{R}_n(H) - R(H) |$ by 
controlling the two summands in the right hand side of the last inequality.  

Coming back to the example of the ERM procedure that we discussed in Section \ref{General setup of empirical risk minimization procedure}, the previous expression can also be used to deduce the weak (essentially universal) risk-consistency of the ERM for the class ${\mathcal H}^{RC}$. More specifically, in line with classical results due to \cite{Vapnik1991}, from the inequalities \eqref{eq:UnifCovergenceArgChain} it follows that in order  to establish the weak (essentially universal) risk-consistency of ERM for reservoir functionals  one simply needs to show that  for any $\epsilon, \delta>0$ there exists $n_0 \in \mathbb{N}^+$ such that for all $n\geq n_0$ it holds that
\begin{equation}
\label{eq:UnifCovergenceArg}
\mathbb{P}\left(\overline{\Delta}_n 
 > \epsilon \right) = \mathbb{P}\left(\sup_{H \in \mathcal{H}^{RC}} \left| R(H) - \widehat{R}_n(H)\right| 
 > \epsilon \right)\leq \delta.
\end{equation}
Whenever the inputs are  i.i.d. (which is a particular case of our setup), this definition is essentially equivalent to saying that  $\mathcal{H}^{RC}$ is a (weak) uniform Glivenko-Cantelli class (see for example \cite{Mukherjee2002}). From expression \eqref{unifc} it also follows then that in order to establish the (weak) risk-consistency, it suffices to show the two-sided uniform convergence over the class $\mathcal{H}^{RC}$ of reservoir functionals, first,  of the truncated versions of  empirical risk to their idealized counterparts and, second,  of these idealized versions of the empirical risk to the generalization error (or statistical risk). More explicitly, we shall separately show that for any $\epsilon_1, \delta_1>0$ and $\epsilon_2, \delta_2>0$  there exist $n_1 \in \mathbb{N}^+$ and  $n_2 \in \mathbb{N}^+$ such that for all $n\geq n_1$ and $n\geq n_2$, respectively, it holds that
\begin{equation}
\label{eq:DifIdealizedTruncated}
\mathbb{P}\left(\sup_{H \in \mathcal{H}^{RC}}\left| \widehat{R}_n(H) - \widehat{R}^\infty_n(H)  \right| 
 > \epsilon_1 \right)\leq \delta_1 
\end{equation}
and
\begin{equation}
\label{eq:DifRiskIdealized}
\mathbb{P}\left(\sup_{H \in \mathcal{H}^{RC}}\left| {R}(H) - \widehat{R}^\infty_n(H)  \right| 
 > \epsilon_2 \right)\leq \delta_2. 
\end{equation} 
One needs to start by showing that the suprema of both these differences over the class $\mathcal{H}^{RC}$  are indeed random variables. This fact has been proved in Lemma~\ref{lem:supmeasurable} in the Appendix~\ref{Preliminary results}. Next, we need to  show that the difference between the idealized  and the truncated empirical risks can be made as small as one wants by choosing an appropriate length $n \in \mathbb{N}^+$ of the training sample. This fact is contained in the following result.
\begin{proposition} \label{prop:finiteHistoryError}
Consider the hypothesis class $\mathcal{H}^{RC}$ of reservoir functionals defined in \eqref{eq:Hfixedgeneral}. Define   
\begin{equation}
\label{eq:C0def} 
C_0 := \frac{2 r L_L \overline{L_h} M_\mathcal{F} }{1-r}.
\end{equation}
Then, for any $n \in \mathbb{N}^+$
\[ \sup_{H \in \mathcal{H}^{RC}} \left| \widehat{R}_n(H)  - \widehat{R}_n^\infty(H) \right| \leq \frac{C_0(1-r^n)}{n} \]
holds $\mathbb{P}$-a.s.
\end{proposition}

This proposition implies that \eqref{eq:DifIdealizedTruncated} indeed holds. In order to complete the uniform convergence argument, we also need to show that \eqref{eq:DifRiskIdealized} holds. Even though in order to prove the (essentially universal)  risk-consistency of the ERM procedure for reservoir systems it is  sufficient to show that the upper bounds  of $\overline{\Delta}_n$ (and ${\Delta}_n$) can be made as small as one wants with $n \rightarrow \infty$, for hyper-parameter selection in practical applications the availability of non-asymptotic bounds is also of much importance. We see in the following section that depending on the particular weak dependence assumption imposed (Assumptions \ref{ass:GLipschitz}-\ref{ass:ThetaAlgDecay}) we will be able to use more or less strong concentration inequalities that yield finite-sample size bounds with different rates of convergence.

\section{Main Results}
\label{Main Results}
In this section we provide high-probability risk bounds for reservoir computing systems.  The main ingredients of the probability bounds are the expected values of ${\Gamma}_n:=\sup_{H \in \mathcal{H}^{RC}} \{{R}(H) - \widehat{R}^\infty_n(H)\}  $ and of $\overline{\Gamma}_n:=\sup_{H \in \mathcal{H}^{RC}}\left| {R}(H) - \widehat{R}^\infty_n(H)  \right|$, which are the maximum difference of the idealized training and the generalization errors over the class $\mathcal{H}^{RC}$ and the maximum of the absolute value of this difference, respectively. Since the random variables in the training sample of the input and the output discrete-time processes are not independent and identically distributed, bounding the expected values of $\Gamma_n$ and $\overline{\Gamma}_n$ is a challenging task. In the first subsection we show that this problem may be circumvented using the following idea: one may compute the empirical risk by partitioning the training sample into blocks of appropriate length and then exploiting the weak dependence of the input and output stochastic processes spelled out in Assumptions~\ref{ass:GLipschitz}-\ref{ass:ThetaAlgDecay}. We first make use of this ``block-partitioning'' idea in order to derive the bounds for the expected values of the random variables $\Gamma_n$ and $\overline{\Gamma}_n$ in the setting of each of those three assumptions. These bounds are expressed in terms of the so-called Rademacher complexities of the reservoir hypothesis classes and the weak dependence coefficients of the input and the target stochastic processes. We provide details concerning the complexity bounds for particular families in the second subsection. In the third subsection we then use the fact that the random fluctuations of $\Gamma_n$ and $\overline{\Gamma}_n$ around their expected values  can be controlled using concentration inequalities, which, as we show further, can be done either with the help of the Markov inequality under the weaker Assumptions \ref{ass:ThetaDecay}-\ref{ass:ThetaAlgDecay}, or using  stronger exponential concentration inequalities (Propositions~\ref{prop:concentration}, \ref{prop:concentrationUnbounded}) under the stronger Assumption~\ref{ass:GLipschitz}. This approach yields explicit expressions for non-asymptotic high-probability  bounds for $\overline{\Delta}_n$ and hence for $\Delta_n$, which we spell out in the third subsection. Finally, showing that these upper bounds can be made as small as one wants as $n \rightarrow \infty$ proves the desired (weak and essentially universal)  risk-consistency of the ERM-selected reservoir systems used as learning models. All the proofs of the main results given in this section are provided in the appendices.

\subsection{Bounding the expected value}

The main ingredient that needs to be introduced in order to bound the expected value of both $\Gamma_n$ and $\overline{\Gamma}_n$ is a complexity measure for the hypothesis classes of reservoir functionals $\mathcal{H}^{RC}$. Many complexity measures have been discussed in the literature in recent years (see for example \cite{Vapnik1968,ledoux:talagrand, Bartlett2003, Ben-David2014, AlexanderRakhlinKarthikSridharan2015}). In this paper we  use the so-called {\bfi  (multivariate) Rademacher-type complexity} associated to a given $\mathcal{H}^{RC}$, which we denote as  $\mathcal{R}_k(\mathcal{H}^{RC})$. More explicitly, let $k \in \mathbb{N}^+$ and consider  $\varepsilon_0,\ldots,\varepsilon_{k-1}$  independent and identically distributed \ Rademacher random variables and let   $\widetilde{{\bf Z}}^{(j)}$, $j=0,\ldots,k-1$, denote independent copies of ${\bf Z}$ ({\bfi ghost processes}), which are also independent of $\varepsilon_0,\ldots,\varepsilon_{k-1}$. The Rademacher-type complexity $\mathcal{R}_k(\mathcal{H}^{RC})$ over $k$ ghost processes is defined as
\begin{equation}
\label{eq:Rcomplexity} 
\mathcal{R}_k(\mathcal{H}^{RC}) = \frac{1}{k} \mathbb{E} \left[ \sup_{H\in \mathcal{H}^{RC}} \left\|  \sum_{j=0}^{k-1} \varepsilon_j H(\widetilde{{\bf Z}}^{(j)}) \right\|_2 \right].  
\end{equation}
Note that $\mathcal{R}_k(\mathcal{H}^{RC})$ is not an empirical Rademacher complexity and that the expectation is taken with respect to the law of Rademacher random variables and the distribution of the input process ${\bf Z}$. In this paper we do not use  the standard approach consisting in bounding the theoretical Rademacher complexity using its empirical analogue (conditional on $(\widetilde{{\bf Z}}^{(j)})_{j \in \{ 0,\ldots, k-1\}}$), since in the context of reservoir systems the ghost processes  $\widetilde{{\bf Z}}^{(j)}$ have no empirical interpretation due to the fact that it is usually only a single trajectory and not  i.i.d.\ samples of input data which are available to the learner.

 The following results provide upper bounds for the expected and the expected absolute value of the largest deviation of the statistical risk from its idealized empirical analogue within the hypothesis class of  reservoir functionals $\mathcal{H}^{RC}$. The two upper bounds \eqref{eq:Rademacher} and  \eqref{eq:RademacherAbs} in the next proposition share the same first three terms, up to a factor $2$ due to the absolute value. The first term is related to the weak dependence coefficients of the input and target signals, $\theta^{z}$ and $\theta^{y}$, respectively. The second term involves the Rademacher-type complexity \eqref{eq:Rcomplexity} of the hypothesis class of reservoir functionals $\mathcal{H}^{RC}$. Finally, the third term is always of order $\frac{\tau}{n}$, where $\tau$ is the block length, which needs to be carefully chosen depending on the rates of decay of $\theta^{z}$ and $\theta^{y}$, as we show later in    Corollary~\ref{cor:ThetaWeighting}. The upper  bound for the expected absolute value \eqref{eq:RademacherAbs} contains an additional term of order $\frac{\sqrt{\tau}}{\sqrt{n}}$.

\begin{proposition}
\label{prop:expBoundWithRademacher} 
Let $\mathcal{H}^{RC}$ be  the hypothesis class  of reservoir functionals associated to the reservoir maps in the class $\mathcal{F}^{RC}$ as given in \eqref{eq:Hfixedgeneral} or \eqref{eq:Hfixed}. Let both the input process ${\bf Z}$ and the target process  ${\bf Y}$  have a \textit{causal Bernoulli shift} structure as in \eqref{eq:ZFMPDef} and take values in $(D_d)^{ \mathbb{Z}_-}$ and $ (\mathbb{R}^m)^{ \mathbb{Z}_-}$, respectively. Then, there exist $B > 0$, $M>0$, and $\{a_\tau\}_{\tau \in \mathbb{N}^+}$ with $a_\tau \in  (0,\infty)$, such that for any $\tau, n \in \mathbb{N}^+$ with $\tau < n$ it holds that 
\begin{equation} 
\label{eq:Rademacher}
 \mathbb{E}\left[  \Gamma_n \right] =\mathbb{E}\left[ \sup_{H \in \mathcal{H}^{RC}} \left\{R(H) - \widehat{R}_n^\infty(H)  \right\}\right]
 \leq \frac{k \tau}{n} a_\tau +\frac{B k \tau}{n} \mathcal{R}_k(\mathcal{H}^{RC}) +  \frac{2M(n-k\tau)}{n} ,  \end{equation} 
where $k = \lfloor n / \tau \rfloor$ and 
\begin{align} 
 \mathbb{E}\left[  \overline{\Gamma}_n \right] =\mathbb{E}\left[  \sup_{H \in \mathcal{H}^{RC}} \left| R(H)-\widehat{R}_n^\infty(H) \right| \right]
 & \leq \frac{k \tau}{n} a_\tau +\frac{2 B k \tau}{n} \mathcal{R}_k(\mathcal{H}^{RC}) + \frac{2M(n-k\tau)}{n}\nonumber \\
 &+\frac{4 \tau \sqrt{k}  }{n} L_L \mathbb{E}\left[ \| {\bf Y}_{0} \|_2^2 \right]^{1/2}.  \label{eq:RademacherAbs} 
\end{align}
\noindent In these expressions, the Rademacher complexity $ \mathcal{R}_k(\mathcal{H}^{RC})$  of the hypothesis class $\mathcal{H}^{RC}$ of reservoir functionals is defined as in \eqref{eq:Rcomplexity}, the constants and the sequence $\{a_\tau\}_{\tau \in \mathbb{N}^+}$  can be explicitly expressed as
\begin{align}
\label{eq:Cdef} 
B &= 2  \sqrt{m}L_L,\\
\label{eq:Mdef} 
M & = L_L \overline{L_h } M_\mathcal{F} + \mathbb{E}[|L({\bf 0},{\bf Y}_{0})|] +L_{h,0} L_L,  \\ 
\label{eq:atau} 
a_\tau & = L_L (2 r^\tau M_\mathcal{F} \overline{L_h } + \theta^y(\tau) + L_R \overline{L_h } \sum_{l=0}^{\tau-1} r^l \theta^z(\tau-l)  ),
\end{align}
where for $I=y,z$  the weak dependence coefficients $\theta^I$ for $\tau \in \mathbb{N}^+$ are defined as in \eqref{eq:ThetaDef}, namely, 
\begin{equation}
\label{eq:ThetaDefRepeat}
\theta^I(\tau)= \mathbb{E}[\|G^I(\ldots,\bm{\xi}_{-1}^I,\bm{\xi}_{0}^I)-G^I(\ldots,\overline{\bm{\xi}}_{-\tau-1}^I,\overline{\bm{\xi}}_{-\tau}^I,\bm{\xi}_{-\tau+1}^I,\ldots,\bm{\xi}_{0}^I) \|_2],
\end{equation}
with $(\overline{\bm{\xi}}_{t}^I)_{t \in \mathbb{Z}_-}$ an independent copy of $(\bm{\xi}_{t}^I)_{t \in \mathbb{Z}_-}$.
\end{proposition}
We now explore  the conditions required for the upper bounds in \eqref{eq:Rademacher} and \eqref{eq:RademacherAbs} to be finite and exhibit a certain decay as a function of $\tau$ and $n$. Notice that (up to a factor $2$) the right-hand sides of the two inequalities are equal up to the last summand in  {\eqref{eq:RademacherAbs}} and hence one needs to  impose that $\mathbb{E}\left[ \| {\bf Y}_{0} \|_2^2 \right] <\infty$. Additionally, in order to better understand the behavior of both bounds as a function of $\tau$ and $n$ one needs to study two more ingredients, namely the sequence $\{a_\tau\}_{\tau \in \mathbb{N}^+}$ and the Rademacher complexity $ \mathcal{R}_k(\mathcal{H}^{RC})$. The behavior of the sequence $\{a_\tau\}_{\tau \in \mathbb{N}^+}$  is exclusively determined by the properties of the input and target processes,  while the Rademacher complexity of $\mathcal{H}^{RC}$ is fully characterized by the type of reservoir and readout maps of the given family of reservoir systems. In the following remark we argue that  under either the stronger Assumption~\ref{ass:GLipschitz} or the weaker Assumptions~\ref{ass:ThetaDecay}-\ref{ass:ThetaAlgDecay} the sequence $\{a_\tau\}_{\tau \in \mathbb{N}^+}$ converges to zero.

\begin{remark} \normalfont
A condition guaranteeing that the sequence $\{a_\tau\}_{\tau \in \mathbb{N}^+}$ in \eqref{eq:atau} converges to zero is, for instance, that
\begin{equation}\label{eq:tausummable}\sum_{\tau = 1}^\infty \theta^I(\tau) < \infty, \quad I=y,z. \end{equation}
To verify this, observe that this condition also implies that 
\[ \sum_{\tau=1}^\infty \sum_{l=0}^{\tau-1} r^l \theta^z(\tau-l) = \sum_{l=0}^\infty r^l \sum_{\tau=l+1}^{\infty} \theta^z(\tau-l) = \frac{1}{1-r} \sum_{\tau=1}^\infty \theta^z(\tau)< \infty, \]
where we used that $r \in (0,1)$. This proves that $\sum_{\tau=1}^\infty a_\tau < \infty$, which necessarily implies that $\lim_{\tau \to \infty} a_\tau = 0$ as required.

It is  easy to verify that  under Assumption~\ref{ass:GLipschitz} one has that 
\begin{align*}
\sum_{\tau = 1}^\infty \theta^I(\tau) &\leq \sum_{\tau = 1}^\infty 2 L_I \mathbb{E}[\|\bm{\xi}^I_0\|_2]  \sum_{j=\tau}^\infty w^I_{j} = 2 L_I \mathbb{E}[\|\bm{\xi}^I_0\|_2]  \sum_{j=1}^\infty \sum_{\tau=1}^j w^I_{j} \\
&= 2 L_I \mathbb{E}[\|\bm{\xi}^I_0\|_2] \sum_{j=1}^\infty j w^I_{j} < \infty,
\end{align*}
which immediately implies that condition \eqref{eq:tausummable} is satisfied. Additionally, notice that under Assumption~\ref{ass:ThetaDecay}, condition \eqref{eq:tausummable}  is also automatically satisfied. However, under Assumption~\ref{ass:ThetaAlgDecay} condition \eqref{eq:tausummable} may not be satisfied, but a straightforward argument (see the proof of part~{\bf (iii)} of Corollary~\ref{cor:ThetaWeighting}) shows that  $\lim_{\tau \to \infty} a_\tau = 0$.
\end{remark}

We just argued that under any of the three assumptions~\ref{ass:GLipschitz}-\ref{ass:ThetaAlgDecay}, the convergence to zero of  the sequence $\{a_\tau\}_{\tau \in \mathbb{N}^+}$ is guaranteed, which implies that if we establish the finiteness and a certain decay of the  Rademacher complexity term we shall have proved that the upper bounds given in  \eqref{eq:Rademacher} and \eqref{eq:RademacherAbs} are finite and tend to $0$ as $n \to \infty$. The rate of convergence of these bounds is however affected by the particular dependence assumption adopted. We address this important issue in the following corollary where we assume that the Rademacher complexity is finite and exhibits a certain decay and we prove decay rates for the bounds in \eqref{eq:Rademacher} and \eqref{eq:RademacherAbs} that are valid under the different assumptions~\ref{ass:GLipschitz}-\ref{ass:ThetaAlgDecay}. The boundedness of the Rademacher complexities is studied in detail later on in Section \ref{Rademacher complexity of reservoir systems} for the different hypothesis classes of reservoir systems that we introduced in Section~\ref{Families of reservoir systems}

\begin{corollary} 
\label{cor:ThetaWeighting}
Assume that there exists $C_{RC} >0$ such that for all $k \in \mathbb{N}^+$ the Rademacher-type complexity $\mathcal{R}_k(\mathcal{H}^{RC})$ of the class $\mathcal{H}^{RC}$ of reservoir functionals satisfies
\begin{equation}
\label{cor:ThetaWeighting:ass}
\mathcal{R}_k(\mathcal{H}^{RC}) \leq \frac{C_{RC}}{\sqrt{k}}.
\end{equation}
Consider the following three cases that correspond to Assumptions \ref{ass:GLipschitz},~\ref{ass:ThetaDecay}, and~\ref{ass:ThetaAlgDecay}, respectively:
\begin{description}
\item[{\bf (i)}] Suppose that Assumption~\ref{ass:GLipschitz} holds and that, additionally, for $I=y,z$ the weighting sequences  $w^I: \mathbb{N} \longrightarrow (0, 1]$ are such that the associated decay ratios $D_{w^I}:=\sup_{i \in \mathbb{N}} \left\{\frac{w^I_{i+1}}{w^I_i}\right\}<1$. Let $\lambda_{max} := \max(r,D_{w^y},D_{w^z})$. Then, there exist $C_1$, $C_2$, $C_3$, $C_{3,abs} >0$ such that  for all $n \in \mathbb{N}^+$ satisfying $\log(n)<n\log(\lambda_{max}^{-1})$ it holds that
 \begin{equation} 
 \label{eq:RademacherCor} 
  \mathbb{E}\left[ \sup_{H \in \mathcal{H}^{RC}} \left\{R(H) - \widehat{R}_n^\infty(H)  \right\} \right]
 \leq \frac{C_1}{n} + \frac{C_2 {\log(n)}}{{n}} + \frac{C_3 \sqrt{\log(n)}}{\sqrt{n}}
 \end{equation}
 and 
 \begin{equation} 
 \label{eq:RademacherCorAbs} 
 \mathbb{E}\left[  \sup_{H \in \mathcal{H}^{RC}} \left|R(H)- \widehat{R}_n^\infty(H)\right| \right]
  \leq \frac{C_1}{n} + \frac{C_2 {\log(n)}}{{n}} + \frac{C_{3,abs} \sqrt{\log(n)}}{\sqrt{n}}.  
  \end{equation}
 The constants can be explicitly chosen as 
 \begin{align} 
 \label{eq:C1C2def} 
 C_1 & = \frac{2M_\mathcal{F} L_L  \overline{L_h}  + L_L C_y}{\lambda_{max}}, \quad\quad
C_2  = \frac{2M}{\log(\lambda_{max}^{-1})} +  \frac{L_L L_R  \overline{L_h}  C_z}{\lambda_{max}\log(\lambda_{max}^{-1})},\\
 \label{eq:C3def} 
C_3 &= \frac{2 \sqrt{m} L_L     C_{RC} }{\sqrt{\log(\lambda_{max}^{-1})}}, \quad\quad
 C_{3, abs} = 2 C_3 + \frac{4 L_L \mathbb{E}\left[ \| {\bf Y}_{0} \|^2_2 \right]^{1/2}}{\sqrt{\log(\lambda_{max}^{-1})}}, 
 \end{align}
 where $M$ is as in \eqref{eq:Mdef}
and $C_I=\dfrac{2 L_I \mathbb{E}[\|\bm{\xi}^I_0\|_2]}{1-D_{w^I}}  $ for $I = y,z$.
\item[{\bf (ii)}] Suppose that Assumption~\ref{ass:ThetaDecay} holds and let $\lambda_{max} := \max(r,\lambda_y,\lambda_z)$ with 
$\lambda_y,\lambda_z$ as in \eqref{eq:tauExpDecay}. Then there exist $C_1$, $C_2$, $C_3$, $C_{3,abs} >0$ such that  for all $n \in \mathbb{N}^+$ satisfying $\log(n)<n\log(\lambda_{max}^{-1})$ the bounds in \eqref{eq:RademacherCor} and 
 \eqref{eq:RademacherCorAbs} hold. The 
 constants can be explicitly chosen as in \eqref{eq:C1C2def}-\eqref{eq:C3def} with $C_I$ as in \eqref{eq:tauExpDecay}. 
\item[{\bf (iii)}] Suppose that Assumption~\ref{ass:ThetaAlgDecay} holds and denote $\alpha := \min(\alpha_y,\alpha_z)$. Then there exist ${C}_1, C_2, C_{1,abs}>0$ such that  for all $n \in \mathbb{N}^+$ it holds that
\[ \mathbb{E}\left[ \sup_{H \in \mathcal{H}^{RC}}\left\{R(H) -  \widehat{R}_n^\infty(H)\right\} \right]
 \leq {C}_1{n^{-\frac{1}{2+\alpha^{-1}}}} + {C_2}{n^{-\frac{2}{2+\alpha^{-1}}}}\]
 and 
 \begin{equation*} 
 \mathbb{E}\left[  \sup_{H \in \mathcal{H}^{RC}} \left| R(H) -\widehat{R}_n^\infty(H)\right| \right]
  \leq  {C}_{1,abs} n^{-\frac{1}{2+\alpha^{-1}}}  + {C_2}{n^{-\frac{2}{2+\alpha^{-1}}}}.  
  \end{equation*}
\end{description}
 The constants can be explicitly chosen as 
 \begin{align}
 \label{eq:C1C2exp} 
 C_1 & = L_L(2 M_\mathcal{F}\overline{L_h} r^{-\gamma_{ \alpha}} + L_R \overline{L_h} C_z C_{ \alpha}  + C_y) + B C_{RC}, \quad  C_2 = 2M, \\
  \label{eq:C1absexp} 
C_{1, abs} &= C_1 + {4 L_L \mathbb{E}\left[ \| {\bf Y}_{0} \|^2_2 \right]^{1/2}} + B C_{RC}, 
 \end{align}
 with $M$, $B$ as in \eqref{eq:Cdef}-\eqref{eq:Mdef} and
 \begin{align}
\label{gamma_alpha}
\gamma_{ \alpha} &= \max_{\tau \in \mathbb{N}^+} \left\{ \frac{\log( \tau )\alpha_z}{\log(r^{-1})} - \frac{\tau}{4}\right\},\\
\label{C_alpha}
C_\alpha &= \max(2^{\alpha_z},r^{-\gamma_\alpha})(1-\sqrt{r})^{-1},
\end{align}
and $C_I$, $\alpha_I$, for $I=y,z$ as in \eqref{eq:tauAlgDecay}.
\end{corollary}

\subsection{Rademacher complexity of reservoir systems}
\label{Rademacher complexity of reservoir systems}

In this section we show that for the most important hypothesis classes of reservoir systems, the Rademacher complexity tends to $0$ as $k \to \infty$ at the rates required in \eqref{cor:ThetaWeighting:ass}  of Corollary~\ref{cor:ThetaWeighting}.  More specifically, in the next propositions we will provide upper bounds for the Rademacher complexities of the most popular reservoir families that we spelled out in Section~\ref{Families of reservoir systems}. 

 \subsubsection*{Reservoir systems with linear reservoir map (LRC) and linear readout}

We now provide a bound for the Rademacher complexity of classes of reservoir functionals associated to  linear reservoir maps and readouts. We recall that in this case we always work with uniformly bounded inputs ${\bf Z}$ (see Section~\ref{Families of reservoir systems}) by some constant $M>0$, that is, $D_d = \overline{B_{M}}$ and so the random variable ${\bf Z}$ takes values in the set 
\begin{equation}
\label{recall km}
K _M  := \left\{ {\bf z} \in (\mathbb{R}^d)^{ \mathbb{Z}_-} | \| {\bf z}_t \|_2 \leq M \ {\rm for \ all} \  t\in \mathbb{Z}_- \right\}.
\end{equation}
\begin{proposition}
\label{prop:LinCase} 
Let $N, d \in \mathbb{N}^+$ and let $\Theta \subset \mathbb{M}_{N} \times \mathbb{M}_{N,d} \times \mathbb{R}^N$. Define the classes of linear reservoir maps as
\[ \mathcal{F}^{RC} := \{F^{A,C,\bm{\zeta}} \, | \, (A,C,\bm{\zeta}) \in \Theta \} \]
 and let $\mathcal{H}^{RC}$ be a class of reservoir functionals of the type defined in \eqref{eq:Hfixed}, associated to reservoir systems with linear reservoir maps  and readouts. Additionally, define 
\begin{align}
\label{defLRCAmax}
 \lambda^A_{ max} := &   \sup_{(A,C,\bm{\zeta}) \in \Theta} \vertiii{A}_2, \\
 \label{defLRCCmax}
 \lambda^C_{ max} := & \sup_{(A,C,\bm{\zeta}) \in \Theta}  \vertiii{C}_2, \\
  \label{defLRCzetamax}
 \lambda^{\bm{\zeta}}_{ max} := & \sup_{(A,C,\bm{\zeta}) \in \Theta} \|\bm{\zeta}\|_2 .
\end{align}
If for the class $\mathcal{F}^{RC}$ it holds that
\begin{equation} 
\label{eq:LRCCond}
 \begin{aligned}
 0<\lambda^A_{ max}<1, \quad  \lambda^C_{ max}  < \infty, \quad \lambda^{\bm{\zeta}}_{ max}  < \infty,
\end{aligned}
\end{equation}
then Assumptions~\ref{ass:FLipschitz}-\ref{ass:XBounded} are satisfied and the Rademacher complexity of the associated class of reservoir functionals satisfies
\begin{align}
\label{RadLRC}
\mathcal{R}_k(\mathcal{H}^{RC}) \leq \frac{C_{LRC}}{\sqrt{k}},
\end{align}
for any  $k \in \mathbb{N}^+$, where 
\begin{equation}
\label{CLRC}
C_{LRC}= \dfrac{ \overline{L_h}}{ 1-\lambda_{max}^A} \left(\lambda_{max}^C   {\mathbb{E} \left[ \left\|   {{\bf Z}}_{0} \right\|_2^2   \right]^{1/2} } + \lambda_{max}^{\bm{\zeta}} \right) + L_{h,0},
\end{equation}
and with ${\bf Z}$  the input process.
\end{proposition}

\begin{remark}
\normalfont
Due to the uniform boundedness of the inputs, the constant ${\mathbb{E} \left[ \left\|   {{\bf Z}}_{0} \right\|_2^2   \right]^{1/2} } $ in \eqref{CLRC} is bounded by the value $M$ that defines the set $K _M$ in which the inputs take values.  Nevertheless, ${\mathbb{E} \left[ \left\|   {{\bf Z}}_{0} \right\|_2^2   \right]^{1/2} } $ can obviously be much smaller than $M$.
\end{remark}

\subsubsection*{Echo State Networks (ESN)}

The following proposition provides an estimate for the Rademacher complexity of hypothesis classes constructed using echo state networks. 
\begin{proposition}
\label{prop:ESNCase} 
Let $N, d \in \mathbb{N}^+$, let $\Theta \subset \mathbb{M}_{N} \times \mathbb{M}_{N,d} \times \mathbb{R}^N$ be a subset, and let $\mathcal{F}^{RC}$ be a   family of echo state reservoir systems defined as
\[ \mathcal{F}^{RC} := \{F^{\sigma, A,C,\bm{\zeta}} \, | \, (A,C,\bm{\zeta}) \in \Theta \}, \]
Suppose that the class $\mathcal{F}^{RC}$ is such that for any $F^{\sigma, A,C,\bm{\zeta}} \in \mathcal{F}^{RC} $  one necessarily has that $-F^{\sigma,A,C,\bm{\zeta}}(-\cdot,\cdot) \in \mathcal{F}^{RC}$.\footnote{This is satisfied for example if $\sigma$ is odd and $(A,C,\bm{\zeta}) \in \Theta \Leftrightarrow (A,-C,-\bm{\zeta}) \in \Theta$.}  Define 
 \begin{align}
 \lambda^A_{ max} := & L_\sigma \sum_{l=1}^N  \sup_{(A,C,\bm{\zeta}) \in \Theta} \|A_{l,\cdot}\|_\infty, \\
  \lambda^C_{ max} := & L_\sigma \sum_{l=1}^N  \sup_{(A,C,\bm{\zeta}) \in \Theta} \|C_{l,\cdot}\|_2 , \\
  \lambda^{\bm{\zeta}}_{ max} := & L_\sigma \sum_{l=1}^N  \sup_{(A,C,\bm{\zeta}) \in \Theta} |\zeta_{l}|.
\end{align}
If for class $\mathcal{F}^{RC}$ it holds that
\begin{equation} 
\label{eq:ESNCond}
 \begin{aligned}
 0<\lambda^A_{ max}<1, \quad  \lambda^C_{ max}  < \infty, \quad \lambda^{\bm{\zeta}}_{ max}  < \infty,
\end{aligned}
\end{equation}
then Assumptions~\ref{ass:FLipschitz}-\ref{ass:FESP} are satisfied and the Rademacher complexity of the associated class of reservoir functionals satisfies
\begin{align}
\label{RadESN}
\mathcal{R}_k(\mathcal{H}^{RC}) \leq \frac{C_{ESN}}{\sqrt{k}},\quad \mbox{for any  $k \in \mathbb{N}^+$,} 
\end{align}
where 
\begin{equation}
\label{CESN}
C_{ESN}= \dfrac{\overline{L_h}}{ 1-\lambda_{max}^A} \left(\lambda_{max}^C   {\mathbb{E} \left[ \left\|   {{\bf Z}}_{0} \right\|_2^2   \right]^{1/2} } + \lambda_{max}^{\bm {\zeta}}  \right) + L_{h,0},
\end{equation}
and with ${\bf Z}$  the input process.
\end{proposition}
\begin{remark} 
\normalfont
Notice that by \eqref{RadLRC}-\eqref{CLRC} and by \eqref{RadESN}-\eqref{CESN} the Rademacher complexities of the hypothesis classes formed by reservoir systems with linear reservoir maps and linear readouts or by echo state networks are finite whenever  the second moment of the input process is finite, which is not directly implied by any of the assumptions \ref{ass:GLipschitz}-\ref{ass:ThetaAlgDecay} and hence needs to be separately assumed.
 
\end{remark}

\subsubsection*{State Affine Systems (SAS)}

In the following proposition we provide  an estimate for the Rademacher complexity of hypothesis classes constructed using state affine systems. In this case we also work with uniformly bounded inputs  (see Section~\ref{Families of reservoir systems}) in a set of the type $K _M $ as in \eqref{recall km} with $M=1$.
\begin{proposition}\label{prop:SASCase} Let $\Theta \subset \mathbb{M}_{N,N}[{\bf z}] \times \mathbb{M}_{N,1}[{\bf z}]$, and define the class of SAS reservoir maps as 
\[ \mathcal{F}^{RC} := \{F^{p,q} \, | \, (p,q) \in \Theta \}. \] Assume that there is a finite set $I_{max} \subset \mathbb{N}^d$ such that for any $P({\bf z})= \sum_{\bm{\alpha} \in \mathbb{N}^d} A_{\bm{\alpha}} {\bf z}^{\bm{\alpha}} $ with $P=p$ or $P=q$, $(p,q) \in \Theta$ one has $A_{\bm{\alpha}}=0$ for ${\bm{\alpha}} \notin I_{max}$ and define $|I_{max}|:= {\rm card}(I_{max})$, 
\begin{equation}
\label{eq:SAScond} 
\begin{aligned} 
\lambda^{SAS} &  := \sup_{(p,q) \in \Theta } \vertiii{ p},\\
c^{SAS} & := \sup_{(p,q) \in \Theta } \vertiii{  q }, \end{aligned} 
\end{equation}
where the norm $\vertiii{\cdot }$ was introduced in \eqref{norm for later sas}. Let $\mathcal{H}^{RC} $ be the hypothesis class of reservoir systems with linear readouts associated to $\mathcal{F}^{RC} $ as in \eqref{eq:Hfixed}.
Then, if for the class $\mathcal{F}^{RC} $ it holds that $\lambda^{SAS}< 1/|I_{max}|$ and $c^{SAS} < \infty$, then Assumptions~\ref{ass:FLipschitz}-\ref{ass:XBounded} are satisfied and for any  $k \in \mathbb{N}^+$ it holds that
\begin{equation*}
\mathcal{R}_k(\mathcal{H}^{RC}) \leq \frac{C_{SAS}}{\sqrt{k}}
\end{equation*} 
with
\begin{equation}
\label{CSAS}
C_{SAS} = \overline{L_h}\frac{c^{SAS}|I_{max}|}{1-|I_{max}|\lambda^{SAS}} + L_{h,0}.
\end{equation}
\end{proposition}

\subsection{High-probability risk bounds for reservoir systems} 
\label{High probability bounds for reservoir systems}
We now use the previous results and the three  assumptions~\ref{ass:GLipschitz}-\ref{ass:ThetaAlgDecay} in conjunction with  different concentration inequalities to produce three families of high-probability bounds for $\overline{\Delta}_n := \sup_{H \in \mathcal{H}} | \widehat{R}_n(H) - R(H) |$ of different strength for reservoir systems, which prove in passing  the (weak) universal risk-consistency of ERM for reservoir functionals. High-probability finite-sample generalization RC bounds of this type were not available in the literature previously.

\begin{theorem} 
\label{thm:main1} 
Let $\mathcal{H}^{RC}$ be a hypothesis class of reservoir functionals of the type specified in \eqref{eq:Hfixed} associated to a class $\mathcal{F}^{RC}$ of reservoir maps that satisfies Assumptions~\ref{ass:FLipschitz}-\ref{ass:XBounded} and assume that the Rademacher complexity of $\mathcal{H}^{RC}$  satisfies \eqref{cor:ThetaWeighting:ass}. Suppose that both the input  ${\bf Z}$ and the target ${\bf Y} $ processes have a \textit{causal Bernoulli shift} structure as in \eqref{eq:ZFMPDef} and that they take values in $(\mathbb{R}^d)^{\mathbb{Z}_{-}} $ and $(\mathbb{R}^m)^{\mathbb{Z}_{-}} $, $d,m\in \mathbb{N}^+$, respectively. 
\begin{description}
\item[(i)] Suppose that Assumption~\ref{ass:GLipschitz} is satisfied and that, additionally, for $I=y,z$ the strictly decreasing weighting sequences  $w^I: \mathbb{N} \longrightarrow (0, 1]$ are such that the associated decay ratios $D_{w^I}:=\sup_{i \in \mathbb{N}} \frac{w^I_{i+1}}{w^I_i}<1$. Let $\lambda_{max} := \max(r,D_{w^{y}},D_{w^z})$. 
\begin{itemize}
\item [\normalfont{\bf(a)}] Assume that the innovations are bounded, that is, there exists $\overline{M} >0$ such that $\|\bm{\xi}_t \|_2 \leq \overline{M}$ for all $t \in \mathbb{Z}_-$. Then there exist constants $C_0, C_1, C_2, C_3, C_{bd} > 0$ such that for all $n \in \mathbb{N}^+$ satisfying $\log(n)<n \log(\lambda_{max}^{-1})$ and for all $\delta \in (0,1)$, the following bound holds
\begin{align}
\label{eq:boundBounded}
\!\!\!\!\!\!\!\!\!\!\!\!\!\!\!\!\!\!\!\!\!\!\!\!\!\!\!\!\!\!\!\!\!\!\!\!\!\!\!\!\!\!\!\!\!\!\mathbb{P}\Bigg( \sup_{H \in \mathcal{H}^{RC}} |\widehat{R}_n(H) - R(H) | &\leq \frac{(1-r^n)C_0+C_1}{n} + \frac{C_2 {\log(n)}}{{n}} + \frac{C_3 \sqrt{\log(n)}}{\sqrt{n}} + \frac{C_{bd} \sqrt{\log(\frac{4}{\delta})}}{\sqrt{2 n}} \Bigg) \geq 1-\delta,
\end{align}
where the constant $C_0$ is explicitly given in \eqref{eq:C0def}, $C_1, C_2$ are given in \eqref{eq:C1C2def}, $C_3$ in \eqref{eq:C3def}, and $C_{bd}$ in \eqref{eq:Cbd}.
\item [\normalfont{\bf(b)}] Assume that for $\Phi(x)=x^p$, $p > 1$ or $\Phi(x)=\exp(x)-1$ the innovations possess  $\Phi^2$-moments, that is, for any $u > 0$, $\mathbb{E}[\Phi(u \|\bm{\xi}_0 \|_2)^2] < \infty$. Then there exist constants $C_0, C_1, C_2, C_{3}>0$ such that for all $n \in \mathbb{N}^+$ satisfying $\log(n)<n \log(\lambda_{max}^{-1})$ and for all $\delta \in (0,1)$ it holds that
\begin{align}
\!\!\!\!\!\!\!\!\!\!\!\!\!\!\!\!\!\!\!\!\!\!\!\!\!\!\!\!\!\!\!\!\!\!\!\!\!\!\!\!\!\!\!\!\!\!\mathbb{P}\Bigg( \sup_{H \in \mathcal{H}^{RC}} |\widehat{R}_n(H) - R(H) | &\leq \frac{(1-r^n)C_0+C_1}{n} + \frac{C_2 {\log(n)}}{{n}} + \frac{C_{3} \sqrt{\log(n)}}{\sqrt{n}} + B_\Phi(n,\delta) \Bigg) \geq 1-\delta,
\end{align}
where $B_\Phi(n,\delta)$ is given in \eqref{eq:BPhi}. The constants are explicitly given: $C_0$ in \eqref{eq:C0def}, $C_1, C_2$ are given in \eqref{eq:C1C2def}, and $C_{3}$ in \eqref{eq:C3def}.
\end{itemize}
\item[(ii)] Suppose that Assumption~\ref{ass:ThetaDecay} is satisfied and let $\lambda_{max} := \max(r,\lambda_y,\lambda_z)$ with 
$\lambda_y,\lambda_z$ as in \eqref{eq:tauExpDecay}.  Then there exist constants $C_0$, $C_1$, $C_2$,~$C_{3,abs} >0$~such that for all $n \in \mathbb{N}^+$ satisfying $\log(n)<n \log(\lambda_{max}^{-1})$ and for all $\delta \in (0,1)$ it holds that  \begin{equation}\label{eq:boundAss2}
\!\!\!\!\!\!\!\!\!\!\!\!\!\!\!\!\!\!\!\!\mathbb{P}\left( \sup_{H \in \mathcal{H}^{RC}} |\widehat{R}_n(H) - R(H) | \leq \frac{(1-r^n)C_0}{n} + \frac{2}{\delta} \left(\frac{C_1}{n} + \frac{C_2 {\log(n)}}{{n}} + \frac{C_{3,abs} \sqrt{\log(n)}}{\sqrt{n}}\right) \right) \geq 1-\delta.
\end{equation}
The constants are explicitly given: $C_0$ in \eqref{eq:C0def}, $C_1, C_2$ are given in \eqref{eq:C1C2def}, and $C_{3,abs}$ in \eqref{eq:C3def} with $C_I$ as in \eqref{eq:tauExpDecay}.
\item [(iii)] Suppose that Assumption~\ref{ass:ThetaAlgDecay} is satisfied. Denote $\alpha = \min(\alpha_y,\alpha_z)$  with 
$\alpha_y,\alpha_z$ as in \eqref{eq:tauAlgDecay}. Then there exist constants $C_0$, ${C}_{1,abs}$, $ C_2 > 0$ such that for all $n \in \mathbb{N}^+$, $\delta \in (0,1)$,   \begin{equation}
\mathbb{P}\left( \sup_{H \in \mathcal{H}^{RC}} |\widehat{R}_n(H) - R(H) | \leq \frac{(1-r^n)C_0}{n} + \frac{2}{\delta} \left({C}_{1,abs} n^{-\frac{1}{2+\alpha^{-1}}}  + {C_2}{n^{-\frac{2}{2+\alpha^{-1}}}} \right) \right) \geq 1-\delta.
\end{equation}
The constants are explicitly given: $C_0$ in \eqref{eq:C0def}, $C_{1,abs}$ is given in  \eqref{eq:C1absexp} and $C_2$ in \eqref{eq:C1C2exp} together with \eqref{gamma_alpha}-\eqref{C_alpha}.
\end{description}
\end{theorem}

In order to obtain explicit high-probability risk bounds for particular families of reservoir systems, one can use the bounds that we obtained for the Rademacher complexities of various families in Section \ref{Rademacher complexity of reservoir systems}. For example, let $\mathcal{F}^{RC}$ be a family of state affine systems that satisfies the assumptions of Proposition~\ref{prop:SASCase}; in that case one takes the value $C_{RC}$ appearing in various constants above (for example in $C_3$ given in \eqref{eq:C3def}) as  $C_{RC} = C_{SAS}$ with $C_{SAS}$ given in \eqref{CSAS}. The same applies to other families: for echo state networks one takes $C_{RC} = C_{ESN}$ with $C_{ESN}$ given in \eqref{CESN}. For the family of reservoir systems with linear reservoir map one takes $C_{RC} = C_{LRC}$, with $C_{LRC}$ given in \eqref{CLRC}.


\begin{remark}
\normalfont
The result in part {\bf (ii)} requires Assumption~\ref{ass:ThetaDecay} which, as we saw in Remark \ref{one implies 2}, is implied  by Assumption~\ref{ass:GLipschitz} with geometric weighting sequences, that is, $w^z_{s} = \lambda_z^s$ and $w^y_{s}= \lambda_y^s$, for some $\lambda_z,\lambda_y \in (0,1)$.
Therefore, both \eqref{eq:boundBounded} and \eqref{eq:boundAss2} provide bounds in this case. We also emphasize that the result in part {\bf (iii)} allows the treatment of long-memory processes as inputs (see, for instance Example \ref{ARFIMA process}). 
\end{remark}

\subsection{High-probability risk bounds for randomly generated reservoir systems} 
\label{High-probability risk bounds for randomly generated reservoir systems}

We now show that the results in Theorem~\ref{thm:main1}  can be reformulated for echo state networks whose parameters $A $, $C $,  and $\boldsymbol{\zeta} $ have been randomly generated. This statement is a theoretical justification of the good empirical properties of this standard {\it modus operandi} in reservoir computing. Even though  results of this type could be formulated for all the reservoir families introduced in Section \ref{Families of reservoir systems} and all the different settings considered in Theorem~\ref{thm:main1}, we restrict our study in the next proposition to echo state networks and part {\bf (ii)}.

\begin{proposition}\label{prop:ESNrandomparams} {\bf (Random reservoirs)} Let ${\bf A},{\bf C},\bm{\zeta}$  be independent random variables with values in  $\mathbb{M}_{N}$, $\mathbb{M}_{N,d}$, and in $\mathbb{R}^N$, respectively. Consider now echo state networks that have those random values as parameters and whose activation function $\sigma$ is odd, that is, consider the random class of reservoir maps defined as
\[ \bm{\mathcal{F}}^{RC} := \{F^{ \sigma, \rho_A \mathbf{A} ,\rho_C \mathbf{C},\rho_{\bm{\zeta}} \boldsymbol{\zeta} } \, | \, (\rho_A,\rho_C,\rho_{\bm{\zeta}}) \in (-\frac{a}{\lambda ^{\boldsymbol{A}}},\frac{a}{\lambda ^{\boldsymbol{A}}}) \times [-c,c] \times [-s,s]  \} \]
for some  $a \in (0,1)$, $c, s > 0$ and with 
\[\begin{aligned} \lambda ^{\bf A}  = & L_\sigma \left(\sum_{l=1}^N  \|{\bf A}_{l,\cdot}\|_\infty \right).
 \end{aligned}
 \]
Suppose also that the input process ${\bf Z}$ and the target process ${\bf Y}$ are independent of the parameter random variables ${\bf A},{\bf C},\bm{\zeta}$ and that Assumption~\ref{ass:ThetaDecay} is satisfied. Let $\lambda_{max} := \max(r,\lambda_y,\lambda_z)$ with 
$\lambda_y,\lambda_z$ as in \eqref{eq:tauExpDecay}. Then there exist constants $C_0$, $C_1$, $C_2$, $C_{3,abs} >$ 0 such that for all $n \in \mathbb{N}^+$ satisfying $\log(n)<n \log(\lambda_{max}^{-1})$ and for all $\delta \in (0,1)$ it holds that  \begin{equation}
\mathbb{P}\left( \sup_{H \in \boldsymbol{\mathcal{H}}^{RC}} |\widehat{R}_n(H) - R(H) | \leq \frac{(1-r^n)C_0}{n} + \frac{2}{\delta} \left(\frac{C_1}{n} + \frac{C_2 {\log(n)}}{{n}} + \frac{C_{3,abs} \sqrt{\log(n)}}{\sqrt{n}}\right) \right) \geq 1-\delta,
\end{equation}
where $\bm{\mathcal{H}}^{RC} $ is the hypothesis class of reservoir functionals associated to $\bm{\mathcal{F}}^{RC} $ and with linear readouts as in \eqref{eq:Hfixed}.
The constants are explicitly given. More specifically, $C_0$ is given in \eqref{eq:C0def}, $C_1, C_2$ are given in \eqref{eq:C1C2def}, $C_{3,abs}$ in \eqref{eq:C3def} with $C_I$ as in \eqref{eq:tauExpDecay}. Additionally, the constant $C_{RC}$ appearing in \eqref{eq:C3def} is given by 
\begin{equation} \label{eq:CRCRandom} 
C_{RC} = \dfrac{\overline{L_h}}{ 1-a}(\mathbb{E}[ \lambda ^{\bf C}]  \mathbb{E}[\|{\bf Z}_{0}\|_2^2]^{1/2} + \mathbb{E}[ \lambda ^{\bm{ \zeta}}]) + L_{h,0},
\end{equation}
where
\[\begin{aligned} 
 \lambda ^{\bf C} =  c L_\sigma \left(\sum_{l=1}^N  \|{\bf C}_{l,\cdot}\|_2\right), \, 
  \lambda ^{\bm{ \zeta}} = s L_\sigma \left(\sum_{l=1}^N   \|\zeta_{l}\|_2 \right).
 \end{aligned}
 \]
\end{proposition}

\section{Appendices}
\label{app:Preliminary results}

These appendices contain preliminary results and the proofs of all the main results of the paper.  

\subsection{Preliminary results}
\label{Preliminary results}


The following Lemma shows that the supremum appearing for instance in \eqref{eq:boundBounded} is indeed a random variable. More precisely,  $\sup_{H \in \mathcal{H}^{RC}} |\widehat{R}_n(H)-R(H)|$ is a measurable mapping from $(\Omega,\mathcal{A})$ to $\mathbb{R}$ equipped with its Borel sigma-algebra. An analogous argument can be used in the case of the other suprema, for instance the supremum in \eqref{eq:Rcomplexity}, considered in the paper. 

\begin{lemma}
\label{lem:supmeasurable} 
Let $\mathcal{H}^{RC} $ be the reservoir hypothesis class introduced in \eqref{eq:Hfixedgeneral} or in \eqref{eq:Hfixed} and let $R$ and $\widehat{R}_n$  be the statistical and empirical risk introduced in \eqref{eq:riskDef} and \eqref{eq:empiricalRiskDef}, respectively. Then, 
\begin{equation*}
\sup_{H \in \mathcal{H}^{RC} } |\widehat{R}_n(H)-R(H)|
\end{equation*}
is a random variable.  
\end{lemma}

\begin{proof} For any $H \in \mathcal{H}^{RC}$ set $\Delta(H):=|\widehat{R}_n(H)-R(H)| $. Then, for any $H, \bar{H} \in \mathcal{H}^{RC}$
\begin{align}
\label{eq:auxEq47}
\Delta(H) - \Delta(\bar{H}) &\leq |\Delta(H)-\Delta(\bar{H})| \leq \vert \widehat{R}_n(H)-R(H)  - \widehat{R}_n(\bar{H})+R(\bar{H}) \vert \nonumber\\
& \leq |\widehat{R}_n(H)-\widehat{R}_n(\bar{H})|+|R(H)-R(\bar{H})| \nonumber
\\ & \leq \frac{L_L}{n}\sum_{i=0}^{n-1}\|H({\bf Z}_{-i}^{-n+1})-\bar{H}({\bf Z}_{-i}^{-n+1})\|_2 + L_L \mathbb{E}[\|H({\bf Z})-\bar{H}({\bf Z})\|_2] \nonumber
\\ & \leq 2 L_L \sup_{{\bf z} \in (D_d)^{\mathbb{Z}_-}} \|H({\bf z})-\bar{H}({\bf z})\|_2,
\end{align}
where we used the (reverse) triangle inequality and  the Lipschitz property \eqref{eq:lossLipschitz} of the loss function. 
Further, by using the definition \eqref{eq:Hfixedgeneral} of $\mathcal{H}^{RC}$,  Assumption~\ref{ass:XBounded} on the   boundedness of the functionals associated to reservoir maps, and again the triangle inequality, one obtains
\[\sup_{H \in \mathcal{H}^{RC}} \Delta(H) =  \sup_{H \in \mathcal{H}^{RC}} |\widehat{R}_n(H)-R(H)| \leq \Delta(0) + 4 L_L [\overline{L_h} M_\mathcal{F} + L_{h,0}] \]
and so \eqref{eq:lossMoment} yields that $\sup_{H \in \mathcal{H}^{RC}} |\widehat{R}_n(H)-R(H)|$ is finite, $\mathbb{P}$-a.s.

It remains to prove the measurability. The separability assumption imposed on $\mathcal{F}^{RC}$ guarantees the existence of a countable subset $\{F_j \}_{j \in \mathbb{N}^+} \subset \mathcal{F}^{RC}$ which is dense with respect to the supremum norm. Let also $\{h_k\}_{k \in \mathbb{N}^+} \subset \mathcal{F}^{O}$ be a countable dense subset of readouts that by hypothesis exists. This can be used to construct a countable dense subset of $\mathcal{H}^{RC}$.  Indeed, for any $H \in \mathcal{H}^{RC}$, $H = h(H^F)$, one may choose indices $(j_l,k_l)_{l \in \mathbb{N}^+}$ such that $\|F_{j_l} - F \|_\infty \to 0$ and $\|h_{k_l} - h \|_\infty \to 0$ as $l \to \infty$. Consequently, using the triangle inequality and an argument as in part {\bf (iii)} of Theorem~3.1 in \cite{RC7}, one obtains for any ${\bf z} \in (D_d)^{\mathbb{Z}_-}$ that
\[\begin{aligned} 
\|H({\bf z})-h_{k_l}(H^{F_{j_l}}({\bf z}))\|_2 & \leq  \|h-h_{k_l} \|_\infty +  \overline{L_h} \|H^{F}({\bf z})-H^{F_{j_l}}({\bf z})\|_2 \\ & \leq  \|h-h_{k_l} \|_\infty + \frac{1}{1-r}\| F_{j_l} - F \|_\infty. 
\end{aligned}\]
Combining this with \eqref{eq:auxEq47} and setting $H_l = h_{k_l}(H^{F_{j_l}})$ one obtains that
\[ \lim_{l \to \infty} |\Delta(H)-\Delta(H_l)| \leq \lim_{l \to \infty} 2 L_L \left( \|h-h_{k_l} \|_\infty + \frac{1}{1-r}\| F_{j_l} - F \|_\infty \right) = 0. \]
In particular this shows that for any $H \in \mathcal{H}^{RC}$, $\Delta(H) \leq \sup_{j,k \in \mathbb{N}^+} \Delta(h_{k}(H^{F_j}))$. Taking the supremum over $H \in \mathcal{H}^{RC}$ thus shows that 
\[ \sup_{H \in \mathcal{H}^{RC}} |\widehat{R}_n(H)-R(H)| = \sup_{j,k \in \mathbb{N}^+} \Delta(h_{k}(H^{F_j})) \]
is measurable, as it is the supremum of a countable collection of random variables. 
\end{proof}

\subsection{Proof of Proposition~\ref{prop:contractionEchoState}}
 
Consider the map 
\begin{equation}
\label{script f as product definition}
\begin{array}{cccc}
\mathcal{F}: & (\overline{B_S})^{\mathbb{Z}_{-}} \times (D_d)^{\mathbb{Z}_-}&\longrightarrow &(\overline{B_S})^{\mathbb{Z}_{-}}\\
	&(\mathbf{x}, {\bf z})&\longmapsto & \left(\mathcal{F}(\mathbf{x}, {\bf z})\right)_t:=F(\mathbf{x} _{t-1}, {\bf z}_t), \, t \in \mathbb{Z}_-,
\end{array}
\end{equation}
and endow $(D _d)^{\mathbb{Z}_{-}} $ and $(\overline{B_S})^{\mathbb{Z}_{-}} $ with the relative topologies induced by the product topologies in $({\mathbb{R}}^d)^{\mathbb{Z}_{-}} $ and $({\mathbb{R}}^N)^{\mathbb{Z}_{-}} $, respectively. Notice that $\mathcal{F} $ can be written as 
\begin{equation}
\label{script f as product}
\mathcal{F}=\prod_{t \in \mathbb{Z}_{-}} F  _t  \quad \mbox{with} \quad F _t:= F \circ p _t \circ  \left(T _1\times {\rm id}_{(D _d)^{\mathbb{Z}_{-}} }\right): (\overline{B_S})^{\mathbb{Z}_{-}} \times (D _d)^{\mathbb{Z}_{-}} \longrightarrow  \overline{B_S},
\end{equation}
where $p _t $ yields the canonical projection of any sequence onto its $t $-th component.
It is easy to see that the maps $p _t $  and $T _1 $ are continuous with respect to the product topologies and hence $\mathcal{F} $ is a Cartesian product of continuous functions, which is always continuous in the product topology. 

We now recall that, by the compactness of $\overline{B_S} $, we have that $(\overline{B_S})^{\mathbb{Z}_{-}} \subset \ell_{-}^{\infty,w}(\mathbb{R}^N)$ and that by Corollary 2.7 in \cite{RC7}, the product topology on $(\overline{B_S})^{\mathbb{Z}_{-}} $ coincides with the norm topology induced by $\ell_{-}^{\infty,w}(\mathbb{R}^N)$, for any weighting sequence $w$, that we choose in the sequel satisfying the inequality $r L _w<1 $. 

We now show that $\mathcal{F}  $ is a contraction in the first entry with constant $r L _w<1 $. Indeed, for any $\mathbf{x}^1, {\bf x}^2 \in (\overline{B_S})^{\mathbb{Z}_{-}}$ and any ${\bf z} \in (D_d)^{\mathbb{Z}_{-}} $, we have
\begin{multline}
\label{contraction part 1}
\left\|\mathcal{F}(\mathbf{x}^1, {\bf z})-\mathcal{F}(\mathbf{x}^2 , {\bf z})\right\|_{\infty,w}=
\sup_{t \in \mathbb{Z}_{-}}\left\{\left\|F(\mathbf{x}_{t-1}^1, {\bf z}_t)-F(\mathbf{x}_{t-1}^2, {\bf z}_t)\right\|_2 w_{-t}\right\}\\
\leq 
\sup_{t \in \mathbb{Z}_{-}}\left\{\left\|\mathbf{x}_{t-1}^1-\mathbf{x}^2_{t-1}\right\|_2    r w_{-t}\right\},
\end{multline}
where we used that $F$ is a contraction in the first entry. Now, 
\begin{equation*}
\label{contraction part 2}
\sup_{t \in \mathbb{Z}_{-}}\left\{\left\|\mathbf{x}_{t-1}^1-\mathbf{x}^2_{t-1}\right\|_2r w_{-t}\right\}=
r\sup_{t \in \mathbb{Z}_{-}}\left\{\left\|\mathbf{x}_{t-1}^1-\mathbf{x}^2_{t-1}\right\|_2 w_{-(t-1)}\frac{w _{-t}}{w _{-(t-1)}}\right\}\leq
r L _w \left\|\mathbf{x}^1- \mathbf{x}^2 \right\|_{\infty,w},
\end{equation*}
which shows that $\mathcal{F} $ is a family of contractions with constant $r L _w<1 $ that is continuously parametrized by the elements in $(D_d)^{\mathbb{Z}_{-}}$.
In view of these facts and given that the product topology in $(D _d)^{\mathbb{Z}_{-}} \subset ({\mathbb{R}}^d)^{\mathbb{Z}_{-}} $ is metrizable (see \cite[Theorem 20.5]{Munkres:topology}) and that $(\overline{B_S})^{\mathbb{Z}_{-}} \subset ({\mathbb{R}}^N)^{\mathbb{Z}_{-}} $ is compact by  Tychonoff's Theorem (see \cite[Theorem 37.3]{Munkres:topology}) in the product topology and hence complete, Theorem 6.4.1 in \cite{Sternberg:dynamical:book} implies the existence of a unique fixed point of $\mathcal{F}(\cdot,{\bf z}) $ for each ${\bf z} \in (D _d)^{\mathbb{Z}_{-}} $, which establishes the ESP. Moreover, that result also shows the continuity of the associated filter $U ^F: (D _d)^{\mathbb{Z}_{-}}  \longrightarrow ((\overline{B_S})^{\mathbb{Z}_{-}}, \left\|\cdot \right\| _{\infty,w}) $. $\blacksquare$

\subsection{Proof of Proposition~\ref{prop:finiteHistoryError}}
In order to proceed with the proof of this proposition, we first need the following lemma.
\begin{lemma} 
\label{lem:HFLipschitz}
For any $F \in \mathcal{F}^{RC}$ and any ${\bf z},{\bf \overline{z}} \in (D_d)^{\mathbb{Z}_-}$ the following holds for all $i \in \mathbb{N}^+$:
\begin{equation} \label{eq:HEstimate}
\|H^F({\bf z})-H^F({\bf \overline{z}})\|_2 \leq 2 r^i M_\mathcal{F} + L_R \sum_{j=0}^{i-1} r^j \|{\bf z}_{-j}-{\bf \overline{z}}_{-j}\|_2. \end{equation}
\end{lemma}
{\bf Proof of Lemma~\ref{lem:HFLipschitz}.} Let ${\bf z},{\bf \overline{z}} \in (D_d)^{\mathbb{Z}_-}$ and denote by $\mathbf{x}$ the solution to \eqref{eq:RCSystemDet} and by $\overline{\mathbf{x}}$ the solution to \eqref{eq:RCSystemDet} with ${\bf z}$ replaced by ${\bf \overline{z}}$. Then the  triangle inequality and Assumption~\ref{ass:FLipschitz}  on $F \in \mathcal{F}^{RC}$ yield
\[\begin{aligned} \|H^F({\bf z})-H^F({\bf \overline{z}})\|_2 & = \|\mathbf{x}_0-\overline{\mathbf{x}}_0 \|_2 \\
& \leq \| F(\mathbf{x}_{-1},{\bf z}_0) - F(\mathbf{x}_{-1},\overline{{\bf z}}_0) \|_2 + \|F(\mathbf{x}_{-1},\overline{{\bf z}}_0)- F(\overline{\mathbf{x}}_{-1},\overline{{\bf z}}_0) \|_2 \\
& \leq L_R \|{\bf z}_0- \overline{{\bf z}}_0 \|_2 + r \|\mathbf{x}_{-1}-\overline{\mathbf{x}}_{-1} \|_2.
  \end{aligned} \]
By iterating this estimate one obtains
\[ \|H^F({\bf z})-H^F({\bf \overline{z}})\|_2 \leq r^i \|\mathbf{x}_{-i}-\overline{\mathbf{x}}_{-i} \|_2+ L_R \sum_{j=0}^{i-1} r^j \|{\bf z}_{-j}-{\bf \overline{z}}_{-j}\|_2, \]
from which the claim follows by Assumption~\ref{ass:XBounded}. $\blacktriangledown$

\noindent We now proceed to prove Proposition~\ref{prop:finiteHistoryError}. Let $\widetilde{{\bf Z}} := {\bf Z}_{0}^{-n+1}$ and write  for any $H \in \mathcal{H}^{RC}$ 
\begin{align*} 
| \widehat{R}_n(H)-\widehat{R}_n^\infty(H) | & = \left| \frac{1}{n} \sum_{i=0}^{n-1} L(H({\bf Z}_{-i}^{-n+1}),{\bf Y}_{-i}) -  L(H({\bf Z}_{-i}^{-\infty}),{\bf Y}_{-i}) \right|
\\ & \leq \frac{1}{n} \sum_{i=0}^{n-1} L_L \| h(H^F(\widetilde{{\bf Z}}_{-i}^{-\infty}))- h(H^F({\bf Z}_{-i}^{-\infty}))\|_2 
\\ & \leq \frac{1}{n} \sum_{i=0}^{n-1} L_L \overline{L_h }  (2 r^{n-i} M_\mathcal{F} + L_R \sum_{j=0}^{n-i-1} r^j \|\widetilde{{\bf Z}}_{-j-i} - {\bf Z}_{-j-i}\|_2)
\\ & = \frac{2 L_L \overline{L_h } M_\mathcal{F}}{n} \sum_{i=0}^{n-1}  r^{n-i} \\
&= \frac{1-r^n}{n} \frac{2r L_L \overline{L_h }  M_\mathcal{F}}{1-r} \\ & 
= \frac{1-r^n}{n} C_0.
\end{align*}
In these derivations, the first inequality follows from  the triangle inequality and the Lipschitz continuity of the loss function \eqref{eq:lossLipschitz}, the second one is a consequence of the Lipschitz continuity of the readout map and of  \eqref{eq:HEstimate} in Lemma~\ref{lem:HFLipschitz}, which  finally yield the claim with the choice of constant $C_0$ in \eqref{eq:C0def}. $\blacksquare$

\subsection{Proof of Proposition~\ref{prop:expBoundWithRademacher}}

In order to simplify the notation, for any $i \in \mathbb{N}$ we define an $(\mathbb{R}^d)^{\mathbb{Z}_-} \times \mathbb{R}^m$-valued random variable ${\bf V}_{-i}$ as ${\bf V}_{-i}:=({\bf Z}_{-i}^{- \infty},{\bf Y}_{-i})$  and denote its associated loss by $L_H({\bf V}_{-i}) := L(H({\bf Z}_{-i}^{- \infty}),{\bf Y}_{-i})$. We start by using  the assumptions on the Lipschitz-continuity of both the loss function \eqref{eq:lossLipschitz} and of the reservoir readout map and hence for any $i \in \mathbb{N}$ and $H\in \mathcal{H}^{RC}$ write
\begin{align*} 
\left|L_H({\bf V}_{-i})\right| & \leq \left|L_H({\bf V}_{-i}) - L(\mathbf{0},{\bf Y}_{-i})\right| + \left|L(\mathbf{0},{\bf Y}_{-i})\right|\\ 
& \leq L_L \|H({\bf Z}_{-i}^{- \infty})\|_2 + \left|L(\mathbf{0},{\bf Y}_{-i})\right| \\
& \leq L_L \|h(H^F({\bf Z}_{-i}^{- \infty})) - h({\bf 0}) \|_2 + L_L \|h({\bf 0}) \|_2 + \left|L(\mathbf{0},{\bf Y}_{-i})\right| \\
& \leq L_L  \overline{L_h } M_\mathcal{F} + \left|L(\mathbf{0},{\bf Y}_{-i})\right| + L_{h,0} L_L. 
\end{align*}
We continue by decomposing $n = k\tau + (n-k\tau)$ with $k = \floor{\dfrac{n}{\tau}}$. For the last $(n-k\tau)$ elements one estimates
\begin{align}\label{eq:auxEq25} 
\mathbb{E}\left[ \sup_{H \in \mathcal{H}^{RC}}  \sum_{i=k \tau}^{n-1} \left\{\mathbb{E}\left[L_H({\bf V}_{-i}) \right] - L_H({\bf V}_{-i}) \right\}
\right] &
\leq 2 (n-k \tau) (L_L\overline{L_h }  M_\mathcal{F} + \mathbb{E} \left[\left|L({\bf 0},{\bf Y}_{0})\right|\right] + L_{h,0}L_L) \nonumber\\
&= 2 M (n-k \tau)
\end{align}
with $M$ as in \eqref{eq:Mdef}.
Subsequently using the definitions of the generalization error \eqref{eq:riskDef} and the idealized empirical risk \eqref{eq:empiricalRiskFullHistoryDef} one obtains
\begin{align}
\label{eq:blocking} 
\mathbb{E}&\left[ \sup_{H \in \mathcal{H}^{RC}} \left\{ R(H)- \widehat{R}_n^\infty(H)  \right\}\right]\nonumber \\
 & = \mathbb{E}\left[ \sup_{H \in \mathcal{H}^{RC}}  \mathbb{E}[L(H({\bf Z}),{\bf Y}_0)] - \left\{ \frac{1}{n} \sum_{i=0}^{n-1} L(H({\bf Z}_{-i}^{-\infty}),{\bf Y}_{-i})\right\}   \right]\nonumber  \\
 &=\mathbb{E}\left[ \sup_{H \in \mathcal{H}^{RC}}  \frac{1}{n} \sum_{i=0}^{n-1} \left\{ \mathbb{E}[L_H({\bf V}_{-i})] - L_H({\bf V}_{-i})   \right\}\right]\nonumber\\
&  \leq \frac{1}{n} \mathbb{E}\left[ \sup_{H \in \mathcal{H}^{RC}}  \sum_{i=0}^{k\tau-1} \left\{\mathbb{E}[L_H({\bf V}_{-i})] - L_H({\bf V}_{-i})  \right\}
\right] + \frac{2M(n-k\tau)}{n}\nonumber  \\
& = \frac{1}{n}\mathbb{E}\left[ \sup_{H \in \mathcal{H}^{RC}}  \sum_{j=0}^{k-1} \sum_{i=0}^{\tau-1} \left\{\mathbb{E} \left[L_H({\bf V}_{-(\tau j+i)}) \right] - L_H({\bf V}_{-(\tau j+i)})  \right\}\right] + \frac{2M(n-k\tau)}{n} \nonumber \\
& \leq \frac{\tau}{n}\mathbb{E}\left[ \sup_{H \in \mathcal{H}^{RC}}  \sum_{j=0}^{k-1} \left\{\mathbb{E}[L_H({\bf V}_{-\tau j})] - L_H({\bf V}_{-\tau j}) \right\} \right] + \frac{2M(n-k\tau)}{n}.
\end{align}
In order to obtain a bound for the first summand in the last expression, we introduce  ghost samples and  use  tools that hinge on the independence between them.  Let $\overline{\bm{\xi}}^{(j)}=(\overline{\bm{\xi}}^{y,(j)}_t,\overline{\bm{\xi}}^{z,(j)}_t)_{t \in \mathbb{Z}_-}$, $j=0,\ldots,k-1$ denote independent copies of $\bm{\xi}$. Next, for $I=y,z$ define $\bm{\xi}^{I,(j)}$ by setting $\bm{\xi}^{I,(j)}_{i}=\bm{\xi}_{i}^I$ for $i=-\tau (j+1)+1,\ldots,0$ and $\bm{\xi}^{I,(j)}_{i}=\overline{\bm{\xi}}_{i}^{I,(j)}$ for $i \leq -\tau (j+1)$. Additionally, let $\overline{{\bf Z}}_t^{(j)} := G^z(\ldots, \bm{\xi}^{z,(j)}_{t - 1}, \bm{\xi}^{z,(j)}_{t})$ and $\overline{{\bf Y}}_t^{(j)} := G^y(\ldots, \bm{\xi}^{y,(j)}_{t-1}, \bm{\xi}^{y,(j)}_{t})$ for $t \in \mathbb{Z}_-$ and define the $(\mathbb{R}^d)^{\mathbb{Z}_-} \times \mathbb{R}^m$-valued random variables ${\bf U}^{(j)}:=(\overline{{\bf Z}}_{- \tau j}^{-\infty, (j)},\overline{{\bf Y}}_{- \tau j}^{(j)})$, $j=0,\ldots,k-1$. Then one has
\begin{align*}
 \overline{{\bf Z}}_{t - \tau j}^{(j)} &= G^z(\ldots, \bm{\xi}^{z,(j)}_{t-\tau j - 1}, \bm{\xi}^{z,(j)}_{t-\tau j})\\
&= \begin{cases} G^z(\ldots,\overline{\bm{\xi}}^{z,(j)}_{-\tau j-\tau},\bm{\xi}^z_{-\tau j-\tau+1},\ldots,\bm{\xi}^z_{t-\tau j}),  &t=-\tau+1 ,\ldots,0, \\
G^z(\ldots,\overline{\bm{\xi}}^{z,(j)}_{t-\tau j-1},\overline{\bm{\xi}}^{z,(j)}_{t-\tau j}),  &t\leq-\tau
\end{cases}
\end{align*}
and so, for any $j=0,\ldots,k-1$, the random variable ${\bf U}^{(j)}$ is measurable with respect to the $\sigma$-algebra generated by $(\bm{\xi}_{-\tau j+t})_{t=-\tau+1,\ldots,0}$ and $\overline{\bm{\xi}}^{(j)}$. The  assumption of independence between the ghost samples implies that ${\bf U}^{(0)},\ldots,{\bf U}^{(k-1)}$ are also independent and identically distributed with the same distribution as ${\bf V}_0= ({\bf Z}, {\bf Y}_0)$ introduced above. Hence one can rewrite the first summand of the right hand side of the last inequality in \eqref{eq:blocking} as
\begin{align}\label{eq:auxEq20}
\mathbb{E} & \left[ \sup_{H \in \mathcal{H}^{RC}}  \sum_{j=0}^{k-1} \left\{\mathbb{E}[L_H({\bf V}_{-\tau j})]  - L_H({\bf V}_{-\tau j})  \right\} \right]   \nonumber \\
& \leq \mathbb{E} \left[ \sup_{H \in \mathcal{H}^{RC}}  \sum_{j=0}^{k-1} \left\{\mathbb{E}[L_H({\bf V}_{-\tau j})] - L_H({\bf U}^{(j)})  \right\}\right] 
+  \mathbb{E} \left[ \sup_{H \in \mathcal{H}^{RC}}  \sum_{j=0}^{k-1} \left\{L_H({\bf U}^{(j)}) - L_H({\bf V}_{-\tau j})  \right\}\right] \nonumber \\
& = \mathbb{E} \left[ \sup_{H \in \mathcal{H}^{RC}}  \sum_{j=0}^{k-1} \left\{\mathbb{E}[L_H({\bf U}^{(j)})] - L_H({\bf U}^{(j)})  \right\}\right] +  \mathbb{E} \left[ \sup_{H \in \mathcal{H}^{RC}}  \sum_{j=0}^{k-1} \left\{L_H({\bf U}^{(j)}) - L_H({\bf V}_{-\tau j}) \right\}\right].
\end{align}
We now analyze these two terms separately. 
For the second term, we first note that for any $H \in \mathcal{H}^{RC}$ it holds that
\begin{align}\label{lossH} 
\left| L_H({\bf V}_{-\tau j}) - L_H({\bf U}^{(j)})\right|&= \left|L(H({\bf Z}_{-\tau j}^{-\infty}),{\bf Y}_{-\tau j}) -  L(H(\overline{{\bf Z}}_{-\tau j}^{-\infty, (j)}),\overline{{\bf Y}}_{- \tau j}^{(j)})\right| \nonumber \\
& \leq L_L \left( \|H({\bf Z}_{-\tau j}^{-\infty})-H(\overline{{\bf Z}}_{-\tau j}^{-\infty, {(j)}})\|_2 + \|{\bf Y}_{-\tau j} - \overline{{\bf Y}}_{- \tau j}^{(j)} \|_2 \right)
. \end{align}
Next, we use the Lipschitz-continuity of the readout maps (see \eqref{eq:Hfixedgeneral}) together with the estimate \eqref{eq:HEstimate} in Lemma \ref{lem:HFLipschitz} and compute
\begin{equation} 
\label{eq:auxEq11} 
\begin{aligned} 
\sup_{H \in \mathcal{H}^{RC}} & \|H({\bf Z}_{-\tau j}^{-\infty})-H(\overline{{\bf Z}}_{-\tau j}^{-\infty, {(j)}})\|_2 
 \leq 2 r^{\tau} M_\mathcal{F} \overline{L_h }  + L_R \overline{L_h }  \sum_{l=0}^{\tau-1} r^l \|{\bf Z}_{-l-\tau j}-\overline{{\bf Z}}^{(j)}_{-l-\tau j}\|_2 .
 \end{aligned}\end{equation}
Combining \eqref{eq:auxEq11} with \eqref{lossH} we now estimate the second term in \eqref{eq:auxEq20} by
\begin{align}\label{eq:auxEq21}
\mathbb{E} & \left[ \sup_{H \in \mathcal{H}^{RC}}  \sum_{j=0}^{k-1}  \left\{L_H({\bf U}^{(j)}) - L_H({\bf V}_{-\tau j})  \right\}\right] \nonumber \\
& \leq 
L_L  \sum_{j=0}^{k-1} \left(2 r^{\tau} M_\mathcal{F} \overline{L_h }  + \mathbb{E}[\|{\bf Y}_{-\tau j} - \overline{{\bf Y}}_{- \tau j}^{(j)} \|_2] + L_R \overline{L_h }  \sum_{l=0}^{\tau-1} r^l \mathbb{E}[ \|{\bf Z}_{-l-\tau j}-\overline{{\bf Z}}^{(j)}_{-l-\tau j}\|_2]   \right)  \nonumber \\
& = 
 k   a_\tau
\end{align}
with $a_\tau$ as in \eqref{eq:atau}.

In order to estimate the first term in \eqref{eq:auxEq20}, one relies on techniques which are common in the  case of independent and identically distributed random variables. Here we start by introducing real Rademacher random variables  $\varepsilon_0,\ldots,\varepsilon_{k-1}$ (see for example \cite[Definition 3.2.9]{AnalysisBanachSpaces:vol1}), which are independent of all the other random variables considered so far. In what follows we need to use the structure for the loss function introduced in \eqref{defLoss} as well as the other hypotheses on it that we spelled out in Section \ref{Learning procedure for reservoir systems}. The fact that the loss functions are a sum of representing functions $f_i: \mathbb{R} \longrightarrow \mathbb{R}$, implies that their evaluation on the hypothesis class ${\mathcal H}^{RC} $ can be expressed through the evaluation of each representing function on the sets $\mathcal{H}^{RC}_i  $, $i \in \left\{1,\ldots, m\right\}$, defined by
\begin{equation}
\label{HRCi}
\mathcal{H}^{RC}_i := \{ \widetilde{H} \colon (D_d)^{\mathbb{Z}_-} \times \mathbb{R}^m \to \mathbb{R} \mid \widetilde{H}(\mathbf{x},{\bf y}) := (H(\mathbf{x}))_i-y_i, H \in \mathcal{H}^{RC} \}.
\end{equation}
Using independence and a symmetrization trick by \cite{gine1984} (see for example \citep*[Lemma~6.3]{ledoux:talagrand} or the proof of \citep*[Theorem~8]{Bartlett2003}) one writes
\begin{align}\label{eq:auxEq19}
\frac{1}{k}\mathbb{E}  \left[ \sup_{H \in \mathcal{H}^{RC}}  \sum_{j=0}^{k-1} \left\{ \mathbb{E}[L_H({\bf U}^{(j)})] - L_H({\bf U}^{(j)}) \right\}\right] & \leq \frac{2}{k} \mathbb{E} \left[ \sup_{H \in \mathcal{H}^{RC}}  \sum_{j=0}^{k-1} \varepsilon_j L_H({\bf U}^{(j)}) \right] \nonumber \\ 
&\leq \frac{2}{k} \sum_{i=1}^m  \mathbb{E} \left[ \sup_{\widetilde{H} \in \mathcal{H}^{RC}_i}  \sum_{j=0}^{k-1} \varepsilon_j (f_i \circ \widetilde{H})({\bf U}^{(j)}) \right]. 
\end{align}
Applying the contraction principle for  Rademacher random variables (see \citep*[Lemma~26.9]{Ben-David2014} and also \citep*[Theorem~4.12]{ledoux:talagrand}) to the last expression one obtains
\begin{align}
\frac{1}{k}\mathbb{E} & \left[ \sup_{H \in \mathcal{H}^{RC}}  \sum_{j=0}^{k-1} \left\{ \mathbb{E}[L_H({\bf U}^{(j)})] -L_H({\bf U}^{(j)}) \right\}\right] \nonumber \\ 
& \leq \frac{2L_L}{\sqrt{m} k} \sum_{i=1}^m  \mathbb{E} \left[ \sup_{\widetilde{H} \in \mathcal{H}^{RC}_i}  \sum_{j=0}^{k-1} \varepsilon_j \widetilde{H}({\bf U}^{(j)}) \right]\nonumber \\ 
& = \frac{2L_L}{\sqrt{m} k} \sum_{i=1}^m  \mathbb{E} \left[ \sup_{{H} \in \mathcal{H}^{RC}}    \sum_{j=0}^{k-1} \varepsilon_j \left(H(\overline{{\bf Z}}_{- \tau j}^{-\infty,(j)})  - \overline{{\bf Y}}_{- \tau j}^{(j)}\right)_i\right]\nonumber \\
& \leq \frac{2L_L}{\sqrt{m} k } \sum_{i=1}^m \left(\mathbb{E} \left[ \sup_{{H} \in \mathcal{H}^{RC}} \left\|  \sum_{j=0}^{k-1} \varepsilon_j H(\overline{{\bf Z}}_{- \tau j}^{-\infty,(j)}) \right\|_2 \right]  + \mathbb{E}\left[ - \sum_{j=0}^{k-1} \varepsilon_j (\overline{{\bf Y}}_{- \tau j}^{(j)})_i \right] \right)\nonumber\\
& \leq \frac{2\sqrt{m} L_L  }{ k} \mathbb{E} \left[ \sup_{{H} \in \mathcal{H}^{RC}} \left\|  \sum_{j=0}^{k-1} \varepsilon_j H(\overline{{\bf Z}}_{- \tau j}^{-\infty,(j)}) \right\|_2 \right] =  {2\sqrt{m} L_L  } \mathcal{R}_k(\mathcal{H}^{RC}),\label{eq1}
\end{align}
with the Rademacher complexity defined as in \eqref{eq:Rcomplexity}. Note that the last term in the fourth line  is equal to zero due to the independence  and the fact that the expectation of Rademacher random variables is zero.
 
We now come back to the estimate of the expected maximum difference between the in-class statistical risk and the idealized empirical risk and rewrite the expression \eqref{eq:blocking}  using \eqref{eq:auxEq20} and \eqref{eq:auxEq21} 
 \begin{align}
\label{eq:blockingFinal} 
&\mathbb{E}\left[ \sup_{H \in \mathcal{H}^{RC}} \left\{ R(H)- \widehat{R}_n^\infty(H)  \right\}\right] 
   \leq \frac{\tau}{n} \left\{{2 k\sqrt{m}L_L  }{}\mathcal{R}_k(\mathcal{H}^{RC})  +  ka_\tau\right\} + \frac{2M(n-k\tau)}{n},
\end{align}
 which then yields \eqref{eq:Rademacher} as required.
 
It remains to prove \eqref{eq:RademacherAbs}. To do so,  notice that  the triangle inequality and the same arguments used in \eqref{eq:blocking}, \eqref{eq:auxEq20}, \eqref{eq:auxEq21} and \eqref{eq:auxEq19}  may be applied in the presence of absolute values to obtain 
\begin{align} 
\label{eq:auxEq22} 
\mathbb{E}\left[  \sup_{H \in \mathcal{H}^{RC}} \left|\widehat{R}_n^\infty(H)  - R(H) \right| \right]
& \leq \frac{k \tau}{n} a_\tau +\frac{ 2 \tau}{n} \sum_{i=1}^m  \mathbb{E} \left[ \sup_{\widetilde{H} \in \mathcal{H}^{RC}_i}  \left| \sum_{j=0}^{k-1} \varepsilon_j (f_i \circ \widetilde{H})({\bf U}^{(j)}) \right| \right] \nonumber\\
&\quad+ \frac{2M(n-k\tau)}{n}.  
 \end{align}
 Applying again the contraction principle for Rademacher random variables  \citep*[Theorem~12.4]{Bartlett2003}, one estimates, for any $i=1,\ldots,m$, 
 \begin{equation}
 \label{eq:auxEq23} 
 \begin{aligned}
 \frac{1}{k} & \sum_{i=1}^m \mathbb{E}  \left[ \sup_{\widetilde{H} \in \mathcal{H}^{RC}_i} \left| \sum_{j=0}^{k-1} \varepsilon_j (f_i \circ \widetilde{H})({\bf U}^{(j)}) \right| \right] \\ 
 & \leq \frac{2 L_L}{\sqrt{m} k} \sum_{i=1}^m \mathbb{E} \left[ \sup_{\widetilde{H} \in \mathcal{H}^{RC}_i} \left| \sum_{j=0}^{k-1} \varepsilon_j  \widetilde{H}({\bf U}^{(j)}) \right| \right]\\ 
 & \leq \frac{2 L_L}{\sqrt{m} k} \sum_{i=1}^m \left( \mathbb{E} \left[ \sup_{H \in \mathcal{H}^{RC}}  \left| \sum_{j=0}^{k-1} \varepsilon_j (H(\overline{{\bf Z}}_{- \tau j}^{- \infty, (j)}) )_i\right| \right] + \mathbb{E}\left[ \left| \sum_{j=0}^{k-1} \varepsilon_j (\overline{{\bf Y}}_{- \tau j}^{(j)})_i \right| \right]\right)\\ 
 & \leq \frac{2 L_L}{\sqrt{m} k} \left(m  \mathbb{E} \left[ \sup_{H \in \mathcal{H}^{RC}} \left\|  \sum_{j=0}^{k-1} \varepsilon_j H(\overline{{\bf Z}}_{- \tau j}^{- \infty, (j)}) \right\|_2 \right] + \sum_{i=1}^m \mathbb{E}\left[ \left| \sum_{j=0}^{k-1} \varepsilon_j (\overline{{\bf Y}}_{- \tau j}^{(j)})_i \right|^2 \right]^{1/2} \right) \\
 & \leq \frac{2 L_L}{\sqrt{m} k} \left(m k  \mathcal{R}_k(\mathcal{H}^{RC})  + \sum_{i=1}^m \mathbb{E}\left[ \left| \sum_{j=0}^{k-1} \varepsilon_j (\overline{{\bf Y}}_{- \tau j}^{(j)})_i \right|^2 \right]^{1/2} \right)
 \end{aligned}
\end{equation}
with the Rademacher complexity defined as in \eqref{eq:Rcomplexity}.
Finally, using the independence of the Rademacher sequence and the ghost samples $(\overline{{\bf Y}}^{(j)})_{j = 0,\ldots, k-1}$ as well as the stationarity properties of the latter, one has
 \[\sum_{i=1}^m  \left(\mathbb{E}\left[ \left| \sum_{j=0}^{k-1} \varepsilon_j (\overline{{\bf Y}}_{- \tau j}^{(j)})_i \right|^2 \right] \right)^{1/2} = \sum_{i=1}^m \left( \sum_{j=0}^{k-1} \mathbb{E}\left[ (\overline{{\bf Y}}_{- \tau j}^{(j)})_i^2 \right]\right)^{1/2}
 \leq \sqrt{k} \sqrt{m} \mathbb{E}\left[ \| {\bf Y}_{0} \|_2^2 \right]^{1/2}. \]
 The combination of this inequality with \eqref{eq:auxEq22} and \eqref{eq:auxEq23} yields \eqref{eq:RademacherAbs}, as required.   $\blacksquare$

\subsection{Proof of Corollary~\ref{cor:ThetaWeighting}}
{\bf Proof of part {\bf(i)}} We start by noticing that under Assumption~\ref{ass:GLipschitz}, the weak dependence coefficients $\theta^I$ defined in \eqref{eq:ThetaDefRepeat}  for $I=y,z$ and $\tau \in \mathbb{N}^+$ can be estimated as 
\begin{align*} 
\theta^I(\tau) & = \mathbb{E}[\|G^I(\ldots,\bm{\xi}_{-1}^I,\bm{\xi}_{0}^I)-G^I(\ldots,\widetilde{\bm{\xi}}_{-\tau-1}^I, \widetilde{\bm{\xi}}_{-\tau}^I,\bm{\xi}_{-\tau+1}^I,\ldots,\bm{\xi}_{0}^I) \|_2] \\
&  \leq \mathbb{E}\left[L_I  \sum_{j=\tau}^{\infty}w^I_j \|\bm{\xi}_{-j}^I - \widetilde{\bm{\xi}}_{-j}^I\|_2 \right] \\
& \leq 2 L_I \mathbb{E}[\|\bm{\xi}^I_0\|_2]  \sum_{j=\tau}^\infty  w^I_{j} \leq 2 L_I \mathbb{E}[\|\bm{\xi}^I_0\|_2]  \sum_{j=\tau}^\infty  (D_{w^I})^{j} = 2 L_I \mathbb{E}[\|\bm{\xi}^I_0\|_2] \dfrac{(D_{w^I})^{\tau}}{1-D_{w^I}},
\end{align*}
where we used that by hypothesis $D_{w^I}<1$. Consequently, if we set $C_I:=\dfrac{2 L_I \mathbb{E}[\|\bm{\xi}^I_0\|_2]}{1-D_{w^I}}  $, condition  \eqref{eq:tauExpDecay} does hold for all  $\tau \in \mathbb{N}^+ $. We now define $c_0:=2 L_L \overline{L_h}M_\mathcal{F}$, $c_1 := L_L L_R C_z \overline{L_h} $, and $c_2:=L_L C_y$ and with the notation $\lambda_{max} := \max(r,D_{w^y},D_{w^z}) \in (0,1)$  write \eqref{eq:atau} for any $\tau \in \mathbb{N}^+$ as
 \begin{align}
\label{atauBound}
a_\tau &\leq  c_0 r^\tau  + c_1 \sum_{l=0}^{\tau-1} r^l (D_{w^z})^{\tau-l} +  c_2 (D_{w^y})^{\tau} \leq \lambda_{max} ^\tau(c_0 + \tau c_1 + c_2).
\end{align}
Next, let  $\tau \in \mathbb{N}^+$ with $\tau < n$ and set $k = \lfloor n / \tau \rfloor$. Inserting assumption \eqref{cor:ThetaWeighting:ass} in  \eqref{eq:Rademacher} and then using that $n/\tau - 1 \leq k \leq n/\tau$, one obtains
\begin{equation} 
\label{eq:auxEq45}
 \begin{aligned} 
 \mathbb{E}\left[ \sup_{H \in \mathcal{H}^{RC}} \left\{R(H) - \widehat{R}_n^\infty(H)  \right\} \right] & 
 \leq \frac{k \tau}{n} a_\tau +\frac{BC_{RC} \sqrt{k} \tau}{n}  + \frac{2M(n-k\tau)}{n} \\
 &  \leq  a_\tau +\frac{B C_{RC} \sqrt{\tau}}{\sqrt{n}}  + \frac{2M \tau}{n}.
 \end{aligned} 
 \end{equation}
Our goal now is to choose the length of the block $\tau$ depending on $\lambda_{max}$. Recall that by hypothesis $\log(n)< n\log(\lambda_{max}^{-1})$, which means that in order to be able to apply the blocking technique for a given value $\lambda_{max} \in (0,1)$ the number of observations $n \in \mathbb{N}^+$ should be sufficiently large. In this situation one can choose $\tau = \lfloor \log(n)/\log(\lambda_{max}^{-1}) \rfloor$, which is then guaranteed to satisfy $\tau < n$. Notice that  then $\lambda_{max}^{\tau+1} \leq {1}/{n}$ and consequently  \eqref{eq:RademacherCor} follows from \eqref{atauBound}  and \eqref{eq:auxEq45} with the appropriate choice of constants as given in \eqref{eq:C1C2def}-\eqref{eq:C3def}.  
Finally, the last term in \eqref{eq:RademacherCorAbs} follows by noticing that
\begin{equation} 
\label{eq:auxEq46} 
\frac{4 \tau \sqrt{k} L_L \mathbb{E}\left[ \| {\bf Y}_{0} \|_2^2 \right]^{1/2}}{n } \leq \frac{4 \sqrt{\tau} L_L \mathbb{E}\left[ \| {\bf Y}_{0} \|_2^2 \right]^{1/2}}{\sqrt{n} },
\end{equation} 
and hence one gets \eqref{eq:RademacherCorAbs} as required.

\medskip

\noindent{\bf Proof of part {\bf(ii)}} Recall that by Assumption~\ref{ass:ThetaDecay} for $I=y,z$ there exist $\lambda_I \in (0,1)$ and $C_I >0$ such that, for all $\tau \in \mathbb{N}^+$, it holds that
\begin{equation*}
 \theta^I(\tau) \leq C_I \lambda_I^\tau.
 \end{equation*}
 Mimicking the proof of part {\bf (i)} with $\lambda_{max}:=\max{(r,  \lambda_y,  \lambda_z)}$ yields the claim. 
\medskip

\noindent{\bf Proof of part {\bf(iii)}} 
By our choice of $\gamma_\alpha $ it holds that $\tau/4 +\gamma_\alpha \geq \log(\tau) \alpha_z / \log(r^{-1}) $ for all $\tau \in \mathbb{N}^+$.
One has $r^{l/2} (\tau-l)^{-\alpha_z} \leq 2^{\alpha_z} \tau^{-\alpha_z}$ for $l \leq \tau/2$ and $r^{l/2} (\tau-l)^{-\alpha_z} \leq r^{\tau/4} \leq r^{-\gamma_\alpha} \tau^{-\alpha_z}$, for $\tau/2 \leq l \leq \tau-1$. Setting $C_\alpha = \max(2^{\alpha_z},r^{-\gamma_\alpha})(1-\sqrt{r})^{-1}$ one has for all $\tau \in \mathbb{N}^+$ that
\[ \sum_{l=0}^{\tau-1} r^l (\tau-l)^{-\alpha_z} \leq \sum_{l=0}^{\infty} r^{l/2} \max(2^{\alpha_z},r^{-\gamma_\alpha}) \tau^{-\alpha_z} = C_\alpha \tau^{-\alpha_z }. \]
Defining $c_0$, $c_1$, and $c_2$ as in the proof of part {\bf (i)}, applying \eqref{eq:tauAlgDecay} and inserting the above estimate thus allows us to bound \eqref{eq:atau} for any $\tau \in \mathbb{N}^+$ by
\[  a_\tau \leq c_0 r^\tau  + c_1 \sum_{l=0}^{\tau-1} r^l (\tau-l)^{-\alpha_z} +  c_2 \tau^{-\alpha_y} \leq \tau^{-\alpha} \left(r^{-\gamma_\alpha} c_0  + C_\alpha c_1 + c_2 \right). \]
Furthermore, \eqref{eq:auxEq45} remains valid and so choosing $\tau = n^{\beta}$ yields
\[\mathbb{E}\left[ \sup_{H \in \mathcal{H}^{RC}} \left\{R(H)- \widehat{R}_n^\infty(H)  \right\}  \right]
 \leq \frac{\left(r^{-\gamma_\alpha} c_0  + C_\alpha c_1 + c_2 \right)}{n^{\alpha \beta}} +\frac{B C_{RC}}{n^{1/2-\beta/2}}  + \frac{2M}{n^{1-\beta}}.
 \] 
Taking $\beta = \frac{1}{2}(\alpha + \frac{1}{2})^{-1}$ yields $1/2-\beta/2 =  \frac{\alpha}{2}(\alpha + \frac{1}{2})^{-1} $ and hence the desired result.
The last term in \eqref{eq:RademacherAbs} can be bounded by proceeding analogously to part {\bf (i)} and noticing that \eqref{eq:auxEq46} remains valid, which concludes the proof. $\blacksquare$

\subsection{Proof of Proposition~\ref{prop:LinCase} (Reservoir systems with linear reservoir and  readout maps)}
The condition \eqref{eq:LRCCond} together with Proposition \ref{prop:contractionEchoState} ensure that the ESP property of the reservoir systems in the hypothesis class is guaranteed and that for any ${\bf z} \in K _M$  we have that $H^{A,C,\bm{\zeta}}( {\bf z}) = \sum^{\infty}_{i = 0}A^i (C {\bf z}_{-i} + \bm{\zeta})$. 
Using the definition of Rademacher complexity, one estimates 
\allowdisplaybreaks
\begin{align}  
\label{exprLRC}
\mathbb{E} &\left[ \sup_{H \in \mathcal{H}^{RC}} \left\|  \sum_{j=0}^{k-1} \varepsilon_j H(\widetilde{{\bf Z}}^{(j)}) \right\|_2 \right]\nonumber \\ 
& =  \mathbb{E}   \left[ \sup_{\substack{(A,C,\bm{\zeta}) \in \Theta\\ W: \vertiii{W}_2\leq \overline{L_h} \\ \boldsymbol{a}: \|\boldsymbol{a}\|_2\leq L_{h,0}}} \left\|  \sum_{j=0}^{k-1} \varepsilon_j (WH^{A,C,\bm{\zeta}}(\widetilde{{\bf Z}}^{(j)})+\boldsymbol{a})\right\|_2 \right] \nonumber \\ 
& \le  \mathbb{E}   \left[ \sup_{\substack{(A,C,\bm{\zeta}) \in \Theta\\ W: \vertiii{W}_2\leq \overline{L_h} }} \left\|  \sum_{j=0}^{k-1} \varepsilon_j WH^{A,C,\bm{\zeta}}(\widetilde{{\bf Z}}^{(j)})\right\|_2 \right] +\mathbb{E}   \left[ \sup_{\boldsymbol{a}: \|\boldsymbol{a}\|_2\leq L_{h,0}} \left\|  \sum_{j=0}^{k-1} \varepsilon_j \boldsymbol{a}\right\|_2 \right] \nonumber \\ 
& \le  \sup_{W: \vertiii{W}_2\leq \overline{L_h}} \vertiii{W}_2 \mathbb{E}   \left[ \sup_{\substack{(A,C,\bm{\zeta}) \in \Theta}} \left\|  \sum_{j=0}^{k-1} \varepsilon_j H^{A,C,\bm{\zeta}}(\widetilde{{\bf Z}}^{(j)})\right\|_2 \right] +\mathbb{E}   \left[ \sup_{\boldsymbol{a}: \|\boldsymbol{a}\|_2\leq L_{h,0}} \left\|  \sum_{j=0}^{k-1} \varepsilon_j \boldsymbol{a}\right\|_2 \right] \nonumber \\ 
& \le \overline{L_h}\sum^{\infty}_{l = 0} \sup_{(A,C,\bm{\zeta}) \in \Theta}\vertiii{A^l}_2 \mathbb{E}   \left[ \sup_{(A,C,\bm{\zeta}) \in \Theta} \left\|  \sum_{j=0}^{k-1} \varepsilon_j (C\widetilde{{\bf Z}}^{(j)}_{-l} + \bm{\zeta})\right\|_2 \right] +\mathbb{E}   \left[ \sup_{\boldsymbol{a}: \|\boldsymbol{a}\|_2\leq L_{h,0}} \left\|  \sum_{j=0}^{k-1} \varepsilon_j \boldsymbol{a}\right\|_2 \right] \nonumber\\
&\le  \overline{L_h}\sum^{\infty}_{l = 0} \sup_{(A,C,\bm{\zeta}) \in \Theta}\vertiii{A^l}_2 \Bigg(\sup_{\substack{(A,C,\bm{\zeta}) \in \Theta}}\vertiii{C}_2\mathbb{E}   \left[  \left\|  \sum_{j=0}^{k-1} \varepsilon_j \widetilde{{\bf Z}}^{(j)}_{-l}\right\|_2 \right] +\sup_{(A,C,\bm{\zeta}) \in \Theta}\|\bm{\zeta}\|_2 \mathbb{E}   \left[ \left|  \sum_{j=0}^{k-1} \varepsilon_j \right| \right]\Bigg)\nonumber \\ 
&\quad +\mathbb{E}   \left[\sup_{\boldsymbol{a}: \|\boldsymbol{a}\|_2\leq L_{h,0}} \left\|  \sum_{j=0}^{k-1} \varepsilon_j \boldsymbol{a}\right\|_2 \right] \nonumber\\
& \leq  \overline{L_h}\sum_{l=0}^\infty (\lambda_{max}^A)^l \left(\lambda_{max}^C  \mathbb{E} \left[ \left\| \sum_{j=0}^{k-1} \varepsilon_j \widetilde{{\bf Z}}_{0}^{(j)}  \right\|_2 \right] + \lambda_{max}^{\bm{\zeta}} \mathbb{E} \left[ \left| \sum_{j=0}^{k-1} \varepsilon_j \right| \right]\right) +\mathbb{E}   \left[ \sup_{ \boldsymbol{a}: \|\boldsymbol{a}\|_2\leq L_{h,0}} \left\|  \sum_{j=0}^{k-1} \varepsilon_j \boldsymbol{a}\right\|_2 \right],
\end{align}
where we used stationarity, the fact that $\vertiii{W}_2 \leq \overline{L_h}$ for all readout maps from the class $H\in \mathcal{H}^{RC}$, and constants as in \eqref{defLRCAmax}-\eqref{defLRCzetamax}.
For the first summand in this expression we have
\allowdisplaybreaks
\[  \mathbb{E} \left[ \left\| \sum_{j=0}^{k-1} \varepsilon_j \widetilde{{\bf Z}}_{0}^{(j)}  \right\|_2 \right]^2 \leq {\mathbb{E} \left[ \left\| \sum_{j=0}^{k-1} \varepsilon_j \widetilde{{\bf Z}}_{0}^{(j)}  \right\|_2^2  \right]}=  {\sum_{j=0}^{k-1} \mathbb{E} \left[ \left\|   \widetilde{{\bf Z}}_{0}^{(j)}  \right\|_2^2   \right]\mathbb{E}[ \varepsilon_j^2] }={{k}} {\mathbb{E} \left[ \left\|   {{\bf Z}}_{0} \right\|_2^2   \right] }, \]
where in the first step  we used  the Jensen inequality, the next equality is obtained using the independence of the ghost samples and the Rademacher sequence and also the fact that   $\mathbb{E}[\varepsilon_{j'} \varepsilon_j] = 0$ when $j \neq j'$ for Rademacher variables.
The second summand in \eqref{exprLRC} is bounded using  the  inequality by \cite{Khintchine1923} 

\begin{equation*}
\mathbb{E} \left[ \left| \sum_{j=0}^{k-1} \varepsilon_j \right| \right] \le \sqrt{k}.
\end{equation*}
We bound the third term as the first one and obtain
\begin{equation}
\label{eq:Radfora}
\mathbb{E}   \left[\sup_{\boldsymbol{a}: \|\boldsymbol{a}\|_2\leq L_{h,0}} \left\|  \sum_{j=0}^{k-1} \varepsilon_j \boldsymbol{a}\right\|_2 \right] ^2
 \leq{{k}} L_{h,0}^2,
\end{equation}
where we took into account that $\|{\boldsymbol{a}}\|_2 \leq L_{h,0}$ for all readout maps from the class $H\in \mathcal{H}^{RC}$.
Finally,   \eqref{exprLRC} can be rewritten as
\begin{align*} 
\mathbb{E} \left[ \sup_{H \in \mathcal{H}^{RC}} \left\|  \sum_{j=0}^{k-1} \varepsilon_j H(\widetilde{{\bf Z}}^{(j)}) \right\|_2 \right] &  
\leq \sqrt{k} \ \overline{L_h} \sum_{l=0}^\infty (\lambda_{max}^A)^l \left(\lambda_{max}^C {\mathbb{E} \left[ \left\|   {{\bf Z}}_{0} \right\|_2^2   \right]^{1/2} } + \lambda_{max}^{\bm {\zeta}}  \right) + {\sqrt{k}} L_{h,0}\\
&=\sqrt{k} \left(\dfrac{ \overline{L_h}}{ 1-\lambda_{max}^A} \left(\lambda_{max}^C   {\mathbb{E} \left[ \left\|   {{\bf Z}}_{0} \right\|_2^2   \right]^{1/2} } + \lambda_{max}^{\bm{\zeta}} \right) + L_{h,0}\right),
 \end{align*}
 where we used that $\lambda_{max}^A \in (0,1)$. Finally,  the choice of constants which satisfy conditions \eqref{eq:LRCCond} yields \eqref{RadLRC}, as required.
$\blacksquare$

\subsection{Proof of Proposition~\ref{prop:ESNCase} (Echo State Networks)}
Firstly, note that for any $\mathbf{x}\in D_N$, it holds that $\left\| \mathbf{x} \right\|_2 \le \left\| \mathbf{x} \right\|_1 = \sum^{N}_{i = 1} | x_i|$ and hence one can write
\[\begin{aligned}
\mathbb{E} \left[ \sup_{H \in \mathcal{H}^{RC}} \left\|  \sum_{j=0}^{k-1} \varepsilon_j H(\widetilde{{\bf Z}}^{(j)}) \right\|_2 \right] 
&= \mathbb{E} \left[ \sup_{\substack{F \in \mathcal{F}^{RC}\\ W: \vertiii{W}_2\leq \overline{L_h} \\ \boldsymbol{a}: \|\boldsymbol{a}\|_2\leq L_{h,0}}}  \left\|  \sum_{j=0}^{k-1} \varepsilon_j (WH^F(\widetilde{{\bf Z}}^{(j)}) + \boldsymbol{a}) \right\|_2 \right]\nonumber\\
&\leq\mathbb{E} \left[ \sup_{\substack{F \in \mathcal{F}^{RC}\\ W: \vertiii{W}_2\leq \overline{L_h} }}  \left\|  \sum_{j=0}^{k-1} \varepsilon_j WH^F(\widetilde{{\bf Z}}^{(j)}) \right\|_2 \right]+\mathbb{E} \left[ \sup_{ \boldsymbol{a}: \|\boldsymbol{a}\|_2\leq L_{h,0}}  \left\|  \sum_{j=0}^{k-1} \varepsilon_j \boldsymbol{a} \right\|_2 \right]\nonumber\\
&\leq \overline{L_h}  \ \sum_{l=1}^N \mathbb{E} \left[\sup_{\substack{(A,C,\bm{\zeta}) \in \Theta}}  \left|  \sum_{j=0}^{k-1} \varepsilon_j H^{\sigma, A, C, \bm{\zeta}}_{l}(\widetilde{{\bf Z}}^{(j)}) \right| \right] + {\sqrt{k}} L_{h,0},
\end{aligned}
\]
where we used the same arguments as in \eqref{eq:Radfora}.
Using the assumed symmetry of the family $\mathcal{F}^{RC}$ in the first step and the contraction principle in the second step one may estimate
\allowdisplaybreaks
\begin{align*}  
\sum_{l=1}^N \mathbb{E} &  \left[ \sup_{\substack{(A,C,\bm{\zeta}) \in \Theta}} \left|  \sum_{j=0}^{k-1} \varepsilon_j H_l^{\sigma, A, C, \bm{\zeta}}(\widetilde{{\bf Z}}^{(j)}) \right| \right] \\ 
& = \sum_{l=1}^N \mathbb{E} \left[ \sup_{\substack{(A,C,\bm{\zeta}) \in \Theta}} \sum_{j=0}^{k-1} \varepsilon_j H_l^{\sigma, A, C, \bm{\zeta}}(\widetilde{{\bf Z}}^{(j)}) \right] \\
& \leq L_\sigma \sum_{l=1}^N \mathbb{E} \left[ \sup_{\substack{(A,C,\bm{\zeta}) \in \Theta}} \sum_{j=0}^{k-1} \varepsilon_j (A_{l,\cdot} H^{\sigma, A, C, \bm{\zeta}}( \widetilde{{\bf Z}}_{-1}^{-\infty,(j)})+ C_{l,\cdot} \widetilde{{\bf Z}}_{0}^{(j)} + \zeta_l) \right]\\
& \leq L_\sigma \sum_{l=1}^N  \sup_{\substack{(A,C,\bm{\zeta}) \in \Theta}} \|A_{l,\cdot}\|_\infty  \mathbb{E} \left[ \sup_{\substack{(A,C,\bm{\zeta}) \in \Theta}}\left\| \sum_{j=0}^{k-1} \varepsilon_j H^{\sigma, A, C, \bm{\zeta}}( \widetilde{{\bf Z}}_{-1}^{-\infty,(j)}) \right\|_1 \right] \\ 
& \quad + L_\sigma \sum_{l=1}^N  \sup_{\substack{(A,C,\bm{\zeta}) \in \Theta}} \|C_{l,\cdot}\|_2 
\mathbb{E} \left[ \left\| \sum_{j=0}^{k-1} \varepsilon_j \widetilde{{\bf Z}}_{0}^{(j)}  \right\|_2 \right] + L_\sigma \sum_{l=1}^N \mathbb{E} \left[ \sup_{\substack{(A,C,\bm{\zeta})\in \Theta}} \sum_{j=0}^{k-1} \varepsilon_j \zeta_l  \right].
\end{align*}
Since $\lambda_{max}^A \in (0,1)$ by assumption and using Assumption~\ref{ass:XBounded}, one may iterate the above inequality to obtain
\allowdisplaybreaks
\begin{align} 
\label{expr}
\sum_{l=1}^N \mathbb{E}   \left[\sup_{\substack{(A,C,\bm{\zeta}) \in \Theta}} \left|  \sum_{j=0}^{k-1} \varepsilon_j H_l^{\sigma, A, C, \bm{\zeta}}(\widetilde{{\bf Z}}^{(j)}) \right| \right] &  \leq \sum_{l=0}^\infty (\lambda_{max}^A)^l \left(\lambda_{max}^C  \mathbb{E} \left[ \left\| \sum_{j=0}^{k-1} \varepsilon_j \widetilde{{\bf Z}}_{0}^j  \right\|_2 \right] + \lambda_{max}^{\bm {\zeta}} \mathbb{E} \left[ \left| \sum_{j=0}^{k-1} \varepsilon_j \right| \right]\right).  \end{align}
For the first summand in this expression we obtain
\[  \mathbb{E} \left[ \left\| \sum_{j=0}^{k-1} \varepsilon_j \widetilde{{\bf Z}}_{0}^{(j)}  \right\|_2 \right]^2 \leq {\mathbb{E} \left[ \left\| \sum_{j=0}^{k-1} \varepsilon_j \widetilde{{\bf Z}}_{0}^{(j)}  \right\|_2^2  \right]}=  {\sum_{j=0}^{k-1} \mathbb{E} \left[ \left\|   \widetilde{{\bf Z}}_{0}^{(j)}  \right\|_2^2   \right]\mathbb{E}[ \varepsilon_j^2] }={{k}} {\mathbb{E} \left[ \left\|   {{\bf Z}}_{0} \right\|_2^2   \right] }, \]
where in the first step  we used   Jensen's inequality, the next equality is obtained using the independence of the ghost samples and the Rademacher sequence and also the fact that   $\mathbb{E}[\varepsilon_{j'} \varepsilon_j] = 0$ for $j \neq j'$ for Rademacher variables. The last step trivially follows again from the definition of ghost samples and the definition of Rademacher variables.

The second summand in \eqref{expr} is bounded using  Khintchine's  inequality (\cite{Khintchine1923})
\begin{equation}
\label{}
\mathbb{E} \left[ \left| \sum_{j=0}^{k-1} \varepsilon_j \right| \right] \le \sqrt{k}
\end{equation}
and hence in \eqref{expr} we obtain
\begin{align*} 
\sum_{l=1}^N \mathbb{E}   \left[\sup_{\substack{(A,C,\bm{\zeta}) \in \Theta}} \left|  \sum_{j=0}^{k-1} \varepsilon_j H_l^{\sigma, A, C, \bm{\zeta}}(\widetilde{{\bf Z}}^{(j)}) \right| \right] &  
\leq \sqrt{k} \ \sum_{l=0}^\infty (\lambda_{max}^A)^l \left(\lambda_{max}^C  {\mathbb{E} \left[ \left\|   {{\bf Z}}_{0} \right\|_2^2   \right]^{1/2} } + \lambda_{max}^{\bm {\zeta}}  \right)\\
&=\dfrac{\sqrt{k} \ }{ 1-\lambda_{max}^A} \left(\lambda_{max}^C   {\mathbb{E} \left[ \left\|   {{\bf Z}}_{0} \right\|_2^2   \right]^{1/2} } + \lambda_{max}^{\bm {\zeta}}  \right),
 \end{align*}
 which finally gives
 \begin{align}
\label{}
\mathbb{E} \left[ \sup_{H \in \mathcal{H}^{RC}} \left\|  \sum_{j=0}^{k-1} \varepsilon_j H(\widetilde{{\bf Z}}^{(j)}) \right\|_2 \right]\leq \sqrt{k} \left(\dfrac{\overline{L_h}}{ 1-\lambda_{max}^A} \left(\lambda_{max}^C   {\mathbb{E} \left[ \left\|   {{\bf Z}}_{0} \right\|_2^2   \right]^{1/2} } + \lambda_{max}^{\bm {\zeta}}  \right) + L_{h,0}\right)
\end{align}
which with the choice of constant in \eqref{CESN} gives \eqref{RadESN} as required.
$\blacksquare$

\subsection{Proof of Proposition~\ref{prop:SASCase} (State-Affine Systems)}

Let $(p,q) \in \Theta$. Then, the conditions on the quantities $\lambda^{SAS} $ and $c^{SAS}$ introduced in \eqref{eq:SAScond} imply that
\begin{equation} \label{eq:pContr}
M _p=\max_{{\bf z} \in \overline{B_{\left\|\cdot \right\|}({\bf 0}, M)}} \left\{\vertiii{ p({\bf z})} _2\right\} < 1
\end{equation}
and so $F^{p,q}$ is a contraction on the first entry. Thus, Proposition~\ref{prop:contractionEchoState} implies that the system \eqref{eq:RCSystemDet} has the echo state property. Moreover, by \citep*[Proposition~14]{RC6}, for any ${\bf z} \in K _M$,  we have that 
\[ 
H^{p,q}({\bf z}) =  \sum_{i=0}^{\infty} p({\bf z}_0) p({\bf z}_{-1}) \cdots p({\bf z}_{-i+1}) q({\bf z}_{-i}).  
\]
Next, write $p({\bf z})= \sum_{\bm{\alpha} \in I_{max}} {\bf z}^{\bm{\alpha}} B_{\bm{\alpha}}$ and $q({\bf z})= \sum_{\bm{\alpha} \in I_{max}} {\bf z}^{\bm{\alpha}} C_{\bm{\alpha}}$. Again,  by the conditions on the quantities $\lambda^{SAS} $ and $c^{SAS}$ introduced in \eqref{eq:SAScond} one has that  the image of $H^{p,q}$ is bounded and so, combining this with \eqref{eq:pContr} one obtains 
\begin{align*}
\big\|  &\sum_{j=0}^{k-1} \varepsilon_j H^{p,q}(\widetilde{{\bf Z}}^{(j)}) \big\|_2 \\
& \leq \sum_{i=0}^\infty \left\|  \sum_{j=0}^{k-1} \varepsilon_j p(\widetilde{{\bf Z}}^{(j)}_0) \cdots p(\widetilde{{\bf Z}}^{(j)}_{-i+1}) q(\widetilde{{\bf Z}}^{(j)}_{-i}) \right\|_2 \\
& \leq  \sum_{i=0}^\infty \sum_{\bm{\alpha} \in I_{max}} \left\| B_{\bm{\alpha}} \sum_{j=0}^{k-1} \varepsilon_j (\widetilde{{\bf Z}}^{(j)}_0)^{\bm{\alpha}} p(\widetilde{{\bf Z}}^{(j)}_{-1}) \cdots p(\widetilde{{\bf Z}}^{(j)}_{-i+1}) q(\widetilde{{\bf Z}}^{(j)}_{-i}) \right\|_2
 \\
& \leq  \sum_{i=0}^\infty \sum_{\bm{\alpha}^1 \in I_{max}} \cdots \sum_{\bm{\alpha}^i \in I_{max}} \vertiii{ B_{\bm{\alpha}^1} }_2 \cdots \vertiii{ B_{\bm{\alpha}^i} }_2 \left\| \sum_{j=0}^{k-1} \varepsilon_j (\widetilde{{\bf Z}}^{(j)}_0)^{\bm{\alpha}^1} \cdots (\widetilde{{\bf Z}}^{(j)}_{-i+1})^{\bm{\alpha}^i} q(\widetilde{{\bf Z}}^{(j)}_{-i}) \right\|_2
\\
& \leq  \sum_{i=0}^\infty \sum_{\bm{\alpha}^1 \in I_{max}} \cdots \sum_{\bm{\alpha}^i \in I_{max}} \vertiii{ B_{\bm{\alpha}^1} }_2 \cdots \vertiii{ B_{\bm{\alpha}^i} }_2 \sum_{\bm{\alpha} \in I_{max}} \| C_{\bm{\alpha}} \|_2 \left| \sum_{j=0}^{k-1} \varepsilon_j (\widetilde{{\bf Z}}^{(j)}_0)^{\bm{\alpha}^1} \cdots (\widetilde{{\bf Z}}^{(j)}_{-i+1})^{\bm{\alpha}^i} (\widetilde{{\bf Z}}^{(j)}_{-i})^{\bm{\alpha}} \right|.
\end{align*}
Therefore, by this expression and using the same type of arguments for linear readout as in the example of echo state networks one obtains
\begin{align*}
\mathbb{E}  & \left[ \sup_{H \in \mathcal{H}^{RC}} \left\|  \sum_{j=0}^{k-1} \varepsilon_j H(\widetilde{{\bf Z}}^j) \right\|_2 \right]
\\ &=\mathbb{E} \left[ \sup_{\substack{F \in \mathcal{F}^{RC}\\ W: \vertiii{W}_2\leq \overline{L_h} \\ \boldsymbol{a}: \|\boldsymbol{a}\|_2\leq L_{h,0}}}  \left\|  \sum_{j=0}^{k-1} \varepsilon_j (WH^F(\widetilde{{\bf Z}}^{(j)}) + \boldsymbol{a}) \right\|_2 \right]\nonumber\\
&   \leq \overline{L_h}\sum_{i=0}^\infty \sum_{\bm{\alpha}^1 \in I_{max}} \cdots \sum_{\bm{\alpha}^i \in I_{max}} (\lambda^{SAS})^i c^{SAS} \sum_{\bm{\alpha} \in I_{max}} \mathbb{E} \left[ \left| \sum_{j=0}^{k-1} \varepsilon_j (\widetilde{{\bf Z}}^{(j)}_0)^{\bm{\alpha}^1} \cdots (\widetilde{{\bf Z}}^{(j)}_{-i+1})^{\bm{\alpha}^i} (\widetilde{{\bf Z}}^{(j)}_{-i})^{\bm{\alpha}} \right| \right] + \sqrt{k} L_{h,0}
\\ 
&   \leq \overline{L_h} \sum_{i=0}^\infty \sum_{\bm{\alpha}^1 \in I_{max}} \cdots \sum_{\bm{\alpha}^i \in I_{max}} (\lambda^{SAS})^i c^{SAS} \sum_{\bm{\alpha} \in I_{max}} \mathbb{E} \left[ \left| \sum_{j=0}^{k-1} \varepsilon_j (\widetilde{{\bf Z}}^{(j)}_0)^{\bm{\alpha}^1} \cdots (\widetilde{{\bf Z}}^{(j)}_{-i+1})^{\bm{\alpha}^i} (\widetilde{{\bf Z}}^{(j)}_{-i})^{\bm{\alpha}} \right|^2 \right]^{1/2} + \sqrt{k} L_{h,0}
\\ 
&   = \sqrt{k}  \overline{L_h}\sum_{i=0}^\infty \sum_{\bm{\alpha}^1 \in I_{max}} \cdots \sum_{\bm{\alpha}^i \in I_{max}} (\lambda^{SAS})^i c^{SAS} \sum_{\bm{\alpha} \in I_{max}} \mathbb{E} \left[ \left| ({\bf Z}_0)^{\bm{\alpha}^1} \cdots ({\bf Z}_{-i+1})^{\bm{\alpha}^i} ({\bf Z}_{-i})^{\bm{\alpha}} \right|^2 \right]^{1/2} + \sqrt{k} L_{h,0}
\\ 
&   \leq \sqrt{k} \left( \overline{L_h} \frac{c^{SAS}|I_{max}|}{1-|I_{max}|\lambda^{SAS}} + L_{h,0} \right). \quad\blacksquare
\end{align*}

\subsection{Proof of Theorem~\ref{thm:main1}}

We start by working out two  concentration inequalities that are needed in the proof. These are contained in the two propositions \ref{prop:concentration} and \ref{prop:concentrationUnbounded} and are used in part {\bf (i)} of the theorem in relation with the use of Assumption~\ref{ass:GLipschitz}.

\subsubsection{(Exponential) concentration inequalities}
\begin{proposition}
\label{prop:concentration} 
Define $\Gamma_n := \sup_{H \in \mathcal{H}^{RC}} \{ R(H) - \widehat{R}_n^\infty(H) \}$. Suppose that Assumptions~\ref{ass:FLipschitz}-\ref{ass:XBounded} hold, that Assumption~\ref{ass:GLipschitz} is satisfied, and that the Bernoulli shifts innovations are bounded, that is, there exists $\overline{M} >0$ such that for $I=y,z$ and for all $t \in \mathbb{Z}_-$, $\|\bm{\xi}^I_t \|_2 \leq \overline{M}$.
 Then there exists $C_{bd}>0$ such that for any $\eta > 0$, $n \in \mathbb{N}^+$ it holds that
\begin{align}
\label{bndConcentration}
\mathbb{P}\left( | \Gamma_n - \mathbb{E}[\Gamma_n] | \geq \eta \right) &  \leq 2 \exp\left(-\frac{2 n \eta^2}{C_{bd}^2}\right). 
\end{align} 
The constant $C_{bd}$ is explicitly given by 
\begin{equation}
\label{eq:Cbd}  
C_{bd} = 2 L_L\left( \dfrac{\overline{L_h } }{1-r}\left(M_\mathcal{F}  r  + L_R \overline{M} L_z \|w^z\|_1 \right) + \overline{M} L_y\|w^y\|_1  \right).
\end{equation}
\end{proposition}

\noindent\textbf{Proof.\ \ }
The main idea of the proof is to exploit the Bernoulli shift structure and apply  McDiarmid's inequality (see for example \cite{Boucheron2013}), which out of the bound of differences of functions constructed in a particular manner yields the bound in \eqref{bndConcentration}.   In order to ease the notation, we first define $\mathcal{Y} := (\overline{B_{\overline{M}}} \times \overline{B_{\overline{M}}}) \subset (\mathbb{R}^{q_y} \times \mathbb{R}^{q_z})$. Consider now a function $\phi \colon \mathcal{Y}^{n-1} \times \mathcal{Y}^{\mathbb{Z}_-} \to \mathbb{R}$, which is  defined for ${\bm{u}}_i=({\bm{u}}_{-i}^y,{\bm{u}}_{-i}^z) \in \mathcal{Y}$, $i=0,\ldots,n-2$ and ${\bm{u}}_{n-1}=({\bm{u}}_{-n+1+t}^y,{\bm{u}}_{-n+1+t}^z)_{t \in \mathbb{Z}_-} \in \mathcal{Y}^{\mathbb{Z}_-}$ by
\begin{equation*}
\phi({\bm{u}}_0,\ldots,{\bm{u}}_{n-2},{\bm{u}}_{n-1}) = 
\sup_{H \in \mathcal{H}^{RC}}  \left\{ R(H) - \frac{1}{n} \sum_{i=0}^{n-1} L(H(G^z({\bm{u}}_{-i}^{z,-\infty})),G^y({\bm{u}}_{-i}^{y, -\infty})) \right\}.
\end{equation*}

Fix $k \in \{0,\ldots,n-1\}$ and let $\widetilde{{\bm{u}}}$ be an identical copy of the sequence  ${\bm{u}}$ so that only the  $k$-th entry in $\widetilde{{\bm{u}}}$ is different from ${\bm{u}}$, that is $\widetilde{{\bm{u}}}_{-i} = {\bm{u}}_{-i}$ for all $i \neq k$. We now estimate the difference of the function $\phi \colon \mathcal{Y}^{n-1} \times \mathcal{Y}^{\mathbb{Z}_-} \to \mathbb{R}$ evaluated at  ${\bm{u}}$ and $\widetilde{{\bm{u}}}$ as follows:
\begin{align} 
\phi & ({\bm{u}}_0,\ldots,{{\bm{u}}}_{n-1})  - \phi(\widetilde{\bm{u}}_0,\ldots,\widetilde{\bm{u}}_{n-1}) \nonumber \\ 
& \leq \sup_{H \in \mathcal{H}^{RC}} \inf_{\widetilde{H} \in \mathcal{H}^{RC}}  \frac{1}{n} \sum_{i=0}^{n-1} \left\{L(\widetilde{H}(G^z(\widetilde{\bm{u}}_{-i}^{z,-\infty})),G^y(\widetilde{\bm{u}}_{-i}^{y, -\infty})) -L(H(G^z({\bm{u}}_{-i}^{z,-\infty})),G^y({\bm{u}}_{-i}^{y, -\infty})) \right\}\nonumber \\ 
	&\qquad-R(\widetilde{H}) + R(H)\nonumber \\ 
& \leq \sup_{H \in \mathcal{H}^{RC}} \frac{1}{n} \sum_{i=0}^{n-1} \left\{L(H(G^z(\widetilde{\bm{u}}_{-i}^{z,-\infty})),G^y(\widetilde{\bm{u}}_{-i}^{y, -\infty})) -L(H(G^z({\bm{u}}_{-i}^{z,-\infty})),G^y({\bm{u}}_{-i}^{y, -\infty})) \right\}\nonumber\\ 
& = \sup_{H \in \mathcal{H}^{RC}} \frac{1}{n} \sum_{i=0}^{k} \left\{L(H(G^z(\widetilde{\bm{u}}_{-i}^{z,-\infty})),G^y(\widetilde{\bm{u}}_{-i}^{y, -\infty})) -L(H(G^z({\bm{u}}_{-i}^{z,-\infty})),G^y({\bm{u}}_{-i}^{y, -\infty})) \right\}\nonumber\\
& \leq \sup_{H \in \mathcal{H}^{RC}} \frac{1}{n} \sum_{i=0}^{k} \left\{ L_L(\|H(G^z(\widetilde{\bm{u}}_{-i}^{z,-\infty})) - H(G^z({\bm{u}}_{-i}^{z,-\infty})) \|_2 + \|G^y(\widetilde{\bm{u}}_{-i}^{y, -\infty})- G^y({\bm{u}}_{-i}^{y, -\infty})\|_2) \right\},
\label{difPhi}
\end{align}
where in the last inequality we used that by assumption the loss function $L$ is $L_L$-Lipschitz. For the first  summand under the supremum in \eqref{difPhi}  we use the bound \eqref{eq:HEstimate} and hence write
\begin{align}
\sum_{i=0}^k &\|H(G^z(\widetilde{\bm{u}}_{-i}^{z,-\infty})) - H(G^z({\bm{u}}_{-i}^{z,-\infty})) \|_2 \nonumber\\
&\leq \sum_{i=0}^{k}  \overline{L_h }  (2 r^{k + 1 -i} M_\mathcal{F} + L_R \sum_{j=0}^{k-i} r^j \|G^z(\widetilde{\bm{u}}_{-j-i}^{z,-\infty}) - G^z({\bm{u}}_{-j-i}^{z,-\infty}) \|_2)\nonumber\\
&\leq  \overline{L_h }  \left(2 M_\mathcal{F} r \dfrac{1-r^{k+1}}{1-r}  + L_R \sum_{i=0}^{k} \sum_{j=i}^{k} r^{j-i} \|G^z(\widetilde{\bm{u}}_{-j}^{z, -\infty}) - G^z({\bm{u}}_{-j}^{z, -\infty})\|_2)\right).\label{Hbound}
\end{align}
Notice that the second summand under the supremum in \eqref{difPhi} and the second summand in \eqref{Hbound} can be bounded using that \eqref{eq:ZFMPProperty} holds by Assumption~\ref{ass:GLipschitz} and that by hypothesis both  $\widetilde{{\bm{u}}}$ and ${{\bm{u}}}$ satisfy $\|\bm{u}^I_t\|_2\leq\overline{M}$ and $\|\widetilde{\bm{u}}^I_t\|_2\leq\overline{M}$ for any $t\in \mathbb{Z}_-$. More specifically, for $I=y,z$ and any $i \in \{0,\ldots,k\}$ one obtains 
\begin{align*}
\|G^I(\widetilde{\bm{u}}_{-i}^{I, -\infty}) - G^I({\bm{u}}_{-i}^{I, -\infty}) \|_2 & \leq  L_I \sum_{l=0}^\infty w^I_{l} \| {\bm{u}}_{-l-i}^I - \widetilde{\bm{u}}_{-l-i}^I \| _2\\
&= L_I w^I_{k-i} \| {\bm{u}}_{-k}^I - \widetilde{\bm{u}}_{-k}^I \| _2\leq  2 L_I w^I_{k-i} \overline{M}, 
\end{align*}
where we used that $\widetilde{{\bm{u}}}_{-l-i} = {\bm{u}}_{-l-i}$, for all $l+i \neq k$, $l\in \mathbb{N}$. Combining this expression with \eqref{Hbound} we estimate \eqref{difPhi} as
\begin{align*} 
\phi & ({\bm{u}}_0,\ldots,{\bm{u}}_{n-1})  - \phi(\widetilde{{\bm{u}}}_0,\ldots,\widetilde{{\bm{u}}}_{n-1}) 
\\ &\leq \sup_{H \in \mathcal{H}^{RC}} \frac{1}{n} \sum_{i=0}^{k} \left\{ L_L(\|H(G^z(\widetilde{\bm{u}}_{-i}^{z,-\infty})) - H(G^z({\bm{u}}_{-i}^{z,-\infty})) \|_2 + \|G^y(\widetilde{\bm{u}}_{-i}^{y, -\infty})- G^y({\bm{u}}_{-i}^{y, -\infty})\|_2) \right\}\\
&\leq  \frac{2}{n} L_L\left(\overline{L_h }  M_\mathcal{F} r \dfrac{1-r^{k+1}}{1-r}  + \overline{L_h } L_R \overline{M} L_z\sum_{i=0}^{k} \sum_{j=i}^{k} r^{j-i}  w^z_{k-j}  + \overline{M} L_y\sum_{i=0}^{k} w^y_{k-i}  \right)\\
& =  \frac{2}{n} L_L\left(\overline{L_h }  M_\mathcal{F} r \dfrac{1-r^{k+1}}{1-r}  + \overline{L_h } L_R \overline{M} L_z\sum_{i=0}^{k} \sum_{j=0}^{k-i} r^{j}  w^z_{k-i - j}  + \overline{M} L_y\sum_{i=0}^{k} w^y_{i}  \right)\nonumber\\
& \leq  \frac{2}{n} L_L\left(\overline{L_h }  M_\mathcal{F}  \dfrac{r}{1-r}  +\overline{L_h } L_R \overline{M} L_z \left(\sum_{i=0}^{\infty} r^i \right) \left(\sum_{j=0}^{\infty} w^z_{j}\right)  + \overline{M} L_y\sum_{i=0}^{\infty} w^y_{i}  \right)\nonumber\\
& =  \frac{2}{n} L_L\left( \dfrac{\overline{L_h } }{1-r}\left(M_\mathcal{F}  r  + L_R \overline{M} L_z \|w^z\|_1 \right) + \overline{M} L_y\|w^y\|_1  \right)
 =  \frac{C_{bd}}{n},
\end{align*}
with the constant $C_{bd}$  as in \eqref{eq:Cbd}. We now use this bound of differences in McDiarmid's inequality and simply notice that $\Gamma_n = \phi  ({\bm{\xi}}_0,\ldots,{\bm{\xi}}_{-n+2},{\bm{\xi}}_{-n+1}^{-\infty})$ in the statement, which immediately yields \eqref{bndConcentration}, as required.
\quad $\blacksquare$

\begin{proposition}
\label{prop:concentrationUnbounded} 
Define $\Gamma_n := \sup_{H \in \mathcal{H}^{RC}} \{ R(H) - \widehat{R}_n^\infty(H) \}$. Suppose that Assumptions~\ref{ass:FLipschitz}-\ref{ass:XBounded} hold and that Assumption~\ref{ass:GLipschitz} is satisfied.
Let $\Phi \colon [0,\infty) \to [0,\infty)$ be a convex and increasing function that satisfies $\Phi(0)=0$. Furthermore, assume  that $\max_{I \in \{y,z\}} \mathbb{E}\left[  \Phi(2C^{mom}_I \| \bm{\xi}_{0}^{I} \|_2)   \right]< \infty $ where 
\begin{align}
\label{eq:CmomzDef} 
C^{mom}_z & = L_L \frac{\overline{L_h }  L_R}{1-r} L_z \|w^z\|_1,\\
\label{eq:CmomyDef} 
C^{mom}_y & = L_L L_y\|w^y\|_1,   
\end{align}
and denote
\begin{equation} 
\label{eq:phiDef} 
\varphi(\overline{M}) := \sum^{}_{I \in \{y,z\}}\mathbb{E}\left[  \Phi(2C^{mom}_I \| \bm{\xi}_{0}^{I} \|_2) \bm{1}_{\{\|\bm{\xi}_{0}^{I}\|_2> \overline{M} \}}  \right].
\end{equation}
Then, there exists a constant $C_0> 0$ such that for any $\eta > 0$, $n \in \mathbb{N}^+$, $\overline{M} > 0$ satisfying
\begin{equation}
\label{eq:etacond} 
\sum^{}_{I \in \{y,z\}} C^{mom}_I \mathbb{E}[  \| \bm{\xi}_{0}^I \|_2 \bm{1}_{\{\|\bm{\xi}_{0}^{I}\|_2> \overline{M} \}}]   < \frac{\eta}{5} 
\end{equation} one has
\[\begin{aligned}
\mathbb{P}\left( | \Gamma_n - \mathbb{E}[\Gamma_n] | \geq \eta \right) &  \leq 2 \exp\left(\frac{-2 n \eta^2}{25(C_0+2\overline{M} (C^{mom}_z + C^{mom}_y))^2}\right) + \frac{1}{2}\frac{\varphi(\overline{M})}{\Phi(\eta/5)}. \end{aligned}\]  
The constant $C_0$ is explicitly given by \eqref{eq:C0def}. 
\end{proposition}

\noindent\textbf{Proof.\ \ }
Let $\overline{M} >0$ and for any $t \in \mathbb{Z}_-$, $I={y,z}$ denote by   $\bm{\xi}_t^{I,\overline{M}}$ the Bernoulli shift innovations whose Euclidean norm is bounded above by $\overline{M}$, that is, $\bm{\xi}_t^{I,\overline{M}}:=\bm{\xi}_t^{I} \bm{1}_{\{\|\bm{\xi}_t^{I}\|_2\leq \overline{M} \}} $. In order to simplify the notation, we define 
\begin{align} 
\label{eq:ZRVTruncated} 
{\bf Z}_t^{\overline{M}} & := G^z(\ldots,\bm{\xi}_{t-1}^{z,\overline{M}},\bm{\xi}_{t}^{z,\overline{M}})=G^z((\bm{\xi}^{z,\overline{M}})_{t}^{-\infty}), \quad t \in \mathbb{Z}_-, \\
\label{eq:YRVTruncated} 
{\bf Y}_t^{\overline{M}} & := G^y(\ldots,\bm{\xi}_{t-1}^{y,\overline{M}},\bm{\xi}_{t}^{y,\overline{M}})=G^y((\bm{\xi}^{y,\overline{M}})_{t}^{-\infty}), \quad t \in \mathbb{Z}_-,\\
\label{eq:empiricalRiskFullHistoryDefTruncated}
\widehat{R}_n^{\infty,\overline{M}}(H) & := \frac{1}{n} \sum_{i=0}^{n-1} L(H(({\bf Z}^{\overline{M}})_{-i}^{-\infty}),{\bf Y}_{-i}^{\overline{M}}), \\
\label{eq:expRiskM}
R^{\overline{M}}(H) & := \mathbb{E}[L(H({\bf Z}^{\overline{M}}),{\bf Y}_0^{\overline{M}})]
\end{align}
and, additionally, denote $\Gamma_n^{\overline{M}} := \sup_{H \in \mathcal{H}^{RC}} \{ R^{\overline{M}}(H) - \widehat{R}_n^{\infty, {\overline{M}}}(H) \}$. 
Firstly, the triangle inequality yields
\begin{align} 
\label{eq:auxEq14}
|\Gamma_n - \mathbb{E}[\Gamma_n]| & = | \Gamma_n - \Gamma_n^{\overline{M}}  - (\mathbb{E}[\Gamma_n]-  \mathbb{E}[\Gamma_n^{\overline{M}}]) +  \Gamma_n^{\overline{M}}- \mathbb{E}[\Gamma_n^{\overline{M}}]| \nonumber\\
&\leq | \Gamma_n - \Gamma_n^{\overline{M}}|  + |\mathbb{E}[\Gamma_n-  \Gamma_n^{\overline{M}}]| + | \Gamma_n^{\overline{M}}- \mathbb{E}[\Gamma_n^{\overline{M}}]|\nonumber \\
&\leq | \Gamma_n - \Gamma_n^{\overline{M}}|  + \mathbb{E}[|\Gamma_n-  \Gamma_n^{\overline{M}}|] + | \Gamma_n^{\overline{M}}- \mathbb{E}[\Gamma_n^{\overline{M}}]|.
\end{align}
For the first summand in expression \eqref{eq:auxEq14} we write
\begin{align} 
\label{eq:auxEq15}
\left| \Gamma_n - \Gamma_n^{\overline{M}}\right|  
&\leq \left| \sup_{H \in \mathcal{H}^{RC}} \left\{ R(H) - \widehat{R}_n^\infty(H)\right\} - \sup_{H \in \mathcal{H}^{RC}} \left\{ R^{\overline{M}}(H) - \widehat{R}_n^{\infty,\overline{M}}(H)\right\}\right|   \nonumber\\ 
&= \left| \sup_{H \in \mathcal{H}^{RC}} \inf_{\widetilde{H} \in \mathcal{H}^{RC}}\left(\left\{ R(H) - \widehat{R}_n^\infty(H)\right\} -  \left\{ R^{\overline{M}}(\widetilde{H}) - \widehat{R}_n^{\infty,\overline{M}}(\widetilde{H})\right\}   \right)\right|\nonumber\\
&\leq \left|\sup_{H \in \mathcal{H}^{RC}} \left\{ R(H) - \widehat{R}_n^\infty(H) -   R^{\overline{M}}({H}) + \widehat{R}_n^{\infty,\overline{M}}({H})\right\}\right|   \nonumber\\
&\leq \sup_{H \in \mathcal{H}^{RC}} \left|R^{\overline{M}}(H)-R(H)\right|  +  \sup_{H \in \mathcal{H}^{RC}} \left| \widehat{R}_n^{\infty,\overline{M}}(H)-\widehat{R}_n^{\infty}(H) \right|. 
\end{align}
Using this result, we can immediately get  the following bound for the second summand in expression \eqref{eq:auxEq14}
\begin{align} 
\label{eq:ExpauxEq15}
\mathbb{E}\left[\left| \Gamma_n - \Gamma_n^{\overline{M}}\right| \right]
&\leq \sup_{H \in \mathcal{H}^{RC}} \left|R^{\overline{M}}(H)-R(H)\right|  +  \mathbb{E} \left[\sup_{H \in \mathcal{H}^{RC}} \left| \widehat{R}_n^{\infty,\overline{M}}(H)-\widehat{R}_n^{\infty}(H) \right| \right]. 
\end{align}
The first two terms in the right hand side of \eqref{eq:auxEq14} are thus controlled by \eqref{eq:auxEq15} and \eqref{eq:ExpauxEq15}.
The third term in \eqref{eq:auxEq14} is of the type required in Proposition~\ref{prop:concentration}, that is, the Bernoulli shifts are defined using bounded innovations and hence the term will be controlled in what follows using the result in  Proposition~\ref{prop:concentration}. 

The next step in our proof is to derive bounds for the terms in the right hand sides of  \eqref{eq:auxEq15} and \eqref{eq:ExpauxEq15}. First, we consider the estimate for the term they share, for which we write 
\begin{align}
\label{eq:auxEq16} 
\sup_{H \in \mathcal{H}^{RC}} &|R^{\overline{M}}(H)-R(H)|\nonumber\\
& = \sup_{H \in \mathcal{H}^{RC}} |\mathbb{E}[L(H({\bf Z}^{\overline{M}}),{\bf Y}_0^{\overline{M}})-L(H({\bf Z}),{\bf Y}_0)]| \nonumber\\
&\leq \sup_{H \in \mathcal{H}^{RC}}  \mathbb{E}\left[L_L(\|H({\bf Z}^{\overline{M}})-H({\bf Z})\|_2 + \|{\bf Y}_{0}^{\overline{M}} - {\bf Y}_{0}\|_2 )\right]  \nonumber \\
&\leq L_L (\overline{L_h }   L_R \sum_{j=0}^{\infty} r^j \mathbb{E}[\|{\bf Z}^{\overline{M}}_{-j} - {\bf Z}_{-j} \|_2] + \mathbb{E}[\|{\bf Y}_{0}^{\overline{M}} - {\bf Y}_{0}\|_2])\nonumber\\
& = L_L \left(\frac{\overline{L_h }   L_R}{1-r}  \mathbb{E}[\|G^z(\bm{\xi}^{z, \overline{M}}) - G^z(\bm{\xi}^{z}) \|_2] + \mathbb{E}[\|G^y(\bm{\xi}^y)- G^y(\bm{\xi}^{y,\overline{M}})\|_2]\right),
 \end{align}
where the first and the second (in)equality are obtained using the definition \eqref{eq:empiricalRiskFullHistoryDefTruncated} and the assumption that the loss function $L$ is $L_L$-Lipschitz, the third step follows from the estimate \eqref{eq:HEstimate} from Lemma \ref{lem:HFLipschitz} and the last step again uses \eqref{eq:ZRVTruncated}-\eqref{eq:YRVTruncated} and the i.i.d.\ assumption on $\bm{\xi}$. In order to proceed, notice that since Assumption~\ref{ass:GLipschitz} holds by hypothesis, by \eqref{eq:ZFMPProperty} one has  for any $j \in \mathbb{N}$, $I=y,z$  the following estimate
\begin{align} 
\label{eq:auxEq13} 
\|G^I((\bm{\xi}^{I})_{-j}^{-\infty}) - G^I((\bm{\xi}^{I,\overline{M}})_{-j}^{-\infty}) \|_2  
&\leq L_I \sum_{l=0}^\infty w^I_{l} \| \bm{\xi}_{-l-j}^I - \bm{\xi}_{-l-j}^{I,\overline{M}} \|_2 \nonumber\\
&= L_I \sum_{l=0}^\infty w^I_{l} \| \bm{\xi}_{-l-j}^I \|_2 \bm{1}_{\{\|\bm{\xi}_{-l-j}^{I}\|_2> \overline{M} \}}. 
\end{align}
Combining \eqref{eq:auxEq16} and \eqref{eq:auxEq13} and using the i.i.d.\ assumption on $\bm{\xi}$ one obtains 
\begin{align}
\sup_{H \in \mathcal{H}^{RC}} & |R^{\overline{M}}(H)-R(H)| \nonumber \\
& \leq L_L \left(\frac{\overline{L_h }   L_R}{1-r} L_z \mathbb{E}[ \sum_{l=0}^\infty w^z_{l} \| \bm{\xi}_{-l}^z \|_2 \bm{1}_{\{\|\bm{\xi}_{-l}^{z}\|_2> \overline{M} \}}] + L_y\mathbb{E}[ \sum_{l=0}^\infty w^y_{l} \| \bm{\xi}_{-l}^y \|_2 \bm{1}_{\{\|\bm{\xi}_{-l}^{y}\|_2> \overline{M} \}}]\right)\nonumber \\
&\leq L_L \left(\frac{\overline{L_h }   L_R}{1-r} L_z \|w^z\|_1\mathbb{E}[  \| \bm{\xi}_{0}^z \|_2 \bm{1}_{\{\|\bm{\xi}_{0}^{z}\|_2> \overline{M} \}}] + L_y\|w^y\|_1\mathbb{E}[  \| \bm{\xi}_{0}^y \|_2 \bm{1}_{\{\|\bm{\xi}_{0}^{y}\|_2> \overline{M} \}}]\right)\nonumber\\
&\leq C^{mom}_z \mathbb{E}[  \| \bm{\xi}_{0}^z \|_2 \bm{1}_{\{\|\bm{\xi}_{0}^{z}\|_2> \overline{M} \}}] + C^{mom}_y \mathbb{E}[  \| \bm{\xi}_{0}^y \|_2 \bm{1}_{\{\|\bm{\xi}_{0}^{y}\|_2> \overline{M} \}}]\nonumber\\
&=\sum^{}_{I \in \{y,z\}} C^{mom}_I \mathbb{E}[  \| \bm{\xi}_{0}^I \|_2 \bm{1}_{\{\|\bm{\xi}_{0}^{I}\|_2> \overline{M} \}}] 
\label{eq:auxEq17} 
 \end{align}
 with $C^{mom}_I$ for $I=y,z$ defined as in \eqref{eq:CmomzDef} and \eqref{eq:CmomyDef}.

Next, we analyze the second term in \eqref{eq:auxEq15} and using the function $\Phi: [0, \infty) \longrightarrow  [0, \infty)$, which by hypothesis satisfies $\Phi(0) = 0$, we obtain for $\eta > 0$
\begin{align}
\label{secondTerm}
\Phi&(\eta)\mathbb{P}  \left(\sup_{H \in \mathcal{H}^{RC}} \left| \widehat{R}_n^{\infty,\overline{M}}(H)-\widehat{R}_n^{\infty}(H) \right| > \eta \right) \nonumber\\
&\leq \mathbb{E}\left[\Phi\left(\sup_{H \in \mathcal{H}^{RC}} \left| \widehat{R}_n^{\infty,\overline{M}}(H)-\widehat{R}_n^{\infty}(H) \right|\right)\right] \nonumber\\
& \leq \mathbb{E}\Big[\Phi\Big( \frac{L_L}{n}\sum_{i=0}^{n-1} \Big(\overline{L_h }   L_R \sum_{j=0}^{\infty} r^j \|G^z(\bm{\xi}^{z,-\infty}_{-j-i}) - G^z((\bm{\xi}^{z,\overline{M}})^{-\infty}_{-j-i}) \|_2 \nonumber\\
&\quad\quad+ \|G^y(\bm{\xi}^{y,-\infty}_{-i})- G^y((\bm{\xi}^{y,\overline{M}})^{-\infty}_{-i}) \|_2 \Big)\Big)\Big] 
\nonumber \\
& \leq \mathbb{E}\Big[\Phi\Big( \frac{L_L}{n}\sum_{i=0}^{n-1} \Big(\frac{\overline{L_h }   L_R}{1-r} \|G^z(\bm{\xi}^{z,-\infty}_{0}) - G^z((\bm{\xi}^{z,\overline{M}})^{-\infty}_{0}) \|_2 \nonumber \\
&\quad\quad+ \|G^y(\bm{\xi}^{y,-\infty}_{0})- G^y((\bm{\xi}^{y,\overline{M}})^{-\infty}_{0}) \|_2 \Big)\Big)\Big] \nonumber
\\
& \leq \mathbb{E}\left[\Phi\left(L_L\left( \frac{ \overline{L_h }  L_R}{1-r}  L_z \sum_{l=0}^\infty w^z_{l} \| \bm{\xi}_{-l}^z \|_2 \bm{1}_{\{\|\bm{\xi}_{-l}^{z}\|_2> \overline{M} \}} + L_y\sum_{l=0}^\infty w^y_{l} \| \bm{\xi}_{-l}^y \|_2 \bm{1}_{\{\|\bm{\xi}_{-l}^{y}\|_2> \overline{M} \}}  \right)\right)\right] \nonumber 
\\
& \leq  \mathbb{E}\left[\Phi\left( L_L\left(\frac{\overline{L_h }  L_R}{1-r} L_z \|w^z\|_1 \| \bm{\xi}_{0}^z \|_2 \bm{1}_{\{\|\bm{\xi}_{0}^{z}\|_2> \overline{M} \}} + L_y \|w^y\|_1 \| \bm{\xi}_{0}^y \|_2 \bm{1}_{\{\|\bm{\xi}_{0}^{y}\|_2> \overline{M} \}} \right) \right)\right]\nonumber \\
& \leq  \dfrac{1}{2} \mathbb{E}\left[\Phi \left( 2 L_L \frac{\overline{L_h }   L_R}{1-r} L_z \|w^z\|_1 \| \bm{\xi}_{0}^z \|_2 \bm{1}_{\{\|\bm{\xi}_{0}^{z}\|_2> \overline{M} \}} \right) \right] + \dfrac{1}{2} \mathbb{E}\left[\Phi\left(2 L_L L_y \|w^y\|_1 \| \bm{\xi}_{0}^y \|_2 \bm{1}_{\{\|\bm{\xi}_{0}^{y}\|_2> \overline{M} \}} \right)\right]\nonumber \\
& \leq  \dfrac{1}{2} \mathbb{E}\left[\Phi \left( 2C^{mom}_z \| \bm{\xi}_{0}^z \|_2  \right) \bm{1}_{\{\|\bm{\xi}_{0}^{z}\|_2> \overline{M} \}} \right] + \dfrac{1}{2} \mathbb{E}\left[\Phi\left(2C^{mom}_y \|\bm{\xi}_{0}^y \|_2 \right)\bm{1}_{\{\|\bm{\xi}_{0}^{y}\|_2> \overline{M} \}} \right]\nonumber \\
&= \dfrac{1}{2} \varphi(\overline{M}),
\end{align}
where $C^{mom}_z$, $C^{mom}_y$, and $\varphi(\overline{M})$ are given in \eqref{eq:CmomzDef},  \eqref{eq:CmomyDef}, and \eqref{eq:phiDef}, respectively.
In these derivations we used Markov's inequality for increasing and non-negative functions in the first step. The second inequality uses the definition in \eqref{eq:empiricalRiskFullHistoryDefTruncated}, the estimate \eqref{eq:HEstimate} from Lemma~\ref{lem:HFLipschitz}, and the fact that $\Phi$ is by hypothesis an increasing function. We subsequently used the stationarity assumption and the convexity of the function $\Phi$. Finally,  \eqref{eq:auxEq13} with the appropriate choice of constants yields the result. 

We now notice that this result provides automatically a bound for the second term in \eqref{eq:auxEq16}. In order to see that, one needs to take as the function $\Phi$ the identity and then 
by the first and the last two lines in \eqref{secondTerm} it holds that
\begin{equation}
\label{ExpectSecondTerm}
\mathbb{E}\left[\sup_{H \in \mathcal{H}^{RC}} \left| \widehat{R}_n^{\infty,\overline{M}}(H)-\widehat{R}_n^{\infty}(H) \right|\right] \leq \sum^{}_{I \in \{y,z\}} C^{mom}_I \mathbb{E}[  \| \bm{\xi}_{0}^I \|_2 \bm{1}_{\{\|\bm{\xi}_{0}^{I}\|_2> \overline{M} \}}]. 
\end{equation}
We now consider again  expression \eqref{eq:auxEq14} and taking into account \eqref{eq:auxEq15}, \eqref{eq:ExpauxEq15}, together with the bounds for their ingredients given in \eqref{eq:auxEq17}, \eqref{secondTerm}, and \eqref{ExpectSecondTerm}  we derive, for any $\eta > 0 $,
\begin{align*}
\mathbb{P}& \left( | \Gamma_n - \mathbb{E}[\Gamma_n] | \geq \eta \right) \\
&\leq \mathbb{P}\left(\sup_{H \in \mathcal{H}^{RC}} | \widehat{R}_n^{\infty,\overline{M}}(H)-R^{\overline{M}}(H) | \geq \frac{\eta}{5} \right) + \mathbb{P}\left( \sup_{H \in \mathcal{H}^{RC}} \left| \widehat{R}_n^{\infty,\overline{M}}(H)-\widehat{R}_n^{\infty}(H) \right|\geq \frac{\eta}{5} \right)\\
&\quad\quad +\mathbb{P}\left( 3 C_{mom} \mathbb{E}\left[ \| \bm{\xi}_{0}^{max} \|_2  \bm{1}_{\{\| \bm{\xi}_{0}^{max} \|_2> \overline{M} \}}\right] \geq \frac{3\eta}{5} \right)\\
&\leq 2 \exp\left(\frac{-2 n \eta^2}{25C_{bd}^2}\right) + \frac{1}{2}\frac{\varphi(\overline{M})}{\Phi(\eta/5)},
\end{align*}
with $C_{bd}$ as in \eqref{eq:Cbd}.
In the second inequality we used the hypothesis in \eqref{eq:etacond} and \eqref{bndConcentration} in Proposition~\ref{prop:concentration}, as well as the bound \eqref{secondTerm}. Finally, noticing that  $C_{bd} = 2\overline{M}(C^{mom}_z + C^{mom}_y) + C_0$ with $C_0$ as in \eqref{eq:C0def} and setting $C^{mom}_z$ and $C^{mom}_y$ as in {\eqref{eq:CmomzDef}}-{\eqref{eq:CmomyDef}} immediately yields the claim.
\quad $\blacksquare$

\begin{corollary} \label{cor:concentrationUnbounded} Suppose that Assumptions~\ref{ass:FLipschitz}-\ref{ass:XBounded} hold and that Assumption~\ref{ass:GLipschitz} is satisfied.
Let $\Phi \colon [0,\infty) \to [0,\infty)$ be a convex and strictly increasing function that satisfies $\Phi(0)=0$. Furthermore, assume that for all $u > 0 $, $\mathbb{E}[ \Phi(u \|\bm{\xi}_{0}^I\|)^2] < \infty $ for $I=y,z$.
Then, for any $\delta \in (0,1)$, $n \in \mathbb{N}^+$, 
\[\begin{aligned}
\mathbb{P}\left( | \Gamma_n - \mathbb{E}[\Gamma_n] | \geq B_\Phi(n,\delta) \right) &  \leq \frac{\delta}{2}, \end{aligned}\]  
where
\begin{align}
\label{eq:BPhi} 
B_\Phi(n,\delta) & = 5 \max\left( \frac{(C_0+2 \Phi^{-1}({n}) (C^{mom}_z + C^{mom}_y)) \sqrt{\log(\frac{8}{\delta})}}{\sqrt{2n}}  , \Phi^{-1}\left( \frac{2 C_\Phi}{\delta \sqrt{n}} \right) \right), \\
\label{eq:CPhi} 
C_\Phi & = \sum^{}_{I \in \{y,z\}} \mathbb{E}\left[  \Phi(2C^{mom}_I \| \bm{\xi}_{0}^{I} \|_2)^2\right]^{1/2}  \mathbb{E}[\Phi( \|\bm{\xi}_0^I \|_2)]^{1/2},
\end{align}
and $C_0$, $C^{mom}_z$, $C^{mom}_y$ are given by \eqref{eq:C0def}, \eqref{eq:CmomzDef}, and \eqref{eq:CmomyDef}. 
\end{corollary}

\noindent\textbf{Proof.\ \ }
We start with the function $\varphi(\overline{M})$ given in \eqref{eq:phiDef} and  obtain that for any $\overline{M} > 0$ it holds that
\begin{align} 
\label{eq:phiDefHoelder} 
\varphi(\overline{M}) &= \sum^{}_{I \in \{y,z\}}\mathbb{E}\left[  \Phi(2C^{mom}_I \| \bm{\xi}_{0}^{I} \|_2) \bm{1}_{\{\|\bm{\xi}_{0}^{I}\|_2> \overline{M} \}}  \right] \nonumber\\
&\leq  \sum^{}_{I \in \{y,z\}}\mathbb{E}\left[  \Phi(2C^{mom}_I \| \bm{\xi}_{0}^{I} \|_2)^2\right]^{1/2} \mathbb{E}\left[ ( \bm{1}_{\{\|\bm{\xi}_{0}^{I}\|_2> \overline{M} \}})^2  \right]^{1/2} \nonumber\\
&= \sum^{}_{I \in \{y,z\}} \mathbb{E}\left[  \Phi(2C^{mom}_I \| \bm{\xi}_{0}^{I} \|_2)^2\right]^{1/2}  \mathbb{P}(\|\bm{\xi}_0^I \|_2>\overline{M})^{1/2} \nonumber\\ 
& \leq  \frac{C_\Phi}{\Phi(\overline{M})^{1/2}},
\end{align}
where the first inequality is a consequence of H\"older's inequality, and the last step is obtained by applying Markov's inequality for increasing nonnegative functions and by using the definition in \eqref{eq:CPhi}. 
Furthermore, by Jensen's inequality and since $\Phi(0)=0$, for $I =y,z$ one has that
\begin{align} 
\label{eq:auxEq28} 
\Phi(\sum^{}_{I \in \{y,z\}} C^{mom}_I\mathbb{E}[\|\bm{\xi}_0^I \|_2 \bm{1}_{\{\|\bm{\xi}_{0}^{I}\|_2> \overline{M} \}} ]) &\leq \sum^{}_{I \in \{y,z\}} \dfrac{1}{2}\Phi \left(2C^{mom}_I \mathbb{E}[  \| \bm{\xi}_{0}^I \|_2 \bm{1}_{\{\|\bm{\xi}_{0}^{I}\|_2> \overline{M} \}}]  \right)\nonumber \\
&
\leq\dfrac{1}{2} \sum^{}_{I \in \{y,z\}} \mathbb{E} \left[   \Phi \left(2C^{mom}_I  \| \bm{\xi}_{0}^I \|_2 \bm{1}_{\{\|\bm{\xi}_{0}^{I}\|_2> \overline{M} \}}\right) \right]  \nonumber \\
&=  \dfrac{1}{2} \varphi(\overline{M}) .  
\end{align}
Choosing $\overline{M} = \Phi^{-1}({n})$ in \eqref{eq:phiDefHoelder} and setting $\eta = B_\Phi(n,\delta)$ defined in \eqref{eq:BPhi} one easily verifies that 
\begin{equation} 
\label{eq:auxEq29} 
\dfrac{1}{2}\frac{\varphi(\overline{M})}{\Phi(\eta/5)} \leq  \dfrac{1}{2} \frac{C_\Phi}{\sqrt{n} \Phi(\eta/5)} \leq \frac{\delta}{4}. 
 \end{equation}
In particular, this also implies that for $\delta\in (0,1)$ one has $\varphi(\overline{M}) < \Phi(\eta/5)$ and so \eqref{eq:auxEq28} yields
\begin{equation*} 
\Phi(\sum^{}_{I \in \{y,z\}} C^{mom}_I\mathbb{E}[\|\bm{\xi}_0^I \|_2 \bm{1}_{\{\|\bm{\xi}_{0}^{I}\|_2> \overline{M} \}} ]) < \Phi(\eta/5)
\end{equation*}
and hence
\[\sum^{}_{I \in \{y,z\}} C^{mom}_I\mathbb{E}[\|\bm{\xi}_0^I \|_2 \bm{1}_{\{\|\bm{\xi}_{0}^{I}\|_2> \overline{M} \}} ] < \frac{\eta}{5}.\]
Thus \eqref{eq:etacond} is satisfied and we may apply Proposition~\ref{prop:concentrationUnbounded} and \eqref{eq:auxEq29}, which yields
\[\begin{aligned}
\mathbb{P}\left( | \Gamma_n - \mathbb{E}[\Gamma_n] | \geq \eta \right) &  \leq 2 \exp\left(\frac{-2 n \eta^2}{25(C_0+ \Phi^{-1}({n}) (C^{mom}_z + C^{mom}_y))^2}\right) + \frac{\delta}{4}  \leq \frac{\delta}{2}. \quad \blacksquare
\end{aligned}\]

\subsubsection{Proof of Theorem~\ref{thm:main1}}
\label{Proof of Theorem:main1}
{\bf Proof of  part (i)}. 
In this situation, the hypotheses of part {\bf (i)} of Corollary~\ref{cor:ThetaWeighting} are satisfied and so the following bound holds: 
\begin{equation}
\label{eq:auxEq18} 
\mathbb{E}\left[ \sup_{H \in \mathcal{H}^{RC}} \left\{R(H) - \widehat{R}_n^\infty(H) \right\}\right]
 \leq \frac{C_1}{n} + \frac{C_2 {\log(n)}}{{n}} + \frac{C_3 \sqrt{\log(n)}}{\sqrt{n}}. 
 \end{equation}
Let us denote $\Gamma_n := \sup_{H \in \mathcal{H}^{RC}} \{ R(H) - \widehat{R}_n^\infty(H) \}$. We may then apply the triangle inequality, insert the estimate on the difference between the empirical risk and its idealized counterpart obtained in Proposition~\ref{prop:finiteHistoryError} as well as the estimate on the expected value \eqref{eq:auxEq18} to obtain that $\mathbb{P}$-a.s.,
\begin{equation} 
\label{eq:mainTheoremprelim} 
\begin{aligned}  
\sup_{H \in \mathcal{H}^{RC}} \lbrace R(H) - \widehat{R}_n(H) \rbrace & = \sup_{H \in \mathcal{H}^{RC}} \lbrace R(H) - \widehat{R}_n(H) + \widehat{R}_n^\infty(H) - \widehat{R}_n^\infty(H) \rbrace - \mathbb{E}[\Gamma_n] +  \mathbb{E}[\Gamma_n]\\
&\leq \sup_{H \in \mathcal{H}^{RC}} \lbrace \widehat{R}_n^\infty(H) - \widehat{R}_n(H) \rbrace  + \Gamma_n   - \mathbb{E}[\Gamma_n]  +  \mathbb{E}[\Gamma_n] \\
&\leq \sup_{H \in \mathcal{H}^{RC}} | \widehat{R}_n^\infty(H) - \widehat{R}_n(H) |  + | \Gamma_n - \mathbb{E}[\Gamma_n] | +  \mathbb{E}[\Gamma_n] \\
& \leq \frac{(1-r^n)C_0}{n} + | \Gamma_n - \mathbb{E}[\Gamma_n] | +  \frac{C_1}{n} + \frac{C_2 {\log(n)}}{{n}} + \frac{C_3 \sqrt{\log(n)}}{\sqrt{n}}. \end{aligned} \end{equation}

\noindent{\bf{Part (a)}}: Denote by $\eta$ the upper bound that we need to prove holds with high probability, that is, 
\[ \eta := \frac{(1-r^n)C_0 + C_1}{n}  + \frac{C_2 {\log(n)}}{{n}} + \frac{C_3 \sqrt{\log(n)}}{\sqrt{n}} + \frac{C_{bd} \sqrt{\log(\frac{4}{\delta})}}{\sqrt{2 n}}. \] 
Combining the estimate \eqref{eq:mainTheoremprelim} with the exponential concentration inequality Proposition~\ref{prop:concentration} then
yields 
\begin{equation}
\mathbb{P}\left( \sup_{H \in \mathcal{H}^{RC}} \lbrace R(H)- \widehat{R}_n(H) \rbrace > \eta \right)  \leq
\mathbb{P}\left( | \Gamma_n - \mathbb{E}[\Gamma_n] |  > \frac{C_{bd} \sqrt{\log(\frac{4}{\delta})}}{\sqrt{2 n}} \right)  \leq \frac{\delta}{2}. 
\end{equation}
By applying the result that we just proved to the loss function $-L$ one obtains that 
\[ \mathbb{P}\left( \sup_{H \in \mathcal{H}^{RC}} \lbrace \widehat{R}_n(H) - R(H)\rbrace > \eta \right) \leq \frac{\delta}{2}. \]
Using that $|x|=\max(x,-x)$ one can thus combine the two estimates to deduce
\begin{align*} 
& \mathbb{P}\left( \sup_{H \in \mathcal{H}^{RC}} \left|R(H)-\widehat{R}_n(H) \right| > \eta \right) \\ 
& \quad \quad \leq \mathbb{P}\left(\left\lbrace \sup_{H \in \mathcal{H}^{RC}} \lbrace R(H) - \widehat{R}_n(H) \rbrace > \eta \right\rbrace \cup \left\lbrace  \sup_{H \in \mathcal{H}^{RC}} \lbrace \widehat{R}_n(H)-R(H) \rbrace > \eta \right\rbrace \right)\leq \delta.
\end{align*}
\noindent{\bf{Part (b)}}: 
   Proceeding analogously as in part {\bf (a)}, denote by $\eta$ the high-probability upper bound which needs to be established, that is, 
\[ \eta = \frac{(1-r^n)C_0 + C_1}{n}  + \frac{C_2 {\log(n)}}{{n}} + \frac{C_{3} \sqrt{\log(n)}}{\sqrt{n}} + B_\Phi(n,\delta).  \] 
Combining \eqref{eq:mainTheoremprelim}  with Corollary~\ref{cor:concentrationUnbounded} then
yields 
\[\begin{aligned}
\mathbb{P}\left( \sup_{H \in \mathcal{H}^{RC}} \lbrace R(H) - \widehat{R}_n(H)  \rbrace > \eta \right) & \leq
\mathbb{P}\left( | \Gamma_n - \mathbb{E}[\Gamma_n] |  > B_\Phi(n,\delta) \right) \leq \frac{\delta}{2}. \end{aligned}\] 
The claim then follows precisely as in the proof of part {\bf (a)}.\\

\medskip

\noindent{\bf Proof of  part (ii).} 
Firstly, one may use Proposition~\ref{prop:finiteHistoryError} to obtain
 $\mathbb{P}$-a.s.,
\begin{equation} \label{eq:auxEq24} \begin{aligned}  \sup_{H \in \mathcal{H}^{RC}} \lbrace R(H)-\widehat{R}_n(H)  \rbrace & \leq \sup_{H \in \mathcal{H}^{RC}} \lbrace  \widehat{R}_n^\infty(H) - \widehat{R}_n(H) \rbrace  + |\Gamma_n | \leq 
\frac{(1-r^n)C_0}{n} + |\Gamma_n |. \end{aligned} \end{equation}
Setting 
\begin{equation} \label{eq:auxEq30} \eta = \frac{2}{\delta} \left( \frac{C_1}{n} + \frac{C_2 {\log(n)}}{{n}} + \frac{C_{3,abs} \sqrt{\log(n)}}{\sqrt{n}}\right)  + \frac{(1-r^n)C_0}{n}\end{equation}
and applying Markov's inequality, \eqref{eq:auxEq24} and  part {\bf (ii)} of Corollary~\ref{cor:ThetaWeighting} then yields 
\[\begin{aligned}
\mathbb{P}\left( \sup_{H \in \mathcal{H}^{RC}} \lbrace R(H) -  \widehat{R}_n(H)  \rbrace > \eta \right) & \leq
\mathbb{P}\left( | \Gamma_n |  > \frac{2}{\delta} \left( \frac{C_1}{n} + \frac{C_2 {\log(n)}}{{n}} + \frac{C_{3,abs} \sqrt{\log(n)}}{\sqrt{n}}\right) \right) \\
&\leq {\mathbb{E}[| \Gamma_n |]} \frac{\delta}{2} \left( \frac{C_1}{n} + \frac{C_2 {\log(n)}}{{n}} + \frac{C_{3,abs} \sqrt{\log(n)}}{\sqrt{n}}\right)^{-1}\\
&   \leq \frac{\delta}{2}. \end{aligned}\] 
By applying what we just proved to the loss function $-L$ the claim then follows precisely as in the proof of part {\bf (i)}.\\

\medskip

\noindent{\bf Proof of  part (iii).}
The proof is the same as the proof of part {\bf (ii)}, except that instead of choosing
$\eta$ as in \eqref{eq:auxEq30} one takes 
\[ \eta =  \frac{(1-r^n)C_0}{n} + \frac{2}{\delta} \left({C}_{1,abs} n^{-\frac{1}{2+\alpha^{-1}}}  + {C_2}{n^{-\frac{2}{2+\alpha^{-1}}}} \right) \]
and instead of using part {\bf (ii)} of Corollary~\ref{cor:ThetaWeighting}  one applies its part {\bf (iii)} to estimate $\mathbb{E}[| \Gamma_n |]$. $\blacksquare$

\subsection{Proof of Proposition~\ref{prop:ESNrandomparams}} 

Denote by $\Theta := \left\{( \rho_A \mathbf{A} ,\rho_C \mathbf{C},\rho_{\bm{\zeta}} \boldsymbol{\zeta} ) \mid  (\rho_A,\rho_C,\rho_{\bm{\zeta}}) \in (-\frac{a}{\lambda ^{\boldsymbol{A}}},\frac{a}{\lambda ^{\boldsymbol{A}}}) \times [-c,c] \times [-s,s] \right\}  $ the random set of admissible parameters for the echo state network. Since
\[ L_\sigma \left(\sum_{l=1}^N  \sup_{(A,C,\overline{\bm{\zeta}}) \in \Theta} \|A_{l,\cdot}\|_\infty \right) = \frac{a}{\lambda ^{\bf A}} L_\sigma \left(\sum_{l=1}^N  \|{\bf A}_{l,\cdot}\|_\infty \right) = a \in (0,1),
 \]
for any realization of ${\bf A},{\bf C},\bm{\zeta}$ (that is, conditional on ${\bf A},{\bf C},\bm{\zeta}$) the assumptions of Proposition~\ref{prop:ESNCase} are satisfied. Thus one may argue as in the proof of part {\bf (ii)} of Theorem~\ref{thm:main1}  to obtain that for any $\eta >0$,
\[\mathbb{P}\left( \left. \sup_{H \in \boldsymbol{ \mathcal{H}}^{RC}} \lbrace R(H) -  \widehat{R}_n(H)  \rbrace > \frac{(1-r^n)C_0}{n} + \eta \right| {\bf A},{\bf C},\bm{\zeta} \right) \leq \frac{\mathbb{E}[|\Gamma_n| \mid  {\bf A},{\bf C},\bm{\zeta}]}{\eta} \]
and then apply part {\bf (ii)} of Corollary~\ref{cor:ThetaWeighting} to obtain
\[ \mathbb{P}\left( \left. \sup_{H \in \boldsymbol{ \mathcal{H}}^{RC}} \lbrace R(H) -  \widehat{R}_n(H)  \rbrace > \frac{(1-r^n)C_0}{n} + \eta \right| {\bf A},{\bf C},\bm{\zeta} \right) \leq \frac{1}{ \eta} \left(\frac{C_1}{n} + \frac{C_2 {\log(n)}}{{n}} + \frac{{\bf C}_{3,abs} \sqrt{\log(n)}}{\sqrt{n}} \right), \]
where the 
 constants can be explicitly chosen using \eqref{eq:C1C2def}-\eqref{eq:C3def}. In particular,  
 $C_1$ and $C_2$ are given by \eqref{eq:C1C2def} with $C_I$ as in \eqref{eq:tauExpDecay}, and ${\bf C}_{3,abs}$ can be written using \eqref{eq:C3def} as 
 \begin{align*}
 {\bf C}_{3, abs} = 2 {\bf C}_3 + \frac{4 L_L \mathbb{E}\left[ \| {\bf Y}_{0} \|^2_2 \right]^{1/2}}{\sqrt{\log(\lambda_{max}^{-1})}}, 
 \end{align*}
with 
\begin{equation*}
{\bf C}_3= \frac{2 \sqrt{m} L_L  {\bf C_{RC}} }{\sqrt{\log(\lambda_{max}^{-1})}}
\end{equation*}
and 
\[ {\bf C_{RC}} =  \dfrac{\overline{L_h}}{ 1-a} \left(\lambda ^{\bf C}   {\mathbb{E} \left[ \left\|   {{\bf Z}}_{0} \right\|_2^2   \right]^{1/2} } + \lambda ^{\bm {\zeta}}  \right) + L_{h,0}. \]
Taking expectations, one sees that 
\[ \mathbb{E}[ {\bf C}_{3, abs}] =  { C}_{3, abs}, \]
where $C_{3, abs}$ is as in \eqref{eq:C3def}  with $C_{RC}$ given by \eqref{eq:CRCRandom}. 
Thus we obtain 
\[  \mathbb{P}\left( \sup_{H \in \boldsymbol{ \mathcal{H}}^{RC}} \lbrace R(H) -  \widehat{R}_n(H)  \rbrace > \frac{(1-r^n)C_0}{n} + \eta  \right) \leq \frac{1}{\eta} \left(\frac{C_1}{n} + \frac{C_2 {\log(n)}}{{n}} + \frac{C_{3,abs} \sqrt{\log(n)}}{\sqrt{n}}\right)\]
and the claim follows by arguing as in the proof of part {\bf (ii)} in Theorem~\ref{thm:main1} . $\blacksquare$

\acks{Lukas Gonon and Juan-Pablo Ortega acknowledge partial financial support coming from the Research Commission of the Universit\"at Sankt Gallen, the Swiss National Science Foundation (grants number 175801/1 and 179114), and the French ANR ``BIPHOPROC'' project (ANR-14-OHRI-0018-02). Lyudmila Grigoryeva acknowledges partial financial support of the Graduate School of Decision Sciences of the Universit\"at Konstanz.} 

\vskip 0.2in
\bibliography{/Users/JP17/Dropbox/Public/GOLibrary}

\begin{thebibliography}{100}
\providecommand{\natexlab}[1]{#1}
\providecommand{\url}[1]{\texttt{#1}}
\expandafter\ifx\csname urlstyle\endcsname\relax
  \providecommand{\doi}[1]{doi: #1}\else
  \providecommand{\doi}{doi: \begingroup \urlstyle{rm}\Url}\fi

\bibitem[Adams and Nobel(2010)]{Adams2010}
T.~M. Adams and A.~B. Nobel.
\newblock {Uniform convergence of Vapnik-Chervonenkis classes under ergodic
  sampling}.
\newblock \emph{Annals of Probability}, 38\penalty0 (4):\penalty0 1345--1367,
  2010.

\bibitem[Alon et~al.(1997)Alon, Ben-David, Cesa-Bianchi, and
  Haussler]{Alon1997}
N.~Alon, S.~Ben-David, N.~Cesa-Bianchi, and D.~Haussler.
\newblock {Scale-sensitive dimensions, uniform convergence, and learnability}.
\newblock \emph{Journal of the ACM}, 44\penalty0 (4):\penalty0 615--631, 1997.

\bibitem[Alquier and Wintenberger(2012)]{alquier:wintenberger}
P.~Alquier and O.~Wintenberger.
\newblock {Model selection for weakly dependent time series forecasting}.
\newblock \emph{Bernoulli}, 18\penalty0 (3):\penalty0 883--913, 2012.

\bibitem[Andrews(1983)]{Andrews1983}
D.~Andrews.
\newblock {First order autoregressive processes and strong mixing}.
\newblock \emph{Cowles Founda-tion Discussion Papers 664}, 1983.

\bibitem[Anthony and Bartlett(1999)]{AB1999}
M.~Anthony and P.~Bartlett.
\newblock \emph{{Neural Network Learning: Theoretical Foundations}}.
\newblock 1999.

\bibitem[Appeltant et~al.(2011)Appeltant, Soriano, {Van der Sande}, Danckaert,
  Massar, Dambre, Schrauwen, Mirasso, and Fischer]{Appeltant2011}
L.~Appeltant, M.~C. Soriano, G.~{Van der Sande}, J.~Danckaert, S.~Massar,
  J.~Dambre, B.~Schrauwen, C.~R. Mirasso, and I.~Fischer.
\newblock {Information processing using a single dynamical node as complex
  system}.
\newblock \emph{Nature Communications}, 2:\penalty0 468, jan 2011.

\bibitem[Baillie(1996)]{Baillie1996}
R.~T. Baillie.
\newblock {Long memory processes and fractional integration in econometrics}.
\newblock \emph{Journal of Econometrics}, 73\penalty0 (1):\penalty0 5--59,
  1996.

\bibitem[Bartlett and Mendelson(2003)]{Bartlett2003}
P.~L. Bartlett and S.~Mendelson.
\newblock {Rademacher and Gaussian complexities: Risk bounds and structural
  results}.
\newblock \emph{Journal of Machine Learning Research}, 3\penalty0 (3):\penalty0
  463--482, 2003.

\bibitem[Bartlett et~al.(2006)Bartlett, Jordan, and Mcauliffe]{Bartlett2006}
P.~L. Bartlett, M.~I. Jordan, and J.~D. Mcauliffe.
\newblock {Convexity, classification, and risk bounds}.
\newblock \emph{Journal of the American Statistical Association}, 101\penalty0
  (473):\penalty0 138--156, 2006.

\bibitem[Bartlett et~al.(2017)Bartlett, Foster, and Telgarsky]{Bartlett2017}
P.~L. Bartlett, D.~J. Foster, and M.~Telgarsky.
\newblock {Spectrally-normalized margin bounds for neural networks}.
\newblock \emph{Advances in Neural Information Processing Systems},
  2017-Decem:\penalty0 6241--6250, 2017.

\bibitem[Ben-David and Shalev-Shwartz(2014)]{Ben-David2014}
S.~Ben-David and S.~Shalev-Shwartz.
\newblock \emph{{Understanding Machine Learning: From Theory to Algorithms}}.
\newblock 2014.

\bibitem[Beran(1994)]{Beran1994}
J.~Beran.
\newblock \emph{{Statistics for Long-Memory Processes}}.
\newblock CRC Press, 1994.

\bibitem[Bollerslev(1986)]{bollerslev:garch}
T.~Bollerslev.
\newblock {Generalized autoregressive conditional heteroskedasticity}.
\newblock \emph{Journal of Econometrics}, 31\penalty0 (3):\penalty0 307--327,
  1986.

\bibitem[Boucheron et~al.(2013)Boucheron, Lugosi, and Massart]{Boucheron2013}
S.~Boucheron, G.~Lugosi, and P.~Massart.
\newblock \emph{{Concentration Inequalities: A Nonasymptotic Theory of
  Independence}}.
\newblock Oxford University Press, 2013.

\bibitem[Bougerol and Picard(1992)]{Bougerol1992}
P.~Bougerol and N.~Picard.
\newblock {Strict Stationarity of Generalized Autoregressive Processes}.
\newblock \emph{The Annals of Probability}, 1992.

\bibitem[Bousquet and Elisseeff(2002)]{Bousquet2002}
O.~Bousquet and A.~Elisseeff.
\newblock {Stability and generalisaiton}.
\newblock \emph{Journal of Machine Learning Reasearch}, 2:\penalty0 499--526,
  2002.

\bibitem[Boyd and Chua(1985)]{Boyd1985}
S.~Boyd and L.~Chua.
\newblock {Fading memory and the problem of approximating nonlinear operators
  with Volterra series}.
\newblock \emph{IEEE Transactions on Circuits and Systems}, 32\penalty0
  (11):\penalty0 1150--1161, nov 1985.

\bibitem[Brandt(1986)]{Brandt1986}
A.~Brandt.
\newblock {The stochastic equation Yn +1=AnYn + Bn with stationary
  coefficients}.
\newblock \emph{Advances in Applied Probability}, 18\penalty0 (01):\penalty0
  211--220, mar 1986.

\bibitem[Brunner et~al.(2013)Brunner, Soriano, Mirasso, and
  Fischer]{photonicReservoir2013}
D.~Brunner, M.~C. Soriano, C.~R. Mirasso, and I.~Fischer.
\newblock {Parallel photonic information processing at gigabyte per second data
  rates using transient states}.
\newblock \emph{Nature Communications}, 4\penalty0 (1364), 2013.

\bibitem[Buehner and Young(2006)]{Buehner:ESN}
M.~Buehner and P.~Young.
\newblock {A tighter bound for the echo state property}.
\newblock \emph{IEEE Transactions on Neural Networks}, 17\penalty0
  (3):\penalty0 820--824, 2006.

\bibitem[Christmann and Steinwart(2008)]{Christmann2008}
A.~Christmann and I.~Steinwart.
\newblock \emph{{Support Vector Machines}}.
\newblock Springer New York, 2008.

\bibitem[Coleman and Mizel(1968)]{Coleman1968}
B.~D. Coleman and V.~J. Mizel.
\newblock {On the general theory of fading memory}.
\newblock \emph{Archive for Rational Mechanics and Analysis}, 29\penalty0
  (1):\penalty0 18--31, jan 1968.

\bibitem[Couillet et~al.(2016)Couillet, Wainrib, Sevi, and Ali]{linearESN}
R.~Couillet, G.~Wainrib, H.~Sevi, and H.~T. Ali.
\newblock {The asymptotic performance of linear echo state neural networks}.
\newblock \emph{Journal of Machine Learning Research}, 17\penalty0
  (178):\penalty0 1--35, 2016.

\bibitem[Cucker and Smale(2002)]{cucker:smale}
F.~Cucker and S.~Smale.
\newblock {On the mathematical foundations of learning}.
\newblock \emph{Bulletin of the American Mathematical Society}, 39\penalty0
  (1):\penalty0 1--49, 2002.

\bibitem[Cucker and Zhou(2007)]{cucker:zhou:book}
F.~Cucker and D.-X. Zhou.
\newblock \emph{{Learning Theory : An Approximation Theory Viewpoint}}.
\newblock Cambridge University Press, 2007.

\bibitem[Cybenko(1989)]{cybenko}
G.~Cybenko.
\newblock {Approximation by superpositions of a sigmoidal function}.
\newblock \emph{Mathematics of Control, Signals, and Systems}, 2\penalty0
  (4):\penalty0 303--314, dec 1989.

\bibitem[Dambre et~al.(2012)Dambre, Verstraeten, Schrauwen, and
  Massar]{dambre2012}
J.~Dambre, D.~Verstraeten, B.~Schrauwen, and S.~Massar.
\newblock {Information processing capacity of dynamical systems}.
\newblock \emph{Scientific reports}, 2\penalty0 (514), 2012.

\bibitem[Dedecker et~al.(2007)Dedecker, Doukhan, Lang, Le{\'{o}}n, Louhichi,
  and Prieur]{Dedecker2007a}
J.~Dedecker, P.~Doukhan, G.~Lang, J.~R. Le{\'{o}}n, S.~Louhichi, and C.~Prieur.
\newblock \emph{{Weak Dependence: With Examples and Applications}}.
\newblock Springer Science+Business Media, 2007.

\bibitem[Dudley(2014)]{Dudley2014}
R.~M. Dudley.
\newblock \emph{{Uniform Central Limit Theorems}}.
\newblock Cambridge University Press, 2nd edition, 2014.

\bibitem[Engle(2009)]{engleCorrelationsBook}
R.~Engle.
\newblock \emph{{Anticipating Correlations}}.
\newblock Princeton University Press, Princeton, NJ, 2009.

\bibitem[Engle(1982)]{engle:arch}
R.~F. Engle.
\newblock {Autoregressive conditional heteroscedasticity with estimates of the
  variance of United Kingdom inflation}.
\newblock \emph{Econometrica}, 50\penalty0 (4):\penalty0 987--1007, 1982.

\bibitem[Fabrizio et~al.(2010)Fabrizio, Giorgi, and Pata]{Fabrizio2010}
M.~Fabrizio, C.~Giorgi, and V.~Pata.
\newblock \emph{{A new approach to equations with memory}}, volume 198.
\newblock 2010.

\bibitem[Fliess and Normand-Cyrot(1980)]{FliessNormand1980}
M.~Fliess and D.~Normand-Cyrot.
\newblock {Vers une approche alg{\'{e}}brique des syst{\`{e}}mes non
  lin{\'{e}}aires en temps discret}.
\newblock In A.~Bensoussan and J.~Lions, editors, \emph{Analysis and
  Optimization of Systems. Lecture Notes in Control and Information Sciences,
  vol. 28}. Springer Berlin Heidelberg, 1980.

\bibitem[Francq and Zakoian(2010)]{Francq2010}
C.~Francq and J.-M. Zakoian.
\newblock \emph{{GARCH Models: Structure, Statistical Inference and Financial
  Applications}}.
\newblock Wiley, 2010.

\bibitem[Funahashi(1989)]{funahashi:universality}
K.-i. Funahashi.
\newblock {On the approximate realization of continuous mappings by neural
  networks}.
\newblock \emph{Neural Networks}, 2:\penalty0 183--192, 1989.

\bibitem[Ganguli et~al.(2008)Ganguli, Huh, and Sompolinsky]{Ganguli2008}
S.~Ganguli, D.~Huh, and H.~Sompolinsky.
\newblock {Memory traces in dynamical systems.}
\newblock \emph{Proceedings of the National Academy of Sciences of the United
  States of America}, 105\penalty0 (48):\penalty0 18970--5, dec 2008.

\bibitem[Gin{\'{e}} and Zinn(1984)]{gine1984}
E.~Gin{\'{e}} and J.~Zinn.
\newblock {Some limit theorems for empirical processes}.
\newblock \emph{Annals of Probability}, 12:\penalty0 929--989, 1984.

\bibitem[Gonon and Ortega(2018)]{RC8}
L.~Gonon and J.-P. Ortega.
\newblock {Reservoir computing universality with stochastic inputs}.
\newblock \emph{IEEE Transactions on Neural Networks and Learning Systems},
  2018.

\bibitem[Gonon et~al.(2019)Gonon, Grigoryeva, and Ortega]{RC12}
L.~Gonon, L.~Grigoryeva, and J.-P. Ortega.
\newblock {Approximation bounds for random neural networks and reservoir
  systems}.
\newblock \emph{Preprint}, 2019.

\bibitem[Grigoryeva and Ortega(2018{\natexlab{a}})]{RC6}
L.~Grigoryeva and J.-P. Ortega.
\newblock {Universal discrete-time reservoir computers with stochastic inputs
  and linear readouts using non-homogeneous state-affine systems}.
\newblock \emph{Journal of Machine Learning Research}, 19\penalty0
  (24):\penalty0 1--40, 2018{\natexlab{a}}.

\bibitem[Grigoryeva and Ortega(2018{\natexlab{b}})]{RC7}
L.~Grigoryeva and J.-P. Ortega.
\newblock {Echo state networks are universal}.
\newblock \emph{Neural Networks}, 108:\penalty0 495--508, 2018{\natexlab{b}}.

\bibitem[Grigoryeva and Ortega(2019)]{RC9}
L.~Grigoryeva and J.-P. Ortega.
\newblock {Differentiable reservoir computing}.
\newblock \emph{arXiv: 1908.05202}, 2019.

\bibitem[Grigoryeva et~al.(2015)Grigoryeva, Henriques, Larger, and
  Ortega]{GHLO2014_capacity}
L.~Grigoryeva, J.~Henriques, L.~Larger, and J.-P. Ortega.
\newblock {Optimal nonlinear information processing capacity in delay-based
  reservoir computers}.
\newblock \emph{Scientific Reports}, 5\penalty0 (12858):\penalty0 1--11, 2015.

\bibitem[Grigoryeva et~al.(2016)Grigoryeva, Henriques, Larger, and Ortega]{RC3}
L.~Grigoryeva, J.~Henriques, L.~Larger, and J.-P. Ortega.
\newblock {Nonlinear memory capacity of parallel time-delay reservoir computers
  in the processing of multidimensional signals}.
\newblock \emph{Neural Computation}, 28:\penalty0 1411--1451, 2016.

\bibitem[Haussler(1992)]{Haussler1992}
D.~Haussler.
\newblock {Decision theoretic generalizations of the PAC model for neural net
  and other learning applications}.
\newblock \emph{Information and Computation}, 1992.

\bibitem[Hermans and Schrauwen(2010)]{Hermans2010}
M.~Hermans and B.~Schrauwen.
\newblock {Memory in linear recurrent neural networks in continuous time.}
\newblock \emph{Neural networks : the official journal of the International
  Neural Network Society}, 23\penalty0 (3):\penalty0 341--55, apr 2010.

\bibitem[Horn and Johnson(2013)]{horn:matrix:analysis}
R.~A. Horn and C.~R. Johnson.
\newblock \emph{{Matrix Analysis}}.
\newblock Cambridge University Press, second edition, 2013.

\bibitem[Hornik et~al.(1989)Hornik, Stinchcombe, and White]{hornik}
K.~Hornik, M.~Stinchcombe, and H.~White.
\newblock {Multilayer feedforward networks are universal approximators}.
\newblock \emph{Neural Networks}, 2\penalty0 (5):\penalty0 359--366, 1989.

\bibitem[Hosking(1981)]{Hosking1981}
J.~R.~M. Hosking.
\newblock {Fractional differencing}.
\newblock \emph{Biometrika}, 1981.

\bibitem[Hyt{\"{o}}nen et~al.(2016)Hyt{\"{o}}nen, van Neerven, Veraar, and
  Weis]{AnalysisBanachSpaces:vol1}
T.~Hyt{\"{o}}nen, J.~van Neerven, M.~Veraar, and L.~Weis.
\newblock \emph{{Analysis in Banach Spaces}}, volume~I.
\newblock Springer International Publishing, 2016.

\bibitem[Ib{\'{a}}{\~{n}}ez-Soria et~al.(2019)Ib{\'{a}}{\~{n}}ez-Soria,
  Soria-Frisch, Garcia-Ojalvo, and Ruffini]{Ibanez-Soria2019}
D.~Ib{\'{a}}{\~{n}}ez-Soria, A.~Soria-Frisch, J.~Garcia-Ojalvo, and G.~Ruffini.
\newblock {Characterization of the non-stationary nature of steady-state visual
  evoked potentials using echo state networks}.
\newblock \emph{PLOS ONE}, 2019.

\bibitem[Jaeger(2002)]{Jaeger:2002}
H.~Jaeger.
\newblock {Short term memory in echo state networks}.
\newblock \emph{Fraunhofer Institute for Autonomous Intelligent Systems.
  Technical Report.}, 152, 2002.

\bibitem[Jaeger(2010)]{jaeger2001}
H.~Jaeger.
\newblock {The 'echo state' approach to analysing and training recurrent neural
  networks with an erratum note}.
\newblock Technical report, German National Research Center for Information
  Technology, 2010.

\bibitem[Jaeger and Haas(2004)]{Jaeger04}
H.~Jaeger and H.~Haas.
\newblock {Harnessing Nonlinearity: Predicting Chaotic Systems and Saving
  Energy in Wireless Communication}.
\newblock \emph{Science}, 304\penalty0 (5667):\penalty0 78--80, 2004.

\bibitem[Khintchine(1923)]{Khintchine1923}
A.~Khintchine.
\newblock {{\"{U}}ber dyadische Br{\"{u}}che}.
\newblock \emph{Mathematische Zeitschriften}, 18:\penalty0 109--116, 1923.

\bibitem[Koiran and Sontag(1998)]{Koiran1998}
P.~Koiran and E.~D. Sontag.
\newblock {Vapnik-Chervonenkis dimension of recurrent neural networks}.
\newblock \emph{Discrete Applied Mathematics}, 1998.

\bibitem[Kuznetsov and Mohri(2017)]{Kuznetsov2017}
V.~Kuznetsov and M.~Mohri.
\newblock {Generalization bounds for non-stationary mixing processes}.
\newblock \emph{Machine Learning}, 106\penalty0 (1):\penalty0 93--117, 2017.

\bibitem[Kuznetsov and Mohri(2018)]{Kuznetsov2018}
V.~Kuznetsov and M.~Mohri.
\newblock {Theory and algorithms for forecasting time series}.
\newblock 2018.

\bibitem[Laporte et~al.(2018)Laporte, Katumba, Dambre, and
  Bienstman]{Laporte2018}
F.~Laporte, A.~Katumba, J.~Dambre, and P.~Bienstman.
\newblock {Numerical demonstration of neuromorphic computing with photonic
  crystal cavities}.
\newblock \emph{Optics Express}, 26\penalty0 (7):\penalty0 7955, apr 2018.

\bibitem[Larger et~al.(2012)Larger, Soriano, Brunner, Appeltant, Gutierrez,
  Pesquera, Mirasso, and Fischer]{Larger2012}
L.~Larger, M.~C. Soriano, D.~Brunner, L.~Appeltant, J.~M. Gutierrez,
  L.~Pesquera, C.~R. Mirasso, and I.~Fischer.
\newblock {Photonic information processing beyond Turing: an optoelectronic
  implementation of reservoir computing}.
\newblock \emph{Optics Express}, 20\penalty0 (3):\penalty0 3241, jan 2012.

\bibitem[Ledoux and Talagrand(1991)]{ledoux:talagrand}
M.~Ledoux and M.~Talagrand.
\newblock \emph{{Probability in Banach Spaces}}.
\newblock Springer-Verlag, 1991.

\bibitem[Legenstein and Maass(2007)]{DynamicalSystemsMaass}
R.~Legenstein and W.~Maass.
\newblock {What makes a dynamical system computationally powerful?}
\newblock In S.~Haykin, editor, \emph{New directions in statistical signal
  processing: from systems to brain}. MIT Press, Cambridge, MA, 2007.

\bibitem[Lu et~al.(2018)Lu, Hunt, and Ott]{Ott2018}
Z.~Lu, B.~R. Hunt, and E.~Ott.
\newblock {Attractor reconstruction by machine learning}.
\newblock \emph{Chaos}, 28\penalty0 (6), 2018.

\bibitem[Luko{\v{s}}evi{\v{c}}ius and Jaeger(2009)]{lukosevicius}
M.~Luko{\v{s}}evi{\v{c}}ius and H.~Jaeger.
\newblock {Reservoir computing approaches to recurrent neural network
  training}.
\newblock \emph{Computer Science Review}, 3\penalty0 (3):\penalty0 127--149,
  2009.

\bibitem[Maass et~al.(2002)Maass, Natschl{\"{a}}ger, and Markram]{maass1}
W.~Maass, T.~Natschl{\"{a}}ger, and H.~Markram.
\newblock {Real-time computing without stable states: a new framework for
  neural computation based on perturbations}.
\newblock \emph{Neural Computation}, 14:\penalty0 2531--2560, 2002.

\bibitem[Maass(2011)]{maass2}
W.~Maass.
\newblock {Liquid state machines: motivation, theory, and applications}.
\newblock In S.~S. {Barry Cooper} and A.~Sorbi, editors, \emph{Computability In
  Context: Computation and Logic in the Real World}, chapter~8, pages 275--296.
  2011.

\bibitem[Maass and Sontag(2000)]{Maass2000}
W.~Maass and E.~D. Sontag.
\newblock {Neural Systems as Nonlinear Filters}.
\newblock \emph{Neural Computation}, 12\penalty0 (8):\penalty0 1743--1772, aug
  2000.

\bibitem[Maass et~al.(2004)Maass, Natschl{\"{a}}ger, and
  Markram]{corticalMaass}
W.~Maass, T.~Natschl{\"{a}}ger, and H.~Markram.
\newblock {Fading memory and kernel properties of generic cortical microcircuit
  models}.
\newblock \emph{Journal of Physiology Paris}, 98\penalty0 (4-6 SPEC.
  ISS.):\penalty0 315--330, 2004.

\bibitem[Maass et~al.(2007)Maass, Joshi, and Sontag]{MaassUniversality}
W.~Maass, P.~Joshi, and E.~D. Sontag.
\newblock {Computational aspects of feedback in neural circuits}.
\newblock \emph{PLoS Computational Biology}, 3\penalty0 (1):\penalty0 e165,
  2007.

\bibitem[Manjunath and Jaeger(2013)]{Manjunath:Jaeger}
G.~Manjunath and H.~Jaeger.
\newblock {Echo state property linked to an input: exploring a fundamental
  characteristic of recurrent neural networks}.
\newblock \emph{Neural Computation}, 25\penalty0 (3):\penalty0 671--696, 2013.

\bibitem[Marzen(2017)]{marzen:capacity}
S.~Marzen.
\newblock {Difference between memory and prediction in linear recurrent
  networks}.
\newblock \emph{Physical Review E}, 96\penalty0 (3):\penalty0 1--7, 2017.

\bibitem[Matthews(1992)]{Matthews:thesis}
M.~B. Matthews.
\newblock \emph{{On the Uniform Approximation of Nonlinear Discrete-Time
  Fading-Memory Systems Using Neural Network Models}}.
\newblock PhD thesis, ETH Z{\"{u}}rich, 1992.

\bibitem[McDonald et~al.(2017)McDonald, Shalizi, and Schervish]{McDonald2012}
D.~J. McDonald, C.~R. Shalizi, and M.~Schervish.
\newblock {Nonparametric risk bounds for time-series forecasting}.
\newblock \emph{Journal of Machine Learning Research}, 18:\penalty0 1--40,
  2017.

\bibitem[Mukherjee et~al.(2006)Mukherjee, Niyogi, Poggio, and
  Rifkin]{Mukherjee2002}
S.~Mukherjee, P.~Niyogi, T.~Poggio, and R.~Rifkin.
\newblock {Learning theory: stability is sufficient for generalization and
  necessary and sufficient for consistency of empirical risk minimization}.
\newblock \emph{Advances in Computational Mathematics}, 25\penalty0
  (1-3):\penalty0 161--193, 2006.

\bibitem[Munkres(2014)]{Munkres:topology}
J.~Munkres.
\newblock \emph{{Topology}}.
\newblock Pearson, second edition, 2014.

\bibitem[Natschl{\"{a}}ger et~al.(2002)Natschl{\"{a}}ger, Maass, and
  Markram]{Natschlager:117806}
T.~Natschl{\"{a}}ger, W.~Maass, and H.~Markram.
\newblock {The "Liquid Computer": a novel strategy for real-time computing on
  time series}.
\newblock \emph{Special Issue on Foundations of Information Processing of
  TELEMATIK}, 8\penalty0 (1):\penalty0 39--43, 2002.

\bibitem[Paquot et~al.(2012)Paquot, Duport, Smerieri, Dambre, Schrauwen,
  Haelterman, and Massar]{Paquot2012}
Y.~Paquot, F.~Duport, A.~Smerieri, J.~Dambre, B.~Schrauwen, M.~Haelterman, and
  S.~Massar.
\newblock {Optoelectronic reservoir computing}.
\newblock \emph{Scientific reports}, 2:\penalty0 287, jan 2012.

\bibitem[Pathak et~al.(2017)Pathak, Lu, Hunt, Girvan, and Ott]{pathak:chaos}
J.~Pathak, Z.~Lu, B.~R. Hunt, M.~Girvan, and E.~Ott.
\newblock {Using machine learning to replicate chaotic attractors and calculate
  Lyapunov exponents from data}.
\newblock \emph{Chaos}, 27\penalty0 (12), 2017.

\bibitem[Pathak et~al.(2018)Pathak, Hunt, Girvan, Lu, and Ott]{Pathak:PRL}
J.~Pathak, B.~Hunt, M.~Girvan, Z.~Lu, and E.~Ott.
\newblock {Model-Free Prediction of Large Spatiotemporally Chaotic Systems from
  Data: A Reservoir Computing Approach}.
\newblock \emph{Physical Review Letters}, 120\penalty0 (2):\penalty0 24102,
  2018.

\bibitem[Poggio et~al.(2004)Poggio, Rifkin, Mukherjee, and Niyogi]{Poggio2004}
T.~Poggio, R.~Rifkin, S.~Mukherjee, and P.~Niyogi.
\newblock {General conditions for predictivity in learning theory}.
\newblock \emph{Nature}, 428\penalty0 (6981):\penalty0 419--422, 2004.

\bibitem[Rakhlin et~al.(2010)Rakhlin, Sridharan, and
  Tewari]{AlexanderRakhlinKarthikSridharan2015}
A.~Rakhlin, K.~Sridharan, and A.~Tewari.
\newblock {Online learning via sequential complexities}.
\newblock \emph{Journal of Machine Learning Research}, 16:\penalty0 155--186,
  2010.

\bibitem[Rakhlin et~al.(2014)Rakhlin, Sridharan, and Tewari]{Rakhlin2014}
A.~Rakhlin, K.~Sridharan, and A.~Tewari.
\newblock {Sequential complexities and uniform martingale laws of large
  numbers}.
\newblock \emph{Probability Theory and Related Fields}, 161\penalty0
  (1-2):\penalty0 111--153, 2014.

\bibitem[Rodan and Tino(2011)]{Rodan2011}
A.~Rodan and P.~Tino.
\newblock {Minimum complexity echo state network.}
\newblock \emph{IEEE Transactions on Neural Networks}, 22\penalty0
  (1):\penalty0 131--44, jan 2011.

\bibitem[Smale and Zhou(2003)]{Smale2003}
S.~Smale and D.-X. Zhou.
\newblock {Estimating the approximation error in learning theory}.
\newblock \emph{Analysis and Applications}, 01\penalty0 (01):\penalty0 17--41,
  2003.

\bibitem[Sontag(1979{\natexlab{a}})]{Sontag1979}
E.~Sontag.
\newblock {Realization theory of discrete-time nonlinear systems: Part I-The
  bounded case}.
\newblock \emph{IEEE Transactions on Circuits and Systems}, 26\penalty0
  (5):\penalty0 342--356, may 1979{\natexlab{a}}.

\bibitem[Sontag(1979{\natexlab{b}})]{sontag:polynomial:1979}
E.~D. Sontag.
\newblock {Polynomial Response Maps}.
\newblock In \emph{Lecture Notes Control in Control and Information Sciences.
  Vol. 13}. Springer Verlag, 1979{\natexlab{b}}.

\bibitem[Sontag(1998)]{sontag:VC}
E.~D. Sontag.
\newblock {VC dimension of neural networks}.
\newblock \emph{NATO ASI Series F Computer and Systems Sciences}, 168:\penalty0
  69--96, 1998.

\bibitem[Sternberg(2010)]{Sternberg:dynamical:book}
S.~Sternberg.
\newblock \emph{{Dynamical Systems}}.
\newblock Dover, 2010.

\bibitem[Vandoorne et~al.(2011)Vandoorne, Dambre, Verstraeten, Schrauwen, and
  Bienstman]{SOASforRC}
K.~Vandoorne, J.~Dambre, D.~Verstraeten, B.~Schrauwen, and P.~Bienstman.
\newblock {Parallel reservoir computing using optical amplifiers}.
\newblock \emph{IEEE Transactions on Neural Networks}, 22\penalty0
  (9):\penalty0 1469--1481, sep 2011.

\bibitem[Vandoorne et~al.(2014)Vandoorne, Mechet, {Van Vaerenbergh}, Fiers,
  Morthier, Verstraeten, Schrauwen, Dambre, and Bienstman]{swirl:paper}
K.~Vandoorne, P.~Mechet, T.~{Van Vaerenbergh}, M.~Fiers, G.~Morthier,
  D.~Verstraeten, B.~Schrauwen, J.~Dambre, and P.~Bienstman.
\newblock {Experimental demonstration of reservoir computing on a silicon
  photonics chip}.
\newblock \emph{Nature Communications}, 5:\penalty0 78--80, mar 2014.

\bibitem[Vapnik(1991)]{Vapnik1991}
V.~Vapnik.
\newblock {Principles of risk minimization for learning theory}.
\newblock In \emph{Advances in Neural Information Processing Systems 4 (NIPS
  1991)}, pages 831--838, 1991.

\bibitem[Vapnik(1998)]{Vapnik1998}
V.~Vapnik.
\newblock \emph{{Statistical Learning Theory}}.
\newblock Wiley, adaptive a edition, 1998.

\bibitem[Vapnik and Chervonenkis(1968)]{Vapnik1968}
V.~Vapnik and A.~Y. Chervonenkis.
\newblock {On the uniform convergence of relative frequencies of events to
  their probabilities}.
\newblock \emph{Dokl. Akad. Nauk SSSR}, 181\penalty0 (4):\penalty0 781, 1968.

\bibitem[Verzelli et~al.(2019)Verzelli, Alippi, and Livi]{Verzelli2019}
P.~Verzelli, C.~Alippi, and L.~Livi.
\newblock {Echo State Networks with self-normalizing activations on the
  hyper-sphere}.
\newblock \emph{Scientific Reports}, 9\penalty0 (13887), 2019.

\bibitem[Vinckier et~al.(2015)Vinckier, Duport, Smerieri, Vandoorne, Bienstman,
  Haelterman, and Massar]{Vinckier2015}
Q.~Vinckier, F.~Duport, A.~Smerieri, K.~Vandoorne, P.~Bienstman, M.~Haelterman,
  and S.~Massar.
\newblock {High-performance photonic reservoir computer based on a coherently
  driven passive cavity}.
\newblock \emph{Optica}, 2\penalty0 (5):\penalty0 438--446, 2015.

\bibitem[Volterra(1930)]{volterra:book}
V.~Volterra.
\newblock \emph{{Theory of Functionals and of Integral and Integro-Differential
  Equations}}.
\newblock Blackie {\&} Son Limited, Glasgow, 1930.

\bibitem[White et~al.(2004)White, Lee, and Sompolinsky]{White2004}
O.~White, D.~Lee, and H.~Sompolinsky.
\newblock {Short-Term Memory in Orthogonal Neural Networks}.
\newblock \emph{Physical Review Letters}, 92\penalty0 (14):\penalty0 148102,
  apr 2004.

\bibitem[Wiener(1958)]{wiener:book}
N.~Wiener.
\newblock \emph{{Nonlinear Problems in Random Theory}}.
\newblock The Technology Press of MIT, 1958.

\bibitem[Yildiz et~al.(2012)Yildiz, Jaeger, and Kiebel]{Yildiz2012}
I.~B. Yildiz, H.~Jaeger, and S.~J. Kiebel.
\newblock {Re-visiting the echo state property.}
\newblock \emph{Neural Networks}, 35:\penalty0 1--9, nov 2012.

\bibitem[Zhang et~al.(2018)Zhang, Lei, and Dhillon]{Zhang2018}
J.~Zhang, Q.~Lei, and I.~S. Dhillon.
\newblock {Stabilizing gradients for deep neural networks via efficient SVD
  parameterization}.
\newblock 2018.

\end{thebibliography}
\end{document}